\documentclass{article}

\PassOptionsToPackage{numbers, compress}{natbib}

\usepackage[main,final]{neurips_2026}

\usepackage[utf8]{inputenc}
\usepackage[T1]{fontenc}
\usepackage{hyperref}
\usepackage{url}
\usepackage{float}
\usepackage{booktabs}
\usepackage{amsfonts}
\usepackage{nicefrac}
\usepackage{microtype}
\usepackage{graphicx}
\graphicspath{{./}{images/}}
\usepackage{wrapfig}
\usepackage{subcaption}
\usepackage{caption}
\usepackage{tcolorbox}
\tcbuselibrary{skins,breakable}
\usepackage[ruled,linesnumbered]{algorithm2e}
\usepackage{appendix}
\usepackage{amsthm,amssymb,amsmath,amsfonts}
\usepackage{enumitem}
\usepackage{tikz}
\IfFileExists{tablefootnote.sty}{\usepackage{tablefootnote}}{}
\usetikzlibrary{arrows.meta}
\usetikzlibrary{decorations.pathreplacing}
\usetikzlibrary{positioning}
\usetikzlibrary{calc}
\usetikzlibrary{backgrounds}
\usepackage{xcolor}
\IfFileExists{tikz-cd.sty}{\usepackage{tikz-cd}}{}
\usepackage{multirow}
\usepackage{makecell}
\definecolor{steelblue}{RGB}{70, 130, 180}
\definecolor{darkorange}{RGB}{230, 120, 50}
\IfFileExists{stmaryrd.sty}{%
  \usepackage{stmaryrd}%
  \input{Ustmry.fd}%
  \DeclareFontShape{U}{stmry}{b}{n}{<->ssub*stmry/m/n}{}%
}{}

\providecommand{\shortrightarrow}{{\mathrel{\rightarrow}}}
\providecommand{\shortleftarrow}{{\mathrel{\leftarrow}}}

\newcommand{\sR}{\mathbb{R}}
\newcommand{\R}{\mathbb{R}}
\newcommand{\sX}{\mathcal{X}}

\newcommand{\gL}{\mathcal{L}}
\newcommand{\gS}{\mathcal{S}}
\newcommand{\gV}{\mathcal{V}}
\newcommand{\cA}{\mathcal{A}}

\newcommand{\mA}{\mathbf{A}}

\newcommand{\mL}{\mathbf{L}}

\newcommand{\mP}{\mathbf{P}}
\newcommand{\mQ}{\mathbf{Q}}
\newcommand{\mR}{\mathbf{R}}

\newcommand{\vc}{\mathbf{c}}

\newcommand{\vu}{\mathbf{u}}
\newcommand{\vv}{\mathbf{v}}
\newcommand{\vw}{\mathbf{w}}
\newcommand{\vx}{\mathbf{x}}

\newcommand{\vz}{\mathbf{z}}
\newcommand{\vone}{\mathbf{1}}
\newcommand{\vzero}{\mathbf{0}}

\newcommand{\vtheta}{\boldsymbol{\theta}}

\newcommand{\E}{\mathbb{E}}
\newcommand{\dif}{\mathop{}\!\mathrm{d}}
\DeclareMathOperator{\KL}{KL}
\DeclareMathOperator{\TV}{TV}
\DeclareMathOperator{\diag}{diag}
\DeclareMathOperator{\osc}{osc}
\DeclareMathOperator{\Span}{span}
\DeclareMathOperator*{\argmin}{arg\,min}

\newcommand{\spc}{{\mathrm{sp}}}
\newcommand{\heu}{{\mathrm{heu}}}
\newcommand{\eff}{{\mathrm{eff}}}

\newcommand{\vcg}{{\mathrm{vcg}}}
\newcommand{\fwd}{\shortrightarrow}
\newcommand{\bwd}{\shortleftarrow}

\newcommand{\Div}{\operatorname{div}}
\newcommand{\Var}{\operatorname{Var}}

\newcommand{\Wass}{\mathsf{W}}

\newcommand{\normone}[1]{\left\|#1\right\|_{1}}
\newcommand{\norminf}[1]{\left\|#1\right\|_{\infty}}
\newcommand{\normcol}[1]{\left\|#1\right\|_{1\to 1}}

\numberwithin{equation}{section}

\newtheorem{theorem}{Theorem}[section]
\newtheorem{lemma}[theorem]{Lemma}
\newtheorem{assumption}{Assumption}[section]

\newtheorem{proposition}[theorem]{Proposition}
\newtheorem{remark}[theorem]{Remark}

\tcbset{
  thmbox/.style={
    enhanced jigsaw,
    breakable,
    colback=black!4,
    colframe=black!18,
    boxrule=0.4pt,
    arc=2pt,
    left=4pt, right=4pt, top=2pt, bottom=2pt,
    boxsep=2pt,
    before skip=4pt, after skip=4pt,
    lines before break=4,
  },
}
\tcolorboxenvironment{theorem}{thmbox}
\tcolorboxenvironment{lemma}{thmbox}
\tcolorboxenvironment{proposition}{thmbox}
\tcolorboxenvironment{corollary}{thmbox}
\tcolorboxenvironment{definition}{thmbox}
\tcolorboxenvironment{assumption}{thmbox}

\definecolor{darkblue}{rgb}{0, 0, 0.5}
\hypersetup{colorlinks=true, citecolor=darkblue, linkcolor=darkblue, urlcolor=darkblue}
\hypersetup{pdftitle={FluxLite: Inference-Time Proposal Control for Discrete Diffusion Models}, pdfauthor={Yinuo Ren, Haoxuan Chen, Grant M. Rotskoff, Jiequn Han, Lexing Ying}}

\makeatletter
\DeclareRobustCommand{\cev}[1]{%
  {\mathpalette\do@cev{#1}}%
}
\newcommand{\do@cev}[2]{%
  \vbox{\offinterlineskip
    \sbox\z@{$\m@th#1 x$}%
    \ialign{##\cr
      \hidewidth\reflectbox{$\m@th#1\vec{}\mkern4mu$}\hidewidth\cr
      \noalign{\kern-\ht\z@}
      $\m@th#1#2$\cr
    }%
  }%
}
\makeatother

\title{FluxLite: Inference-Time Proposal Control for Discrete Diffusion Models}

\author{%
  Yinuo Ren\\
  Kempner Institute\\
  Harvard University\\
  \texttt{yinuo\_ren@harvard.edu}\\
  \And
  Haoxuan Chen\\
  ICME\\
  Stanford University\\
  \texttt{haoxuanc@stanford.edu}\\
  \And
  Grant M. Rotskoff\\
  Department of Chemistry\\
  Stanford University\\
  \texttt{rotskoff@stanford.edu}\\
  \And
  Jiequn Han\\
  Center for Computational Mathematics\\
  Flatiron Institute\\
  \texttt{jhan@flatironinstitute.org}\\
  \And
  Lexing Ying\\
  Department of Mathematics\\
  Stanford University\\
  \texttt{lexing@stanford.edu}
}

\begin{document}

\maketitle

\begin{abstract}
Many inference-time tasks for pretrained discrete diffusion models and diffusion language models reduce to drawing samples from a tilted version of the pretrained distribution. Feynman--Kac sequential Monte Carlo (SMC) makes this correction exact in principle, but its prescribed weights routinely degenerate when the proposal dynamics are misaligned with the tilt, capping the practical gains from additional particles. We introduce \emph{FluxLite}, a lightweight, training-free proposal-control framework for discrete diffusion. On the sparse directed graph of pretrained reverse rates, any sparse jump-rate perturbation can be exactly compensated by a $q_t$-weighted graph-divergence term in the Feynman--Kac potential; the target path is therefore preserved while the residual reweighting variance becomes a local convex objective. We instantiate this principle as two practical samplers: a one-hop local reallocation rule (\texttt{HEU}) and a small nonnegative quadratic program over pretrained-rate bases (\texttt{D-VCG}). We further prove population stability under the standard score-entropy training loss, identifying a tilted-path coverage factor that governs robustness to score error, together with finite-particle convergence for a fixed controlled Feynman--Kac recursion. Empirically, FluxLite improves over standard Feynman--Kac SMC baselines by up to two orders of magnitude in terminal $\KL$ on an analytically tractable finite-state CTMC benchmark, and reduces row-correlation MSE on 2D Ising sampling by $5$--$7\times$ in geometric mean and up to $55\times$ at peak.
\end{abstract}

\section{Introduction}
Many inference-time tasks in modern generative modeling reduce to sampling from a \emph{tilted} version of a pretrained model distribution: reward-aligned generation, classifier-free guidance, posterior sampling under measurement constraints, controllable text generation, and chain-of-thought reasoning. The pattern recurs across autoregressive large language models (LLMs), continuous diffusion models, and discrete diffusion models including Diffusion Language Models (DLMs). It is also part of the broader shift from training-time scaling laws~\citep{kaplan2020scaling,hoffmann2022training} to \emph{inference-time scaling}, in which additional compute at deployment is allocated to selection, verification, search, or sampling rather than to enlarging the base model~\citep{muennighoff2025s1,lifshitz2025multi,zhang2025survey,uehara2025inference,ma2025scaling,chen2025solving,ren2025driftlite,hasan2026discrete}. Appendix~\ref{app:related-notation-further} collects the model-specific and domain-specific literature.

Existing inference-time methods divide into \emph{length scaling} (longer or iteratively refined trajectories) and \emph{width scaling} (exploring multiple candidate hypotheses in parallel). Sampling-based width scaling typically targets a tilted law through Feynman--Kac sequential Monte Carlo (SMC), in which particles are propagated under proposed dynamics and reweighted by an incremental potential prescribed by the path~\citep{singhal2025general,chen2025solving,skreta2025feynman,lew2023sequential,hasan2026discrete}. Its central bottleneck is \emph{weight degeneracy}: the proposal and the incremental weights are tied together by the Feynman--Kac representation, so when they are misaligned, weights concentrate on a vanishing fraction of particles and additional particles do not translate into additional information, eroding gains over Best-of-$N$ or related selection baselines~\citep{lightman2023let,gui2024bonbon}. Twisted and controlled SMC in statistics~\citep{briers2010smoothing,whiteley2014twisted,heng2020controlled} suggests that better proposals can mitigate this, and recent inference-time methods for diffusion and LLMs pursue related ideas~\citep{zhao2024probabilistic,feng2025step,albergo2025nets,holderrieth2025leaps,ou2026inference}, but typically at the cost of extra training or modality-specific heuristics.

This raises a basic question:
\begin{center}
\begin{minipage}{0.85\linewidth}
    \centering\itshape
    Is there a single training-free proposal-control principle that works across continuous and discrete inference-time SMC?
\end{minipage}
\end{center}
DriftLite~\citep{ren2025driftlite} answers this question for \emph{continuous} diffusion models by exploiting the non-uniqueness of the Feynman--Kac PDE: probability transport can be shifted between a control drift and a residual reweighting potential through a $q_t$-weighted divergence, and a basis-restricted variance objective selects a representative without retraining. The discrete case is structurally different: instead of a vector field on $\sR^d$, the proposal is a sparse directed graph of jump rates fixed by the pretrained reverse process, and the divergence operator must respect this sparsity pattern. We show that the tilted continuous-time-Markov-chain (CTMC) Feynman--Kac flow nevertheless admits an analogous equivalence class: sparse jump-rate perturbations are exactly compensated by a graph-divergence correction to the potential. Proposal design therefore becomes a sparse, training-free flux-reallocation problem with the same variance-control objective as in DriftLite.

\subsection{Contributions}

\begin{itemize}[leftmargin=1.5em, itemsep=0pt, topsep=0pt, parsep=0pt]
    \item \textbf{Discrete proposal-control framework.} We introduce \emph{FluxLite}, the discrete analogue of DriftLite. It exposes an exact equivalence class of CTMC--Feynman--Kac representations on the pretrained sparsity graph and recasts proposal control as sparse graph-flux variance minimization (Table~\ref{tab:cont-disc-div}).
    \item \textbf{Theory.} We prove (i) population stability of the guided Feynman--Kac flow under the standard score-entropy training loss, identifying a tilted-path coverage factor that quantifies robustness to score error, and (ii) finite-particle convergence for the controlled SMC recursion. Together these results make explicit how score quality, tilted-path coverage, residual-weight oscillation, and particle count enter inference-time control.
    \item \textbf{Practical samplers and experiments.} We instantiate FluxLite as a one-hop local heuristic (\texttt{HEU}) and a pretrained-rate basis QP, \emph{Discrete Variance-Controlling Guidance} (\texttt{D-VCG}). On analytically tractable finite-state CTMCs, \texttt{D-VCG} cuts terminal $\KL$ over the Discrete Feynman--Kac Corrector (\texttt{D-FKC})~\citep{hasan2026discrete} baseline by up to $114.8\times$; on 2D Ising sampling it reduces row-correlation MSE over \texttt{D-FKC} by $5$--$7\times$ in geometric mean and up to $55\times$ at peak.
\end{itemize}

\section{Preliminaries}
\label{sec:preliminaries}

We review the Feynman--Kac view of inference-time scaling in a form that makes the continuous--discrete correspondence explicit. Section~\ref{subsec:smc-cont} recalls the DriftLite identity for continuous diffusion models; Section~\ref{subsec:smc-scaling-discrete} gives the discrete CTMC--Feynman--Kac flow that FluxLite controls.

\subsection{Continuous Feynman--Kac control}
\label{subsec:smc-cont}

Let $\sX=\sR^d$ and let the backward process of a pretrained diffusion model have base drift $\vv_t^\bwd$, diffusion coefficient $\alpha_t^\bwd$, and marginal density $p_t^\bwd$ at reverse time $t \in [0, T]$. Many inference-time objectives define a tilted path $q_t(x)\propto (p_t^\bwd(x))^\gamma e^{r_t(x)}$, where $\gamma>0$ and $r_t$ is a reward, likelihood, or guidance potential.
For a chosen Feynman--Kac representative of this tilted path, write the proposal drift as $\widetilde\vv_t$. The normalized Feynman--Kac equation can be written compactly as
\begin{equation}
\partial_t q_t=-\nabla\!\cdot(\widetilde\vv_t q_t)+\frac{\alpha_t^\bwd{}^2}{2}\Delta q_t+g_t q_t,
\qquad \E_{q_t}[g_t]=0.
\end{equation}
DriftLite~\citep{ren2025driftlite} observes that this representation is not unique. For any control drift $\vu_t:\sR^d\to\sR^d$,
\begin{equation}
\label{eq:cont-driftlite-equivalence}
-\nabla\!\cdot(\widetilde\vv_t q_t)+g_tq_t
=
-\nabla\!\cdot((\widetilde\vv_t+\vu_t)q_t)
+\big(g_t+\Div_{q_t}\vu_t\big)q_t,
\quad
\Div_{q_t}\vu_t:= q_t^{-1}\nabla\!\cdot(q_t\vu_t).
\end{equation}
Thus the same $q_t$ can be represented by many proposal drifts and residual potentials. DriftLite chooses a representative by solving, or approximating, $\min_{\vu_t}\Var_{q_t}[g_t+\Div_{q_t}\vu_t]$. The discrete construction below replaces the continuous divergence in~\eqref{eq:cont-driftlite-equivalence} by a $q_t$-weighted graph divergence and replaces drift control by sparse jump-rate control. The detailed continuous setup is recalled in Appendix~\ref{app:driftlite-related}.

\subsection{Discrete Feynman--Kac flow}
\label{subsec:smc-scaling-discrete}

Let $\sX$ be a finite state space with $D:=|\sX|<\infty$. We use the column convention: $Q_t(y,x)$ denotes the rate from source $x$ to destination $y$, so a CTMC marginal $p_t$ evolves as $\partial_t p_t(x)=\sum_{y\neq x}\big(Q_t(x,y)p_t(y)-Q_t(y,x)p_t(x)\big)$. Generator-matrix notation and all proofs for Sections~\ref{sec:preliminaries}--\ref{sec:methodology} are in Appendix~\ref{app:proofs-props}.
Let $p_t^\fwd$ be the forward noising marginal and $p_t^\bwd:=p_{T-t}^\fwd$ the reverse-time marginal. When the forward noising rates are known, exact time reversal gives the local-ratio form
\begin{equation}
Q_t^\bwd(y,x)=Q_{T-t}^\fwd(x,y)\frac{p_t^\bwd(y)}{p_t^\bwd(x)} := Q_{T-t}^\fwd(x,y)s_t(x,y),
\qquad x\neq y.
\end{equation}
A learned discrete diffusion model replaces $s_t$ by an estimate $\widehat s_t$, yielding $\widehat Q_t^\bwd(y,x)=Q_{T-t}^\fwd(x,y)\widehat s_t(x,y)$. The reverse marginal satisfies
\begin{equation}
\label{eq:prior-discrete-backward-ctmc}
\partial_t p_t^\bwd(x)=\sum_{y\neq x}\big(Q_t^\bwd(x,y)p_t^\bwd(y)-Q_t^\bwd(y,x)p_t^\bwd(x)\big).
\end{equation}
For a reward, likelihood, or guidance potential $r_t$, we target the normalized tilted path $q_t(x)\propto (p_t^\bwd(x))^\gamma e^{r_t(x)}$ with $\gamma>0$. A typical realization, used in Section~\ref{sec:experiments}, is the linear ramp $r_t(x)=\beta_t\,r(x)$ for a target reward $r:\sX\to\sR$ and a scalar schedule $\beta_t\ge 0$ with $\beta_0=0$ at the noisy end (so the tilt is inactive at initialization) and $\beta_T=1$ at the data end (so $r_T=r$ recovers the target).

\begin{proposition}[Discrete Feynman--Kac representative for a tilted path~\cite{hasan2026discrete}]
\label{prop:debiased-discrete-dynamics}
Assume $p_t^\bwd(x)>0$ for all $x\in\sX$ and $t\in[0,T]$. Fix $\gamma>0$ and a differentiable potential $r_t:\sX\to\sR$. The normalized path $q_t(x)\propto (p_t^\bwd(x))^\gamma e^{r_t(x)}$ satisfies
\begin{equation}
\label{eq:tilt-fk-pde-disc}
\partial_t q_t(x)=\sum_{y\neq x}\big(\widetilde Q_t(x,y)q_t(y)-\widetilde Q_t(y,x)q_t(x)\big)+\widetilde g_t(x)q_t(x),
\qquad \widetilde g_t=\widetilde G_t-\E_{q_t}\widetilde G_t,
\end{equation}
where, for $x\neq y$, the guided jump rate and reweighting potential are given by
{\small\begin{equation}
\label{eq:tilt-fk-disc-trans-mat}
\widetilde Q_t(x,y) = \gamma Q_t^\bwd(x,y)\left(\frac{p_t^\bwd(y)}{p_t^\bwd(x)}\right)^{1-\gamma} e^{r_t(x)-r_t(y)},
\ \widetilde G_t(x) = \dot r_t(x) + \sum_{y\neq x} (\widetilde Q_t(y,x) - \gamma Q_t^\bwd(y,x)).
\end{equation}}
\end{proposition}

Proposition~\ref{prop:debiased-discrete-dynamics} gives the baseline Feynman--Kac representative that FluxLite later changes within an exact equivalence class; common reward-tilting and annealing reductions are listed in Appendix~\ref{app:special-cases}. It also reveals that $\widetilde\mQ_t$ has the same directed support as $\mQ_t^\bwd$: the tilt changes edge weights but introduces no new graph edges. A baseline Feynman--Kac Sequential Monte Carlo (SMC) sampler simulates particles under $\widetilde\mQ_t$, reweights them by $\widetilde g_t$, and returns a terminal weighted empirical law. 

\section{Methodology}
\label{sec:methodology}

We now develop \emph{FluxLite}. Starting from a generic Feynman--Kac representation, we show that the same target path can be represented by many pairs of jump dynamics and reweighting potentials. FluxLite chooses a representative with smaller residual reweighting variance while respecting the sparsity pattern of the backward process. We instantiate this principle as a local rule, \texttt{HEU} (Section~\ref{sec:heuristic}), and a basis-restricted QP, \texttt{Discrete Variance-Controlling Guidance} (\texttt{D-VCG}, Section~\ref{sec:basis}).

\subsection{Graph calculus}
\label{sec:reweighting-free}

The basic object is a divergence operator on a directed graph. Throughout this section, all identities are stated on the support of $q_t$; for readability we assume $q_t(x)>0$ for every $x\in\sX$ at the time under consideration. With the column convention, a perturbation $R_t(y,x)$ adds outflow from $x$ to $y$. We therefore define the graph divergence as signed net outflow per unit mass:
\begin{equation}
\label{eq:div-qt}
\Div_{q_t}\mR_t(x)
:=
\frac{1}{q_t(x)}
\sum_{y\neq x}
\big(
R_t(y,x)q_t(x)-R_t(x,y)q_t(y)
\big).
\end{equation}

Consider a target path $(q_t)_{t\in[0,T]}$ satisfying a generic Feynman--Kac CTMC equation
\begin{equation}
\label{eq:fk-ctmc}
\partial_t q_t(x)
=
\sum_{y\neq x}
\Big(
Q_t(x,y)q_t(y)-Q_t(y,x)q_t(x)
\Big)
+
q_t(x)\,g_t(x),
\end{equation}
where $\mQ_t$ is an off-diagonal rate family and $g_t$ is centered under $q_t$, i.e.\ $\E_{q_t}[g_t]=0$. Proposition~\ref{prop:debiased-discrete-dynamics} gives one concrete representative, $(\widetilde{\mQ}_t,\widetilde g_t)$, for the tilted path.

\paragraph{Flux equivalence.}
The Feynman--Kac representation~\eqref{eq:fk-ctmc} is invariant under a family of perturbations: any added jump flux can be compensated by an appropriate divergence term in the potential.

\begin{proposition}[Equivalence class]
\label{prop:equivalence-class}
Assume the path $(q_t)_{t\in[0,T]}$ satisfies~\eqref{eq:fk-ctmc} and $q_t(x)>0$ for all states under consideration. Let $\mR_t$ be any matrix such that $R_t(x,x)=0$ and
\[
Q_t'(y,x):=Q_t(y,x)+R_t(y,x)\ge 0,
\qquad\text{for all }x\neq y.
\]
Then the same path $(q_t)_{t\in[0,T]}$ also satisfies
\begin{equation}
\label{eq:equivalence-class}
\partial_t q_t(x)
=
\sum_{y\neq x}
\Big(
Q_t'(x,y)q_t(y)-Q_t'(y,x)q_t(x)
\Big)
+
q_t(x)\,g_t'(x),
\end{equation}
where $g_t'(x)=g_t(x)+\Div_{q_t}\mR_t(x)$.
Moreover, if $\E_{q_t}[g_t]=0$, then also $\E_{q_t}[g_t']=0$.
\end{proposition}

Thus each Feynman--Kac representation belongs to an equivalence class $[(\mQ_t,g_t)]$ indexed by all admissible flux reallocations. The point of proposal control is to exploit this non-uniqueness: rather than accepting the baseline representation delivered by the guided Feynman--Kac pair $[(\widetilde\mQ_t, \widetilde g_t)]$, we search inside the equivalence class for a representative with smaller reweighting variance. It is instructive to compare Proposition~\ref{prop:equivalence-class} with its continuous counterpart for diffusion models, restated as Proposition~\ref{prop:debiased-pde-cont-detailed} in Appendix~\ref{app:driftlite-related}. As shown in Table~\ref{tab:cont-disc-div}, each ingredient of the continuous theory has a well-defined discrete counterpart; we use these connections to improve inference-time proposal control in the following sections.

\begin{table}[H]
    \centering
    \caption{Continuous vs.\ discrete variance-control.}
    \label{tab:cont-disc-div}
    \begingroup
    \small
    \setlength{\tabcolsep}{3pt}
    \begin{tabular}{@{}p{0.12\linewidth}p{0.36\linewidth}p{0.48\linewidth}@{}}
    \toprule
     & Continuous diffusion ($\sX=\sR^d$) & Discrete CTMC ($|\sX| = D$) \\
    \midrule
    Generator
    &
    $\widetilde{\gL}_t^\ast\rho = -\nabla \cdot \big(\widetilde{\vv}_t \rho\big) + \frac{\alpha_t^\bwd{}^2}{2}\Delta \rho$
    &
    \small
    $(\widetilde{\gL}_t^\ast\rho)(x) =\sum_{y\neq x}\big(\widetilde Q_t(x,y)\rho(y)-\widetilde Q_t(y,x)\rho(x)\big)$
    \\[0.9em]
    Feynman--Kac form
    &
    \multicolumn{2}{c}{$\partial_t q_t = \widetilde{\gL}_t^\ast q_t + g_t\,q_t,\;\;\E_{q_t}[g_t]=0$}
    \\[0.9em]
    Control
    &
    drift perturbation $\vu_t:\sR^d\to\sR^d$
    &
    rate perturbation $\mR_t \in \R^{D \times D}$
    \\[0.6em]
    Divergence
    &
    \small
    $\Div_{q_t}\vu_t(x):=\frac{\nabla\cdot(q_t(x)\vu_t(x))}{q_t(x)}$
    &
    \small
    $\Div_{q_t}\mR_t(x):=\frac{\sum_{y\neq x}\big(R_t(y,x)q_t(x)-R_t(x,y)q_t(y)\big)}{q_t(x)}$
    \\[0.8em]
    Equivalence
    &
    $(\widetilde{\vv}_t,g_t)\sim(\widetilde{\vv}_t+\vu_t,\;g_t+\Div_{q_t}\vu_t)$
    &
    $(\widetilde{\mQ}_t,g_t)\sim(\widetilde{\mQ}_t+\mR_t,\;g_t+\Div_{q_t}\mR_t)$
    \\[0.8em]
    Optimality
    &
    $\min_{\vu_t}\Var_{q_t}\!\left[g_t+\Div_{q_t}\vu_t\right]$
    &
    $\min_{\mR_t}\Var_{q_t}\!\left[g_t+\Div_{q_t}\mR_t\right]$
    \\
    \bottomrule
    \end{tabular}
    \endgroup
\end{table}

\subsection{Variance control}

Following the same logic as in the continuous case~\citep{ren2025driftlite}, once a Feynman--Kac representation is fixed, the divergence identity of Proposition~\ref{prop:equivalence-class} allows us to reallocate flux from the residual potential into the proposal dynamics. The ideal objective is to eliminate reweighting altogether.

A useful reference point is obtained by turning off the jump dynamics and absorbing everything into the potential.

\begin{proposition}[\texttt{PR}: Pure reweighting]
\label{prop:pure-reweighting}
Assume $q_t(x)>0$ for all $x\in\sX$. Setting $\mQ_t'\equiv 0$ in Proposition~\ref{prop:equivalence-class} yields the \emph{pure-reweighting} representative
\begin{equation}
\label{eq:pure-reweighting}
\partial_t q_t(x)=q_t(x)\,g_t^0(x),
\quad
g_t^0(x)
:=
g_t(x)+
\frac{1}{q_t(x)}
\sum_{y\neq x}
\Big(
Q_t(x,y)q_t(y)-Q_t(y,x)q_t(x)
\Big).
\end{equation}
Moreover, $[(\mQ_t,g_t)]=[(\vzero,g_t^0)]$, and every representative in this equivalence class can be reconstructed from $(\vzero,g_t^0)$: for any candidate off-diagonal rate family $\mQ_t'$, the associated potential is
\begin{equation}
\label{eq:gprime-from-qprime}
g_t'(x)=g_t^0(x)+\Div_{q_t}\mQ_t'(x).
\end{equation}
\end{proposition}

The proof is in Appendix~\ref{app:proof-pure-reweighting}. This representative is typically uncompetitive for simulation because all changes in probability mass are represented through importance weights, but it is algebraically convenient: it serves as the canonical anchor of the equivalence class from which every other representative is obtained via~\eqref{eq:gprime-from-qprime}.

\paragraph{Variance objective.}
We therefore seek a representative $(\mQ_t',g_t')\in[(\mQ_t,g_t)]=[(\vzero,g_t^0)]$ that reduces weight degeneracy by shrinking the variance of the residual centered potential:
\begin{equation}
\label{eq:var-obj}
\min_{\mQ_t'}\;
\gV[\mQ_t']
:=
\Var_{x\sim q_t}\!\big[g_t'(x)\big]
=
\Var_{x\sim q_t}\!\big[g_t^0(x)+\Div_{q_t}\mQ_t'(x)\big].
\end{equation}
If the minimum equals zero, then $g_t'(x)\equiv 0$ and the target path is realized by a pure CTMC without any reweighting.

Unlike the continuous case, where the exact solution of the variance minimization requires solving a Poisson equation (\emph{cf.}~Proposition~\ref{prop:optimal-control-curl-free-cont} in Appendix~\ref{app:driftlite-related}), the zero-variance objective has an explicit solution in the discrete case when sparsity constraints are absent.
\begin{proposition}[\texttt{DEN}: Unconstrained zero-variance reallocation]
\label{prop:zero-variance-unconstrained}
Assume $q_t(x)>0$ for all $x\in\sX$ and that $g_t$ in~\eqref{eq:fk-ctmc} is centered. Let $D:=|\sX|$. Define, for $x\neq y$, the dense rate family
\begin{equation}
Q_t^*(y,x):=\tfrac{1}{D}\big[\tfrac{q_t(y)}{q_t(x)}g_t^0(y)-g_t^0(x)\big]_+,
\qquad Q_t^*(x,x):=0.
\end{equation}
Then $g_t^0(x)+\Div_{q_t}\mQ_t^*(x)=0$ for every $x$, so the target path is realized by the pure CTMC $\partial_t q_t(x)=\sum_{y\neq x}\big(Q_t^*(x,y)q_t(y)-Q_t^*(y,x)q_t(x)\big)$ without any reweighting. 
\end{proposition}

Proposition~\ref{prop:zero-variance-unconstrained} (proof in Appendix~\ref{app:proof-zero-variance-unconstrained}) identifies the ideal object but not a practical algorithm: the construction is dense and typically reallocates flux on the complete directed graph, whereas in discrete diffusion or DLMs the allowed jump graph is sparse and fixed by the pretrained reverse rate family. We therefore restrict attention to controlled rate families that preserve a prescribed sparsity pattern.

\paragraph{Sparse rates.}
Let $\gS^\spc\subseteq \sX\times \sX$ be the directed edge set of allowed transitions at time $t$. Any implementable off-diagonal rate family $\mQ_t'$ must satisfy
\[
Q_t'(y,x)=0,\quad \text{for}\ (y,x)\notin \gS^\spc,\quad \text{and} \quad Q_t'(y,x)\ge 0,\quad \text{for}\ (y,x)\in \gS^\spc.
\]
In practice, $\gS^\spc$ is taken to be the support of the pretrained reverse rate family.

\begin{proposition}[Sparse variance reduction as constrained weighted least squares]
\label{prop:nnls}
Assume $q_t(x)>0$ for all $x\in\sX$. Fix a sparsity pattern $\gS^\spc$ and define
$\mu_t(x):=q_t(x)\,g_t^0(x)$,
where $g_t^0$ is given by~\eqref{eq:pure-reweighting}. Then minimizing~\eqref{eq:var-obj} over all sparse off-diagonal rate families $\mQ_t'$ is equivalent to the constrained weighted least-squares problem
\begin{equation}
\label{eq:nnls}
\min_{Q_t'}
\;
\sum_{x\in \sX}
\frac{1}{q_t(x)}
\Big[
\mu_t(x)
+
\sum_{y:\,(y,x)\in \gS^\spc} Q_t'(y,x)\,q_t(x)
-
\sum_{y:\,(x,y)\in \gS^\spc} Q_t'(x,y)\,q_t(y)
\Big]^2.
\end{equation}
Equivalently, the objective depends only on the edge fluxes $Q_t'(y,x)q_t(x)$, i.e.\ on $\mQ_t'\diag(q_t)$.
\end{proposition}

This characterization makes the sparse-control problem precise: the dense ideal from Proposition~\ref{prop:zero-variance-unconstrained} corresponds to solving away the source term $\mu_t$ on the complete graph, while the implementable optimum is exactly the weighted least-squares problem~\eqref{eq:nnls} restricted to the sparse pattern. However, problem~\eqref{eq:nnls} is a large-scale nonnegativity-constrained weighted least-squares program with no closed-form solution: it requires iterative solvers (\emph{e.g.}, projected gradient) or coarse approximations, both of which add per-step computational cost at inference time. This motivates the practical, closed-form sparse controls developed in the next subsection.

\subsection{Practical sparse controls}
\label{sec:alg-options}

We now describe two practical constructions that produce a sparse off-diagonal rate family $\mQ_t^\eff$ and an associated residual centered potential
\begin{equation}
\label{eq:geff}
g_t^\eff(x):=g_t^0(x)+\Div_{q_t}\mQ_t^\eff(x),
\end{equation}
which is the quantity used for particle weighting in Feynman--Kac SMC. By Proposition~\ref{prop:equivalence-class}, every such choice yields an equivalent representation of the same target path.

\subsubsection{\texorpdfstring{\texttt{HEU}}{HEU}: one-hop local reallocation}
\label{sec:heuristic}

Recall that the dense zero-variance construction in Proposition~\ref{prop:zero-variance-unconstrained} can be written at the flux level as
\[
Q_t^*(y,x)\,q_t(x)
=
\frac{1}{D}\,[q_t(y)g_t^0(y)-q_t(x)g_t^0(x)]_+.
\]
This suggests a sparse local analogue: keep the same pairwise low-to-high transport rule, but restrict it to a local neighborhood around $x$ and replace the global normalization factor $D$ by a local effective neighborhood size.

\begin{proposition}[Local-average reallocation heuristic]
\label{prop:lar}
Assume $q_t(x)>0$ for all $x\in\sX$. Fix a sparsity pattern $\gS^\spc$, let $N_t^+(x):=\{y\in\sX\setminus\{x\}:(y,x)\in\gS^\spc\}$ be the allowed outgoing neighborhood from $x$, let $k_t(x)>0$ be a prescribed local normalization factor, and let $\alpha_t\in(0,1]$ be an optional damping parameter. Define the sparse off-diagonal rate family $\mQ_t^{\heu}$ by
\begin{equation}
\label{eq:lar-rate}
Q_t^{\heu}(y,x)
:=
\frac{\alpha_t}{k_t(x)\,q_t(x)}
\big[\mu_t(y)-\mu_t(x)\big]_+\,
\vone\{y\in N_t^+(x)\},
\end{equation}
with corresponding residual centered potential $g_t^{\heu}(x):=g_t^0(x)+\Div_{q_t}\mQ_t^{\heu}(x)$.
\end{proposition}

Proposition~\ref{prop:lar} is a sparse one-hop surrogate of the dense zero-variance formula: flux is sent only along allowed edges and only from smaller source terms $\mu_t(x)=q_t(x)g_t^0(x)$ to larger ones. Under reciprocal neighborhoods with matching local normalizations, the rule has the interpretation of replacing $\mu_t(x)$ by a damped average of nearby sources, so it removes the local fluctuation of the pure-reweighting source while leaving the local mean. In implementation, $\mu_t$ is evaluated via~\eqref{eq:pure-reweighting}; larger-neighborhood variants and mask-diffusion choices for $k_t(x)$ are deferred to Appendix~\ref{app:heu-discussion}.

\paragraph{Routine \texttt{HEU}.}
For each source state $x$, evaluate the local source terms $\mu_t(y)$ on the implemented neighborhood $N_t^+(x)$ using~\eqref{eq:pure-reweighting}, choose $k_t(x)$ according to one of the rules in Appendix~\ref{app:heu-discussion}, set the outgoing rates by~\eqref{eq:lar-rate}, and evaluate the residual centered potential $g_t^{\heu}$ via the general form~\eqref{eq:geff}.

\subsubsection{\texorpdfstring{\texttt{D-VCG}}{D-VCG}: discrete variance-controlling guidance}
\label{sec:basis}

The sparse least-squares optimum is still too large to solve on DLM-scale state spaces. \texttt{Discrete Variance-Controlling Guidance} (\texttt{D-VCG}) therefore restricts the controlled rate to a small nonnegative span of basis rates obtained by reweighting the pretrained reverse rate by nonnegative multipliers $\varphi_t^{(j)}$: for $j\in[J]$ and $x\neq y$,
\begin{equation}
\label{eq:Q-basis}
Q_t^{(j)}(y,x):=Q_t^\bwd(y,x)\,\varphi_t^{(j)}(y,x).
\end{equation}
The choice of multipliers $\{\varphi_t^{(j)}\}_{j=1}^J$ is open and can be tailored to the task; many bases beyond the two we keep in the main text are admissible, and Appendix~\ref{app:bases} lists the augmented library we designed for the Ising experiments. Throughout the paper we always include the \emph{untilted backward} basis $\varphi_t^{(1)}(y,x)\equiv 1$, and pair it with the \emph{target-aligned anchor}
\[
\varphi_t^{(2)}(y,x)
:=
\left(\frac{p_t^\bwd(x)}{p_t^\bwd(y)}\right)^{1-\gamma}
\exp\big(r_t(y)-r_t(x)\big),
\]
which combines annealing (exponent $\gamma$) and reward tilting (potential $r_t$); the pure-annealing and pure-tilt specializations, recovered by setting $r_t\!\equiv\!0$ or $\gamma\!=\!1$, are listed as their own bases in Appendix~\ref{app:bases}.

The controlled family $\mQ_t^\vcg(\vtheta)=\sum_{j=1}^J\theta_j\mQ_t^{(j)}$, $\vtheta\in\R_+^J$, is a valid off-diagonal rate family for every nonnegative coefficient vector. The residual centered potential is
\[
g_t^\vcg(x;\vtheta)=g_t^0(x)+\sum_{j=1}^J\theta_j\,\Div_{q_t}\mQ_t^{(j)}(x).
\]
Writing $\widetilde D_t^{(j)}(x):=\Div_{q_t}\mQ_t^{(j)}(x)-\E_{q_t}[\Div_{q_t}\mQ_t^{(j)}]$, the population objective reduces to a low-dimensional nonnegative least-squares problem.

\begin{proposition}[\texttt{D-VCG}]
\label{prop:disc-vcg}
The unconstrained minimizers of $\min_{\vtheta\in\R^J}\Var_{q_t}\big(g_t^\vcg(\cdot;\vtheta)\big)$ are exactly the solutions of the normal equations $\mA_t\vtheta=-\vc_t$, where
\begin{equation}
\label{eq:disc-vcg-system}
(\mA_t)_{ij}=\E_{q_t}\!\big[\widetilde D_t^{(i)}(x)\,\widetilde D_t^{(j)}(x)\big],
\qquad
(\vc_t)_i=\E_{q_t}\!\big[g_t^0(x)\,\widetilde D_t^{(i)}(x)\big],
\qquad i,j\in[J].
\end{equation}
If $\mA_t$ is nonsingular, the unconstrained minimizer is unique. With the constraint $\vtheta\ge0$, the same objective is a nonnegative quadratic program in $J$ variables.
\end{proposition}

\paragraph{Routine \texttt{D-VCG}.}
Given weighted particles approximating $q_t$, evaluate the basis divergences on the particle cloud and assemble the weighted system~\eqref{eq:disc-vcg-system}. We approximately solve the resulting nonnegative QP by active-set candidate enumeration; details of the small KKT solves and the Ising basis library are in Appendix~\ref{sec:experimental-details}. Let $\vtheta^\star$ denote the selected nonnegative coefficient vector. The effective rates and residual are $\mQ_t^\eff=\mQ_t^\vcg(\vtheta^\star)$ and $g_t^\eff=g_t^\vcg(\cdot;\vtheta^\star)$.

\paragraph{Feynman--Kac SMC recursion.}
On a grid $0=t_0<\cdots<t_M=T$, the sampler iterates four operations per step: construct $\mQ_{t_k}^\eff$ and $g_{t_k}^\eff$ by \texttt{HEU} or \texttt{D-VCG}, multiply weights by $\exp(\Delta t_k g_{t_k}^\eff)$, resample when the ESS ratio falls below $\tau$, and propagate particles over $[t_k,t_{k+1}]$ with the CTMC rate $\mQ_{t_k}^\eff$. Algorithm~\ref{alg:fk-smc} gives a first-order implementation skeleton; the experimental splitting scheme is specified in Appendix~\ref{app:toy-details}.

\begin{algorithm}[t]
\caption{FluxLite controlled Feynman--Kac SMC}
\label{alg:fk-smc}
\KwIn{time grid $0=t_0<\cdots<t_M=T$; particle count $N$; ESS threshold $\tau$; baseline Feynman--Kac pairs $\{(\mQ_{t_k},g_{t_k})\}_{k=0}^{M-1}$; control routine $\texttt{CONTROL}\in\{\texttt{HEU},\texttt{D-VCG}\}$.}
\KwOut{Weighted terminal particles approximating $q_T$.}
Initialize particles $x_{t_0}^{(n)}\sim q_{t_0}$ and weights $w_{t_0}^{(n)}\leftarrow 1/N$\;
\For{$k=0,\ldots,M-1$}{
    $\Delta t_k\leftarrow t_{k+1}-t_k$ and form the current weighted empirical approximation of $q_{t_k}$\;
    Compute the pure-reweighting source $g_{t_k}^0$ using~\eqref{eq:pure-reweighting}\;
    $(\mQ_{t_k}^\eff,g_{t_k}^\eff)\leftarrow\texttt{CONTROL}(q_{t_k}^{N},\mQ_{t_k},g_{t_k}^0)$ by \texttt{HEU} or \texttt{D-VCG}\;
    Update and normalize weights: $w^{(n)}\leftarrow w^{(n)}\exp\{\Delta t_k g_{t_k}^\eff(x_{t_k}^{(n)})\}$\;
    \If{the ESS ratio is below $\tau$}{resample particles and reset weights to $1/N$}
    Propagate each particle from $t_k$ to $t_{k+1}$ with CTMC rates $\mQ_{t_k}^\eff$\;
}
\end{algorithm}

\section{Theoretical Analysis}
\label{sec:theoretical-analysis}

The analysis separates two sources of error. The first is a \emph{population} error: the deployed discrete diffusion uses an estimated local ratio $\widehat s_t$ rather than the exact ratio $s_t$, so the guided Feynman--Kac operator itself is perturbed. The second is a \emph{particle} error: even for a fixed controlled Feynman--Kac representation, SMC approximates the corresponding grid-level Feynman--Kac recursion using finitely many particles. These results are intentionally modular. They do not claim a full end-to-end bound for adaptive controls estimated from the same particle cloud; the scope of each statement is made explicit below and in Remark~\ref{rem:not-bound}.

\paragraph{Population stability under local-ratio error.}
Let $q_t^{\widehat s}$ denote the exact population normalized Feynman--Kac flow obtained by replacing $s_t$ with $\widehat s_t$ in the reverse rates, guided rates, and potentials of Proposition~\ref{prop:debiased-discrete-dynamics}; it is not the empirical particle law. Write $u_t(x,y):=\widehat s_t(x,y)/s_t(x,y)$. For source-weighted score-entropy training~\citep{lou2024discrete}, set $\nu_t^{\rm SE}(x,y):=p_t^\bwd(x)Q_t^\bwd(y,x)$ and define the time-integrated score-entropy training loss
\begin{equation}
\label{eq:dd-training-loss-main}
\mathfrak L_{\rm DD}:=\int_0^T\mathcal L_{\rm DD}(t)\,\dif t,
\quad
\mathcal L_{\rm DD}(t):=\sum_{x\in\sX}\sum_{y\neq x}\nu_t^{\rm SE}(x,y)\,\ell\big(u_t(x,y)\big),
\quad \ell(u):=-\log u-1+u.
\end{equation}
Up to a constant, $\mathfrak L_{\rm DD}$ upper-bounds the model--data KL~\citep{lou2024discrete,ren2024discrete}. Let $h_t(x):=q_t(x)/p_t^\bwd(x)$; the theorem weights this coverage ratio edgewise rather than replacing it by a worst-case supremum.

\begin{theorem}[Population score-ratio stability, informal]
\label{thm:ratio-error}
Assume the known-forward setting and the hypotheses of Theorem~\ref{thm:ratio-error-formal}, and let $\mathfrak C_{\rm DD}$ be the time-integrated edge-coverage scalar defined in~\eqref{eq:Cdd-integrated-app}. Then
\begin{equation}
\label{eq:main-training-loss-bound}
\TV(q_T^{\widehat s},q_T)
\le
\gamma\, e^{2BT}\,\mathfrak C_{\rm DD}\,\sqrt{\mathfrak L_{\rm DD}}.
\end{equation}
The exponent $e^{2BT}$ comes from the $\ell^1$ logarithmic norm of the implemented Metzler operator. If $\mathfrak L_{\rm DD}=0$ then $q_T^{\widehat s}=q_T$.
\end{theorem}

The factor $\mathfrak C_{\rm DD}$ is a \emph{coverage} scalar, not a score-accuracy condition: inference-time scaling with a fixed pretrained score is robust only when the tilted path remains covered by the original reverse path on trained edges. If reward or temperature scaling pushes $q_t$ toward regions where $h_t$ is large at edge endpoints carrying $\nu_t^{\rm SE}$ mass, the same training loss $\mathfrak L_{\rm DD}$ can produce a larger terminal bias.

\paragraph{Particle error for a fixed controlled representation.}
\label{sec:many-particle}
Now fix a grid $0=t_0<\cdots<t_M=T$ and deterministic controlled rates/potentials $(\mQ_{t_k}^{\eff},g_{t_k}^{\eff})$. Let $\mP_k:=\exp(\Delta t_k\mL_{t_k}^{\eff})$ and $W_k(x):=\exp(\Delta t_k g_{t_k}^{\eff}(x))$. With the column convention, $P_k(y,x)$ is the transition probability from source $x$ to destination $y$, and $(\mathsf P_k f)(x):=(\mP_k^\top f)(x)=\sum_yP_k(y,x)f(y)$. Define
\begin{equation}
\label{eq:main-fixed-grid-recursion}
q_0^\Delta=q_{t_0},\qquad
q_{k+1}^\Delta(f)=\frac{q_k^\Delta(W_k\mathsf P_kf)}{q_k^\Delta(W_k)},
\qquad k=0,\ldots,M-1.
\end{equation}
This theorem analyzes the every-step bootstrap version of Algorithm~\ref{alg:fk-smc}. It isolates the Monte Carlo error of a fixed Feynman--Kac model; adaptive ESS resampling, time discretization, score error, and feedback from estimating controls on the same particles are outside the statement and are discussed in Remark~\ref{rem:not-bound}.

\begin{theorem}[Particle approximation of the fixed grid Feynman--Kac recursion]
\label{thm:many-particle-convergence}
Under the deterministic fixed-grid setting of Assumption~\ref{assump:fixed-grid-fk}, let $q_M^{N,\Delta}$ be the empirical law produced after $M$ bootstrap Feynman--Kac SMC steps from $N$ particles. For every bounded $f:\sX\to\sR$,
\begin{equation}
\E\big[|q_M^{N,\Delta}(f)-q_M^\Delta(f)|\big]
\le
\frac{\osc(f)}{2\sqrt N}\sum_{k=0}^{M}\bar\beta_{k,M},
\quad
\bar\beta_{k,M}:=\prod_{j=k}^{M-1}e^{\Delta t_j\osc(g_{t_j}^{\eff})}.
\end{equation}
\end{theorem}

This theorem explains why FluxLite targets residual potentials: its constant depends on residual oscillation, and Lemma~\ref{lem:incremental-weight-variance} gives $\Var_{q_k^\Delta}(\bar W_k)\le \Delta t_k^2e^{2\Delta t_k\norminf{\bar g_k}}\Var_{q_k^\Delta}(g_{t_k}^{\eff})$. The proof (Appendix~\ref{app:proof-many-particle}) uses conditional i.i.d.\ resampling--mutation and a nonlinear telescoping decomposition, separating fixed-model Monte Carlo error from the population score-ratio perturbation in Theorem~\ref{thm:ratio-error}.

\section{Experiments}
\label{sec:experiments}

We evaluate our method using an exact finite-state CTMC benchmark and a learned 2D Ising benchmark. Across both, \texttt{D-FKC} is our Feynman--Kac corrector baseline and \texttt{PG} is guided propagation without reweighting; \texttt{PR}, \texttt{HEU}, \texttt{D-VCG}, and \texttt{DEN} are the samplers from Propositions~\ref{prop:pure-reweighting},~\ref{prop:lar},~\ref{prop:disc-vcg}, and~\ref{prop:zero-variance-unconstrained}, respectively. Additional large-scale experiments on a discrete inverse problem (monophonic music infilling) and on text-to-image generation under classifier-free guidance are reported in Appendix~\ref{app:large-scale}.

\subsection{Finite-state CTMC}
\label{sec:toy}

The benchmark combines reward tilting and annealing with tensor-product \emph{uniform-state} and \emph{masked-absorbing} CTMC diffusions at $V=5$, $L=3$. The uniform diffusion lives on $V^L=125$ states; the masked diffusion evolves on $(V+1)^L=216$ states including the mask token and has an unmasked terminal target on $V^L$ states. \texttt{PR} uses $(0,g_t^0)$ and is omitted in the masked setting because reweighting cannot introduce states absent from the initial particle cloud.

\begin{figure}[!htbp]
    \centering
    \includegraphics[width=.98\textwidth]{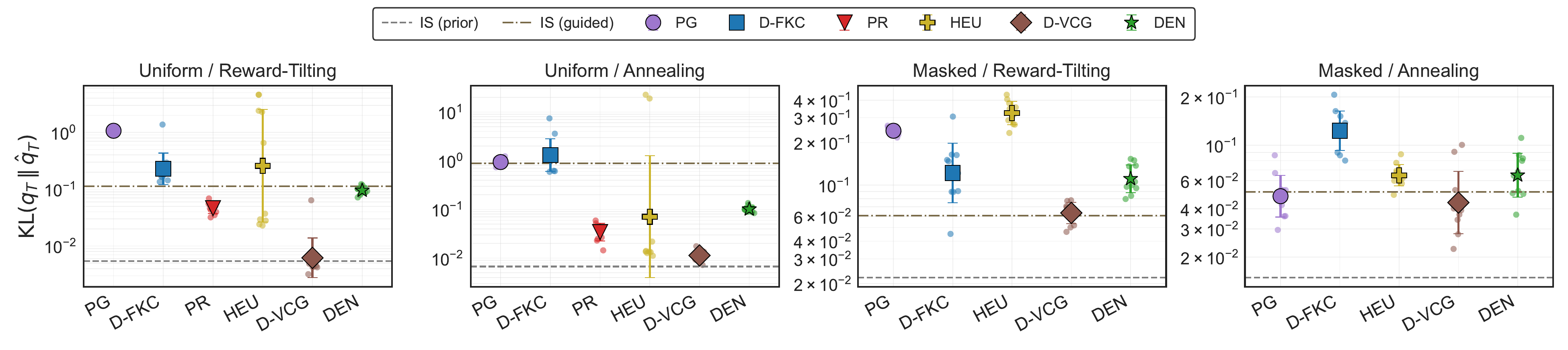}
    \vskip -.5em
    \caption{Finite-state CTMC benchmark: terminal $\KL(q_T\|\widehat q_T^N)$ for four $V=5$, $L=3$ regimes with $N=4000$ particles, 80 integration steps, and 10 seeds. \texttt{D-VCG} improves over \texttt{D-FKC} (geometric-mean ratio of terminal KL across seeds) up to $114.8\times$; settings and sweeps are in Appendix~\ref{sec:experimental-details}.}
    \label{fig:toy-row}
  \end{figure}

Figure~\ref{fig:toy-row} reports terminal $\KL(q_T\|\widehat q_T^N)$. \texttt{D-VCG} reallocates flux toward useful transitions while keeping residual weights controlled, whereas \texttt{PG} lacks correction. \texttt{D-FKC} produces single-seed blow-ups at large $\gamma$ on uniform-state/annealing. \texttt{D-VCG} matches or improves on the population zero-variance oracle \texttt{DEN} across all four cells while operating only on the sparse graph implementable in DLM-scale state spaces. Appendix~\ref{app:toy-details} gives sweeps and implementation details.

\subsection{2D Ising Model}
\label{sec:ising}

We evaluate a score-based discrete-diffusion model trained on Swendsen--Wang samples at $\beta_{\rm train}=0.4$ on a $16\times16$ periodic Ising lattice ($J_{\rm Ising}=1$). Figure~\ref{fig:ising} sweeps target inverse temperature and magnetization reward $r(\sigma)=M(\sigma)=\sum_i\sigma_i$ for $q_{\beta,\beta_r}(\sigma)\propto e^{-\beta H(\sigma)+\beta_r M(\sigma)}$, with a target-specific SW reference and a ghost spin for nonzero fields. The figure labels ``2-basis'' and ``4-basis'' denote the minimal and augmented controller variants; their regime-dependent libraries are specified in Appendix~\ref{app:bases}.

\begin{figure}[!htbp]
    \centering
    \begin{subfigure}[t]{0.48\textwidth}
      \includegraphics[width=\linewidth]{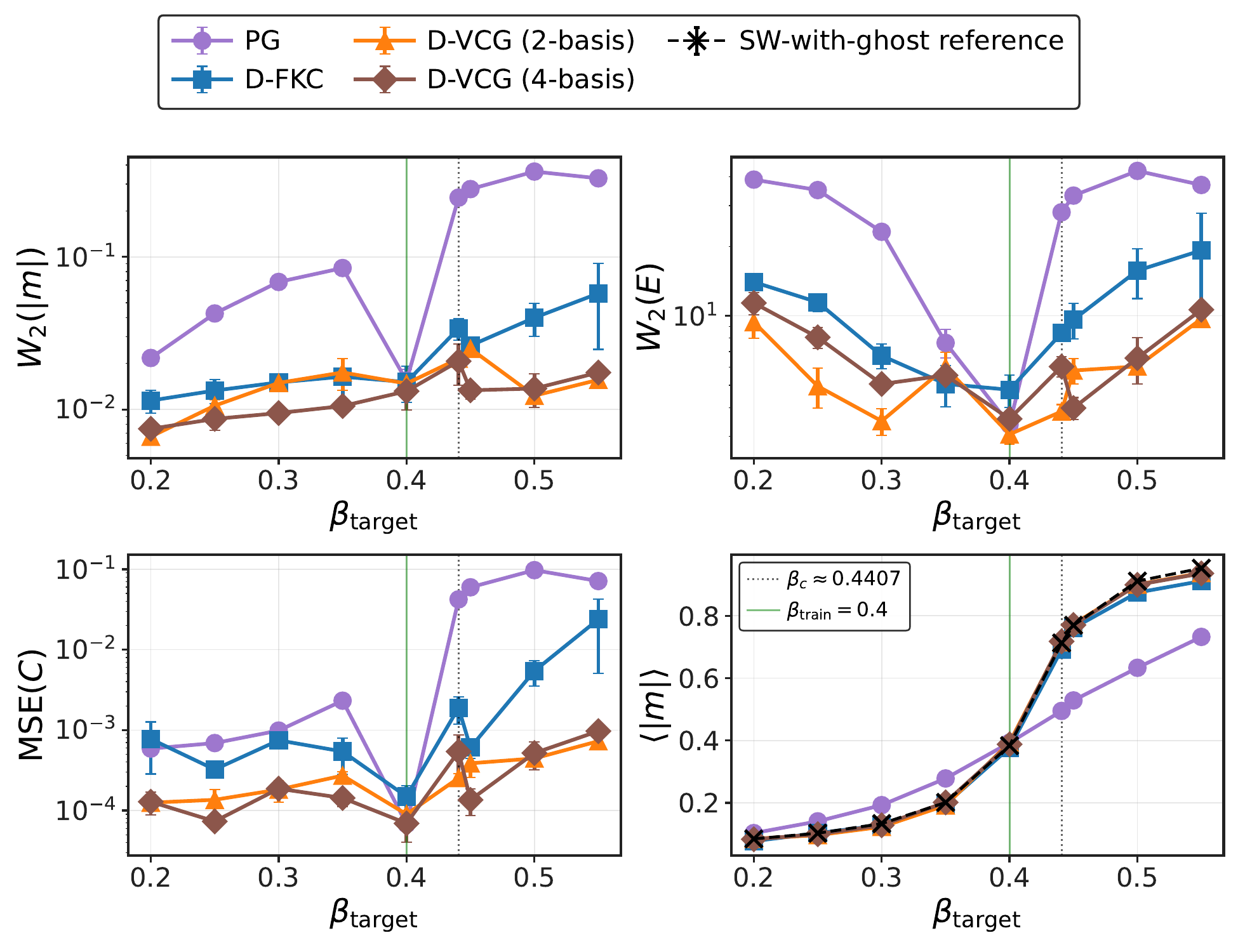}
      \vskip -.5em
      \caption{Pure annealing: 4-basis \texttt{D-VCG} cuts $\mathrm{MSE}(C(r))$ by $5.33\times$ in geometric mean over \texttt{D-FKC}, peak $\mathbf{24.5\times}$.}
      \label{fig:ising-anneal}
    \end{subfigure}
    \hfill
    \begin{subfigure}[t]{0.48\textwidth}
      \includegraphics[width=\linewidth]{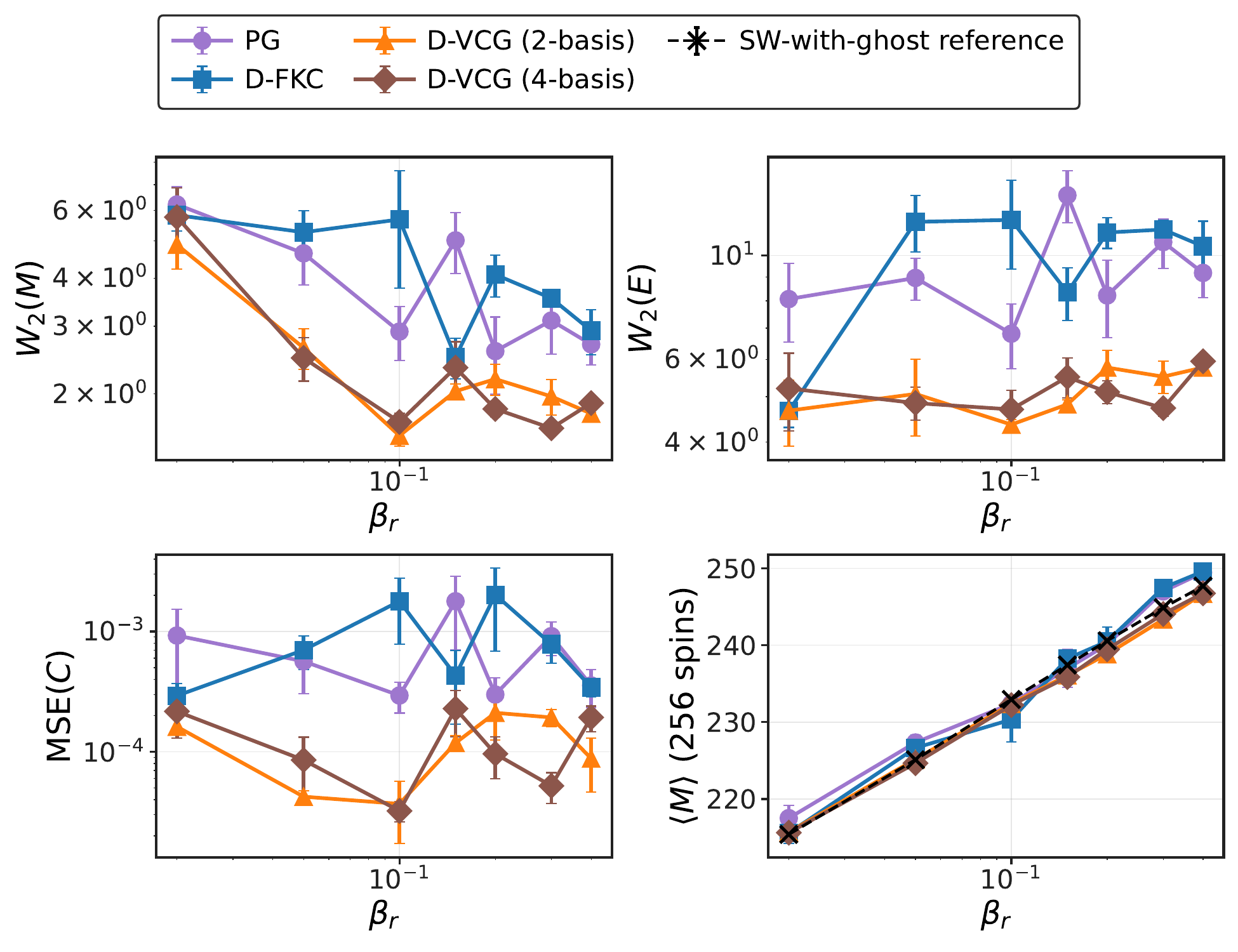}
      \vskip -.5em
      \caption{Joint anneal\,$+$\,tilt: 4-basis \texttt{D-VCG} cuts $\mathrm{MSE}(C(r))$ by $6.77\times$ in geometric mean over \texttt{D-FKC}, peak $\mathbf{55.4\times}$.}
      \label{fig:ising-tilt}
    \end{subfigure}
    \caption{2D Ising experiments against a SW reference, with a ghost spin for nonzero fields. Panels report magnetization/energy-$\Wass_2$, row-correlation $\mathrm{MSE}(C(r))$, and mean magnetization for \texttt{PG}, \texttt{D-FKC}, and 2-/4-basis \texttt{D-VCG}.}
    \label{fig:ising}
\end{figure}

The largest gains appear in row correlations: \texttt{D-VCG} reduces their MSE over \texttt{D-FKC} by $5.33\times$ in geometric mean on annealing and $6.77\times$ on joint anneal plus tilt. The per-cell peaks in Fig.~\ref{fig:ising} are $24.5\times$ and $55.4\times$, with improvements across the plotted annealing and joint-tilt sweeps. A particle-count ablation (Fig.~\ref{fig:ising-N-ablation}) confirms the gap does not close as $N$ grows, with a $15.21\times$ geometric-mean and $34.7\times$ peak advantage persisting from $N=100$ to $N=5000$. Feynman--Kac-reweighted methods track SW means. Appendix~\ref{app:ising-details} gives exact configurations and ablations.

\section{Discussion and Conclusion}
\label{sec:discussion}

FluxLite is a discrete counterpart of DriftLite for Feynman--Kac SMC in discrete diffusion models and DLMs. Its core is an exact CTMC--Feynman--Kac equivalence class on the pretrained transition graph: sparse rate perturbations reallocate flux, and a $q_t$-weighted graph divergence compensates the residual potential without changing tilted marginals. \texttt{HEU} and \texttt{D-VCG} instantiate this variance-control principle; the theory ties score-ratio loss to guided Feynman--Kac bias and residual-potential oscillation to fixed-grid particle error, while CTMC and Ising experiments show large gains over \texttt{D-FKC}. Future directions include hybrid discrete--continuous controls, alignment/preference/pretraining couplings, and distillation into transport maps or parallel-reasoning architectures. On the theoretical side, natural extensions include adaptive ESS, particle-dependent controls, joint score and time-discretization error, and richer learned bases.

\section*{Acknowledgments}

Y.R. and J.H. thank the Scientific Computing Core at the Flatiron Institute, a division of the Simons Foundation, for providing computational resources and support. This work has been made possible in part by a gift from the Chan Zuckerberg Initiative Foundation to establish the Kempner Institute for the Study of Natural and Artificial Intelligence at Harvard University. L.Y. acknowledges support by the National Science Foundation under Award No. DMS-2208163. This material is based upon work supported by the National Science Foundation under Grant No. CHE-2441297 to G.M.R.

\bibliographystyle{plainnat}
\bibliography{reference}

\newpage
\appendix

\section{Related Work}
\label{app:related-notation-further}

\subsection{Markov generative models}
\label{app:related-genai}

Modern generative modeling is shaped by two large Markovian families. Autoregressive LLMs generate discrete sequences through one-step conditional transitions~\citep{radford2019language,brown2020language,achiam2023gpt,gemini2023gemini,touvron2023llama}. Diffusion and flow-based models generate through iterative stochastic or deterministic dynamics~\citep{sohl2015deep,zhang2018monge,song2019generative,ho2020denoising,song2020improved,song2021score,song2021denoising,song2021maximum,liu2023flow,lipman2023flow,albergo2023building,albergo2025stochastic}. Discrete diffusion and DLMs connect these viewpoints by placing diffusion-style reverse dynamics on finite or countable state spaces. Representative examples include diffusion-LM~\citep{li2022diffusion}, BERT-style masking models~\citep{devlin2019bert,chang2022maskgit}, discrete diffusion models~\citep{lou2024discrete,nie2025scaling}, and recent DLMs~\citep{sahoo2024simple,nie2025large,bie2025llada2,song2025seed,labs2025mercury,deepmind2025geminidiffusion,chen2026multimask}. FluxLite focuses on the CTMC structure of these reverse dynamics and on how a Feynman--Kac representation can be controlled without changing the pretrained graph. Broader background on generator matching, mixed state spaces, and discrete diffusion surveys appears in~\citep{holderrieth2025generator,ren2025unified,rojas2025diffuse,chen2024accelerating,ren2024discrete,ren2025fast}.

\subsection{Inference-time scaling}
\label{app:related-inference-time}

\paragraph{Length and width.}
Existing inference-time scaling methods can be grouped into two categories based on how they allocate additional compute at inference time. The first class is \emph{length scaling}, which improves the quality of the output by elongating the response or iteratively refining it~\citep{wei2022chain,wang2023self,yao2023tree}. The second class is \emph{width scaling}, which expands the hypothesis space by exploring a larger set of candidate trajectories via search-based or sampling-based methods~\citep{brown2024large,snell2025scaling,ding2025dynamic,wu2025inference,puri2025rollout,uehara2025inference,jain2025diffusion,ramesh2025test,inoue2025wider,tang2025tr2,zhang2026inference,bai2026prism}. Within sampling-based width scaling, current work casts the problem as drawing samples from a tilted version of the distribution associated with the base model.

\paragraph{Domains.}
The \emph{inference-time scaling} paradigm has been extended from text generation to a wide range of domains: reasoning~\citep{balachandran2025inference,yang2025table,wang2025think}, code generation~\citep{li2025s,deng2025scalertl}, AI for science~\citep{yang2025steering,fan2026physics,serrano2026test,li2026robust}, and video~\citep{liu2025video,yang2025scalingnoise} or image generation~\citep{ma2025scaling,chen2025solving,li2025reflect}. The inference-time scaling literature for diffusion models specifically~\citep{uehara2025inference,ma2025scaling,chen2025solving,ren2025driftlite,hasan2026discrete} forms the most direct context for our work.

\paragraph{SMC methods.}
SMC-based methods have recently gained increasing attention for inference-time scaling on distributions of both continuous~\citep{singhal2025general,chen2025solving,kim2025test,skreta2025feynman,yoon2025psi,he2026rne,wei2026surge,wang2026simple} and discrete~\citep{lew2023sequential,lew2023smcp3,loula2025syntactic,kim2025hypothesis,hasan2026discrete,azizi2026power} type. Theoretical analyses of LLM inference-time compute~\citep{brown2024large,snell2025scaling,muennighoff2025s1,setlur2025scaling,huang2025best,huang2026sample,di2026best,yu2025limits,zhu2026power,puri2025rollout} and the classical theory of SMC~\citep{schweizer2012non,schweizer2012phdthesis,chopin2002sequential,moral2004feynman,del2006sequential,doucet2009tutorial,del2013mean} provide further context. Applications include LLMs and math reasoning~\citep{lew2023sequential,zhao2024probabilistic,feng2025step,singhal2025general,zhao2025d1}, inference-time alignment of continuous and discrete diffusion models~\citep{uehara2025inference,ou2026inference,holderrieth2026glass}, sampling from unnormalized densities~\citep{wu2025reverse}, unified dynamics~\citep{skreta2025feynman} and Bayesian inverse problems~\citep{chen2025solving}, model composition~\citep{du2023reduce,skreta2025superposition}, and inference-time alignment in biology~\citep{li2026robust}. See also the stochastic weighted particle method~\citep{degond1989weighted,degond1990deterministic,rjasanow1996stochastic,bossy1997stochastic,talay2003stochastic,raviart2006analysis,chertock2017practical}, discrete posterior sampling~\citep{chu2025split,luan2025ddps,rout2025test,han2025discrete}, and inference-time scaling of DLLMs via Gibbs sampling~\citep{dang2025inference}.

\paragraph{Guidance and search.}
Closely related to inference-time tilting is the literature on continuous guidance~\citep{ye2024tfg}, discrete guidance~\citep{vignac2023digress,gruver2023protein,li2024derivative,kerby2024training,nisonoff2025unlocking,schiff2025simple,xiong2025guide}, and theoretical analyses of both continuous and discrete guidance~\citep{wu2024theoretical,chidambaram2024does,bradley2024classifier,pavasovic2025classifier,karan2025reguidance,yang2025diffusion,han2025stochastic,tang2024stochastic,azangulov2025adaptive,rojas2026improving,he2025exactly}. Discrete Feynman--Kac formulas~\citep{zoia2012discrete,bressloff2017feynman,del2013feynman}, discrete tree search~\citep{tang2025tr2,bai2026prism}, continuous tree sampling/search~\citep{jain2025diffusion,ramesh2025test,zhang2026inference}, and search-based methods in LLMs~\citep{inoue2025wider} are also relevant.

\paragraph{Closest concurrent work.}
The closest discrete Feynman--Kac papers are Discrete Feynman--Kac correctors and debiasing guidance methods for discrete diffusion~\citep{hasan2026discrete,lee2025debiasing}. These works use particle corrections to target tilted discrete laws and are therefore the nearest methodological neighbors. FluxLite differs in the control object: it changes CTMC proposal rates inside an exact Feynman--Kac equivalence class and compensates the change by a graph-divergence potential. The target preservation is algebraic, while the approximation lies in the low-dimensional variance-control problem used to choose the representative. In contrast, posterior-prediction and steering methods for masked discrete diffusion~\citep{rector2025steering,wang2025fine,eyring2025noise} primarily modify denoising predictions or guidance rules, and hypernetwork approaches~\citep{xie2026hyperalign,ha2017hypernetworks} amortize adaptation into additional parameters. Continuous proposal-control and density-estimation methods~\citep{ren2025driftlite,he2026rne,ou2026inference} share the goal of reducing SMC degeneracy, but their controls act through continuous drifts or learned continuous proposals rather than sparse graph flux.

\subsection{Fine-tuning and alignment}

Training-time alignment methods modify model parameters or auxiliary guidance networks, for example through RL-style objectives for LLMs~\citep{zhu2026flowrl}, DPO and diffusion-DPO variants~\citep{rafailov2023direct,wallace2024diffusion}, direct diffusion optimization~\citep{zheng2025direct}, and learned guidance~\citep{galashov2025learn}. A related line learns samplers, proposals, or control potentials for Monte Carlo and diffusion dynamics~\citep{gabrie2022adaptive,cao2022learning,albergo2024learning,ou2025discrete,zhu2025mdns,guo2026proximal,so2026discrete,fan2023optimizing,havens2025adjoint}, including Doob-$h$-transform and bridge-based constructions~\citep{bu2025post,chang2026inference,zhu2026training,chun2025dynamic}. FluxLite is complementary: it keeps the pretrained discrete diffusion model fixed and optimizes only the inference-time Feynman--Kac proposal/reweighting representation on the same transition graph.

\subsection{Broader future directions}
The condensed main-text discussion points to several adjacent directions: hybrid discrete--continuous controls for multimodal generation and reasoning~\citep{zhou2025transfusion,rojas2025diffuse,chen2024expanding,lin2025investigating,yan2025videochat,xin2025dmllm,li2026lavida}; combinations with alignment, preference optimization, and pretraining biases~\citep{lin2025parm,annadani2025preference,song2025ideas}; and distillation into continuous or categorical transport maps and parallel-reasoning architectures~\citep{sabour2025test,roos2026categorical,holderrieth2026diamond,mahankali2026divide}.

\section{Continuous Proposal Control (DriftLite)}
\label{app:driftlite-related}

Inspired by twisted SMC methods in computational statistics~\citep{briers2010smoothing,whiteley2014twisted,heng2020controlled,lawson2022sixo,lu2024guidance}, several recent works~\citep{zhao2024probabilistic,feng2025step,albergo2025nets,holderrieth2025leaps,ou2026inference} explore proposal control for SMC-based inference-time scaling in diffusion models and LLMs. Among them, DriftLite~\citep{ren2025driftlite} provides a mathematically grounded construction for continuous diffusion models. Since our discrete theory only needs the structural identity, we record a compact version here.

\paragraph{Continuous setup.}
For continuous state space $\sX=\sR^d$, a pretrained continuous diffusion model is characterized by a forward and backward diffusion process, written as the following pair of stochastic differential equations (SDEs) for $s,t\in[0,T]$:
\begin{equation}
\dif \vx_s^\fwd = \vv_s^\fwd(\vx_s^\fwd)\,\dif s + \alpha_s^\fwd\,\dif \vw_s,
\qquad
\dif \vx_t^\bwd = \vv_t^\bwd(\vx_t^\bwd)\,\dif t + \alpha_t^\bwd\,\dif \vw_t,
\end{equation}
where $(\vw_t)_{t\in[0,T]}$ is the standard Brownian motion on $\sR^d$ and the backward drift is given by $\vv_t^\bwd:=-\vv_{T-t}^\fwd+\frac{\alpha_{T-t}^\fwd{}^2+\alpha_t^\bwd{}^2}{2}\nabla\log p_t^\bwd$, with $\nabla\log p_t^\bwd$ the score function (typically approximated by a pretrained neural network). Many inference-time tasks can be expressed as sampling from the tilted path
\[
q_t(\vx)\propto (p_t^\bwd(\vx))^\gamma e^{r_t(\vx)},
\qquad
\vx\in\sR^d,\ t\in[0,T],
\]
where $\gamma>0$ is fixed and $r_t:\sR^d\to\sR$ is a time-dependent reward, guidance, or steering potential. A principled way to simulate this path is to represent it as a normalized Feynman--Kac flow and then run SMC~\citep{singhal2025general,chen2025solving,skreta2025feynman}. The full Feynman--Kac structure is captured by the proposition below; in particular, equation~\eqref{eq:tilt-fk-pde-cont-detailed} below corresponds to the compact normalized Feynman--Kac equation of Section~\ref{subsec:smc-cont}.

\begin{proposition}[Continuous Feynman--Kac control identity~\protect{\cite[Proposition~2.1]{ren2025driftlite}}]
\label{prop:debiased-pde-cont-detailed}
Fix $\gamma>0$ and a time-dependent function $r_t:\sR^d\to\sR$. Then the tilted path $(q_t)_{t\in[0,T]}$ admits a Feynman--Kac representation
\begin{equation}
\label{eq:tilt-fk-pde-cont-detailed}
\partial_t q_t(\vx)
=
(\widetilde{\gL}_t^\ast q_t)(\vx)
+
\widetilde g_t(\vx)\,q_t(\vx),
\qquad
\widetilde g_t=\widetilde G_t-\E_{\vz\sim q_t}[\widetilde G_t(\vz)],
\end{equation}
where
\begin{equation}
(\widetilde{\gL}_t^\ast \rho)(\vx)
:=
-\nabla\cdot\big(\widetilde{\vv}_t(\vx)\rho(\vx)\big)
+
\frac{\alpha_t^\bwd{}^2}{2}\Delta \rho(\vx),
\end{equation}
for an appropriate guided drift $\widetilde{\vv}_t$, and where $\widetilde G_t$ is the corresponding unnormalized potential (see~\citep{ren2025driftlite} for the explicit formula). Moreover, for any control drift $\vu_t:\sR^d\to\sR^d$, the same $(q_t)_{t \in [0, T]}$ also satisfies
\begin{equation}
\partial_t q_t(\vx)
=
-\nabla \cdot \big((\widetilde{\vv}_t(\vx)+\vu_t(\vx))q_t(\vx)\big)
+
\frac{\alpha_t^\bwd{}^2}{2}\Delta q_t(\vx)
+
\big(\widetilde g_t(\vx)+\Div_{q_t}\vu_t(\vx)\big)q_t(\vx),
\end{equation}
where
\begin{equation}
\Div_{q_t}\vu_t(\vx)
:=
\frac{1}{q_t(\vx)}\,\nabla \cdot \big(q_t(\vx)\vu_t(\vx)\big).
\end{equation}
\end{proposition}

The ideal zero-variance continuous control solves the weighted Poisson equation.

\begin{proposition}[Optimal continuous control in DriftLite~\protect{\cite[Proposition~3.2]{ren2025driftlite}}]
\label{prop:optimal-control-curl-free-cont}
Assume $(q_t)_{t\in[0,T]}$ satisfies~\eqref{eq:tilt-fk-pde-cont-detailed}. Under suitable regularity assumptions, there exists a curl-free optimal control $\vu_t^\ast=\nabla \cA_t^\ast$ such that
\[
\widetilde g_t+\Div_{q_t}\vu_t^\ast\equiv 0,
\]
where $\cA_t^\ast$ solves the weighted Poisson equation
\begin{equation}
\label{eqn:variance-minimization-cont-poisson-pde}
\Div_{q_t}\nabla \cA_t^\ast(\vx)
=
\frac{1}{q_t(\vx)}\nabla \cdot \big(q_t(\vx)\nabla \cA_t^\ast(\vx)\big)
=
-\widetilde g_t(\vx).
\end{equation}
Equivalently,
\[
\vu_t^\ast
=
\argmin_{\vu_t}\Var_{\vx\sim q_t}\!\big[\widetilde g_t(\vx)+\Div_{q_t}\vu_t(\vx)\big],
\]
or, in potential form,
\[
\cA_t^\ast
=
\argmin_{\cA_t}\;
\int_{\sR^d}
\left(
\frac{1}{2}q_t(\vx)\|\nabla \cA_t(\vx)\|_2^2
-
q_t(\vx)\widetilde g_t(\vx)\cA_t(\vx)
\right)\dif \vx.
\]
\end{proposition}

A detailed proof can be found in~\citep{ren2025driftlite}. Since both $\vu_t$ and $\cA_t$ are high-dimensional, directly solving~\eqref{eqn:variance-minimization-cont-poisson-pde} is usually intractable. DriftLite therefore restricts the search to finite-dimensional subspaces such as
\[
\Span\{\nabla r_t,\nabla \log p_t^\bwd,\vv_t^\bwd\}
\qquad\text{and}\qquad
\Span\{r_t,\log p_t^\bwd,U_t^\bwd\},
\]
where $\nabla U_t^\bwd=\vv_t^\bwd$.

\section{Notation, Special Cases, and Proofs for Sections~\ref{sec:preliminaries}--\ref{sec:methodology}}
\label{app:proofs-props}

This appendix collects the notation and algebra supporting Sections~\ref{sec:preliminaries} and~\ref{sec:methodology}. It first fixes the column-vector convention and records the two common tilted-path reductions, then proves the discrete Feynman--Kac identity and the FluxLite variance-control propositions, and closes with further details on \texttt{HEU}.

\subsection{Notation}

The discrete state space is denoted by $\sX$. For any function $f:\sX\to\sR$, we write $\|f\|_\infty:=\max_{x\in\sX}|f(x)|$. For any distribution $\pi$ on $\sX$, expectations are denoted $\E_{\pi}[\cdot]$. For any off-diagonal rate family $\mQ$, we set $Q(x,x)=0$ by convention. Whenever a full generator matrix is needed, we use the associated matrix $\mL$ defined by
\[
L(y,x)=Q(y,x)\ \ (y\neq x),
\qquad
L(x,x)=-\sum_{y\neq x}Q(y,x).
\]
For a matrix $\mA\in \sR^{D\times D}$, $\normcol{\mA}:=\max_{x\in\sX}\sum_{y\in\sX}|A(y,x)|$ is the induced $\ell_1$ operator norm (maximum absolute column sum). All rate matrices and generators use this column convention throughout.

\subsection{Special cases of the tilted Feynman--Kac representative}
\label{app:special-cases}

The following two special cases are used repeatedly in the paper; the boundedness and positivity assumptions needed for the theoretical statements are made explicit in Appendix~\ref{app:assumption-blocks}.
\begin{itemize}[leftmargin=1.2em,topsep=2pt,itemsep=2pt,parsep=0pt]
    \item \textbf{Reward-tilted path} ($q_t \propto p_t^\bwd e^{r_t}$): for a possibly time-dependent reward $r_t:\sX\to\sR$,
    \begin{equation}
        \widetilde Q_t(x,y) = Q_t^\bwd(x,y)\,e^{r_t(x)-r_t(y)},
        \qquad
        \widetilde G_t(x) = \dot r_t(x) + \sum_{y\neq x}\Big(\widetilde Q_t(y,x)-Q_t^\bwd(y,x)\Big).
    \end{equation}

    \item \textbf{Annealed path} ($q_t \propto (p_t^\bwd)^\gamma$): for $\gamma>0$,
    \begin{equation}
        \widetilde Q_t(x,y) = \gamma\,Q_t^\bwd(x,y)\Big(\tfrac{p_t^\bwd(y)}{p_t^\bwd(x)}\Big)^{1-\gamma},
        \qquad
        \widetilde G_t(x) = \sum_{y\neq x}\Big(\widetilde Q_t(y,x)-\gamma Q_t^\bwd(y,x)\Big).
    \end{equation}
\end{itemize}

\subsection{Proof of Proposition~\ref{prop:debiased-discrete-dynamics}}
\label{app:proof-tilt-fk}

For readability, write $p_t:=p_t^\bwd$. From
\[
q_t(x)=\frac{1}{Z_t}(p_t(x))^\gamma e^{r_t(x)},
\qquad
Z_t=\sum_{z\in\sX}(p_t(z))^\gamma e^{r_t(z)},
\]
we obtain
\begin{equation}
\label{eq:qt-log-derivative}
\frac{\partial_t q_t(x)}{q_t(x)}
=
\gamma \frac{\partial_t p_t(x)}{p_t(x)}
+
\dot r_t(x)
-
\frac{\dot Z_t}{Z_t}.
\end{equation}
Using the backward master equation~\eqref{eq:prior-discrete-backward-ctmc},
\[
\partial_t p_t(x)
=
\sum_{y\neq x}\Big(Q_t^\bwd(x,y)p_t(y)-Q_t^\bwd(y,x)p_t(x)\Big),
\]
hence
\[
\gamma \frac{\partial_t p_t(x)}{p_t(x)}
=
\sum_{y\neq x} \gamma Q_t^\bwd(x,y)\frac{p_t(y)}{p_t(x)}
-
\sum_{y\neq x}\gamma Q_t^\bwd(y,x).
\]
Multiplying by $q_t(x)$ and using
\[
\gamma Q_t^\bwd(x,y)\frac{p_t(y)}{p_t(x)}q_t(x)
=
\widetilde Q_t(x,y)q_t(y),
\]
we deduce from~\eqref{eq:qt-log-derivative} that
\begin{equation}
\label{eq:q-derivative-before-normalization}
\partial_t q_t(x)
=
\sum_{y\neq x}\widetilde Q_t(x,y)q_t(y)
+
\left(
\dot r_t(x)-\sum_{y\neq x}\gamma Q_t^\bwd(y,x)-\frac{\dot Z_t}{Z_t}
\right)q_t(x).
\end{equation}

It remains to compute $\dot Z_t/Z_t$. Differentiating $Z_t$ gives
\begin{align*}
\frac{\dot Z_t}{Z_t}
&=
\sum_{z\in\sX}q_t(z)
\left(
\gamma \frac{\partial_t p_t(z)}{p_t(z)}+\dot r_t(z)
\right)\\
&=
\sum_{z\in\sX}q_t(z)
\left[
\dot r_t(z)
+
\sum_{y\neq z}
\left(
\widetilde Q_t(y,z)-\gamma Q_t^\bwd(y,z)
\right)
\right]\\
&=
\E_{q_t}[\widetilde G_t],
\end{align*}
where the second equality again uses~\eqref{eq:tilt-fk-disc-trans-mat}.

Substituting this identity into~\eqref{eq:q-derivative-before-normalization} and adding/subtracting $\sum_{y\neq x}\widetilde Q_t(y,x)q_t(x)$ yields
\begin{align*}
\partial_t q_t(x)
&=
\sum_{y\neq x}
\Big(
\widetilde Q_t(x,y)q_t(y)-\widetilde Q_t(y,x)q_t(x)
\Big)\\
&\quad+
\left[
\dot r_t(x)
+
\sum_{y\neq x}
\Big(
\widetilde Q_t(y,x)-\gamma Q_t^\bwd(y,x)
\Big)
-
\E_{q_t}[\widetilde G_t]
\right]q_t(x)\\
&=
\sum_{y\neq x}
\Big(
\widetilde Q_t(x,y)q_t(y)-\widetilde Q_t(y,x)q_t(x)
\Big)
+
\big(\widetilde G_t(x)-\E_{q_t}[\widetilde G_t]\big)q_t(x),
\end{align*}
which is exactly~\eqref{eq:tilt-fk-pde-disc}. \hfill$\square$

\subsection{Proof of Proposition~\ref{prop:equivalence-class}}
\label{app:proof-equivalence-class}

We start from~\eqref{eq:fk-ctmc}:
\[
\partial_t q_t(x)
=
\sum_{y\neq x}
\Big(
Q_t(x,y)q_t(y)-Q_t(y,x)q_t(x)
\Big)
+
q_t(x)\,g_t(x).
\]
Let $\mR_t$ be such that $Q_t'(y,x)=Q_t(y,x)+R_t(y,x)\ge 0$ for all $x\neq y$. Then
\begin{align*}
\sum_{y\neq x}
\Big(
Q_t'(x,y)q_t(y)-Q_t'(y,x)q_t(x)
\Big)
&=
\sum_{y\neq x}
\Big(
Q_t(x,y)q_t(y)-Q_t(y,x)q_t(x)
\Big)\\
&\quad+
\sum_{y\neq x}
\Big(
R_t(x,y)q_t(y)-R_t(y,x)q_t(x)
\Big).
\end{align*}
By~\eqref{eq:div-qt},
\[
q_t(x)\,\Div_{q_t}\mR_t(x)
=
\sum_{y\neq x}
\Big(
R_t(y,x)q_t(x)-R_t(x,y)q_t(y)
\Big).
\]
Therefore, if we set
\[
g_t'(x):=g_t(x)+\Div_{q_t}\mR_t(x),
\]
the added flux is canceled by the change in centered potential, which proves~\eqref{eq:equivalence-class}.

Finally, centering is preserved:
\begin{align*}
\sum_x q_t(x)\,\Div_{q_t}\mR_t(x)
&=
\sum_x\sum_{y\neq x}
\Big(
R_t(y,x)q_t(x)-R_t(x,y)q_t(y)
\Big)\\
&=
\sum_{x\neq y}R_t(y,x)q_t(x)-\sum_{x\neq y}R_t(x,y)q_t(y)=0.
\end{align*}
Hence $\E_{q_t}[g_t']=\E_{q_t}[g_t]=0$. \hfill$\square$

\subsection{Proof of Proposition~\ref{prop:pure-reweighting}}
\label{app:proof-pure-reweighting}

Apply Proposition~\ref{prop:equivalence-class} with residual rate family $\mR_t:=-\mQ_t$, so that $Q_t'(y,x):=Q_t(y,x)+R_t(y,x)=0$ for all $x\neq y$ (admissibility $Q_t'(y,x)\ge 0$ holds with equality). By definition~\eqref{eq:div-qt},
\[
\Div_{q_t}(-\mQ_t)(x)
=\frac{1}{q_t(x)}\sum_{y\neq x}\Big(Q_t(x,y)q_t(y)-Q_t(y,x)q_t(x)\Big),
\]
so $g_t'(x)=g_t(x)+\Div_{q_t}(-\mQ_t)(x)$ coincides with $g_t^0(x)$ in~\eqref{eq:pure-reweighting}, and the dynamics~\eqref{eq:equivalence-class} collapses to $\partial_t q_t(x)=q_t(x)\,g_t^0(x)$.

In particular, $(\vzero,g_t^0)\in[(\mQ_t,g_t)]$, so $[(\vzero,g_t^0)]=[(\mQ_t,g_t)]$. Finally, applying Proposition~\ref{prop:equivalence-class} to the representative $(\vzero,g_t^0)$ with residual $\mR_t:=\mQ_t'$ (the admissibility $Q_t'(y,x)\ge 0$ is exactly the assumed nonnegativity of $\mQ_t'$) gives
\[
g_t'(x)=g_t^0(x)+\Div_{q_t}\mQ_t'(x),
\]
which is~\eqref{eq:gprime-from-qprime}. \hfill$\square$

\subsection{Proof of Proposition~\ref{prop:zero-variance-unconstrained}}
\label{app:proof-zero-variance-unconstrained}

Fix $t$ and suppress $t$ in the notation. Work in the pure-reweighting reference~\eqref{eq:pure-reweighting} with centered potential $g^0$. Define the source vector
\[
\mu(x):=q(x)\,g^0(x),
\]
so that $\sum_x \mu(x)=\E_q[g^0]=0$.

For $x\neq y$, Proposition~\ref{prop:zero-variance-unconstrained} defines
\[
Q^*(y,x)
=
\frac{1}{D}\left[\frac{q(y)}{q(x)}g^0(y)-g^0(x)\right]_+.
\]
Multiplying by $q(x)$ gives the flux from $x$ to $y$:
\begin{equation}
Q^*(y,x)\,q(x) = \frac{1}{D}\,[\mu(y)-\mu(x)]_+.
\end{equation}
Similarly,
\[
Q^*(x,y)\,q(y)=\frac{1}{D}\,[\mu(x)-\mu(y)]_+.
\]

Now compute the divergence term in~\eqref{eq:gprime-from-qprime}:
\begin{align*}
q(x)\,(\Div_q\mQ^*)(x)
&=
\sum_{y\neq x}
\Big(
Q^*(y,x)q(x)-Q^*(x,y)q(y)
\Big)\\
&=
\frac{1}{D}\sum_{y\neq x}
\Big(
[\mu(y)-\mu(x)]_+-[\mu(x)-\mu(y)]_+
\Big).
\end{align*}
Using $[a]_+-[-a]_+=a$, we get
\[
[\mu(y)-\mu(x)]_+-[\mu(x)-\mu(y)]_+=\mu(y)-\mu(x).
\]
Therefore,
\begin{align*}
q(x)\,(\Div_q\mQ^*)(x)
&=
\frac{1}{D}\sum_{y\neq x}\big(\mu(y)-\mu(x)\big)\\
&=
\frac{1}{D}\Big(\sum_{y\neq x}\mu(y)-(D-1)\mu(x)\Big)\\
&=
\frac{1}{D}\Big(\underbrace{\sum_y\mu(y)}_{=0}-D\mu(x)\Big)
=
-\mu(x)
=
-q(x)\,g^0(x).
\end{align*}
Dividing by $q(x)>0$ yields $(\Div_q\mQ^*)(x)=-g^0(x)$ for all $x$. Hence
\[
g^*(x)=g^0(x)+(\Div_q\mQ^*)(x)=0,
\]
which proves the claim. \hfill$\square$
\subsection{Proof of Proposition~\ref{prop:nnls}}
\label{app:proof-nnls}

We prove that minimizing the variance objective~\eqref{eq:var-obj} over sparse off-diagonal rate families is equivalent to the constrained weighted least-squares problem~\eqref{eq:nnls}.

\paragraph{Residual flux form.}
Work in the pure-reweighting reference $(\vzero,g_t^0)$. For any candidate off-diagonal rate family $\mQ_t'$, Proposition~\ref{prop:equivalence-class} yields
\[
g_t'(x)=g_t^0(x)+\Div_{q_t}\mQ_t'(x).
\]
Multiplying by $q_t(x)$ and using~\eqref{eq:div-qt},
\begin{equation}
q_t(x)\,g_t'(x)
=
\mu_t(x)
+
\sum_{y\neq x}
\Big(
Q_t'(y,x)q_t(x)-Q_t'(x,y)q_t(y)
\Big),
\end{equation}
where $\mu_t(x):=q_t(x)g_t^0(x)$.

Under the sparsity constraint, $Q_t'(y,x)=0$ whenever $(y,x)\notin \gS^\spc$, so this becomes
\[
q_t(x)\,g_t'(x)
=
\mu_t(x)
+
\sum_{y:\,(y,x)\in \gS^\spc} Q_t'(y,x)q_t(x)
-
\sum_{y:\,(x,y)\in \gS^\spc} Q_t'(x,y)q_t(y).
\]

\paragraph{Weighted least squares.}
Because $g_t'$ is centered under $q_t$ by Proposition~\ref{prop:equivalence-class},
\[
\Var_{q_t}(g_t')
=
\sum_{x\in \sX} q_t(x)\,g_t'(x)^2.
\]
Using the previous identity,
\[
\sum_{x\in \sX} q_t(x)\,g_t'(x)^2
=
\sum_{x\in \sX}
\frac{1}{q_t(x)}
\left[
\mu_t(x)
+
\sum_{y:\,(y,x)\in \gS^\spc} Q_t'(y,x)q_t(x)
-
\sum_{y:\,(x,y)\in \gS^\spc} Q_t'(x,y)q_t(y)
\right]^2.
\]
Hence minimizing~\eqref{eq:var-obj} over all sparse off-diagonal rate families $\mQ_t'$ with
\[
Q_t'(y,x)=0 \ \text{for }(y,x)\notin \gS^\spc,
\qquad
Q_t'(y,x)\ge 0 \ \text{for }(y,x)\in \gS^\spc,
\]
is exactly the constrained weighted least-squares problem~\eqref{eq:nnls}. \hfill$\square$

\subsection{Proof of Proposition~\ref{prop:disc-vcg}}
\label{app:proof-disc-vcg}

Suppress $t$ for readability and write expectations under $q$. By definition,
\[
g^\vcg(x;\vtheta)
=
g^0(x)+\sum_{j=1}^J\theta_j(\Div_q \mQ^{(j)})(x).
\]
Let
\[
\bar d^{(j)}:=\E_q[(\Div_q \mQ^{(j)})(x)],
\qquad
\widetilde D^{(j)}(x):=(\Div_q \mQ^{(j)})(x)-\bar d^{(j)}.
\]
Since $\E_q[g^0]=0$, the centered version of $g^\vcg(\cdot;\vtheta)$ is
\[
g^0(x)+\sum_{j=1}^J\theta_j\widetilde D^{(j)}(x).
\]
Therefore
\begin{align*}
\Var_q(g^\vcg(\cdot;\vtheta))
&=
\E_q[(g^0(x))^2]
+2\sum_{i=1}^J \theta_i\,
\E_q\!\left[g^0(x)\widetilde D^{(i)}(x)\right] \\
&\quad+
\sum_{i,j=1}^J \theta_i\theta_j\,
\E_q\!\left[\widetilde D^{(i)}(x)\widetilde D^{(j)}(x)\right].
\end{align*}
This is a convex quadratic function of $\vtheta$. Its gradient is $2(\mA\vtheta+\vc)$, so the first-order optimality condition for the unconstrained problem is
\[
\sum_{j=1}^J A_{ij}\theta_j=-c_i,
\]
with $A_{ij}$ and $c_i$ as in~\eqref{eq:disc-vcg-system}. Because the objective is convex, these normal equations characterize the full set of unconstrained minimizers; if $\mA$ is nonsingular, the minimizer is unique. The nonnegative version is the same quadratic objective restricted to $\R_+^J$. \hfill$\square$

\subsection{\texorpdfstring{\texttt{HEU}}{HEU} details}
\label{app:heu-discussion}

This appendix expands on Proposition~\ref{prop:lar} (the local-average reallocation heuristic). The implementable heuristic uses the outgoing one-hop neighborhood $N_t^+(x)$. More generally, one could replace $N_t^+(x)$ by a larger local neighborhood $B_t(x)\subseteq\sX$, e.g.\ by including incoming or multi-hop neighbors; this would make the rule closer to the dense zero-variance construction of Proposition~\ref{prop:zero-variance-unconstrained}, at the price of higher computational cost.

\paragraph{Local-averaging viewpoint.}
When the implemented neighborhood is reciprocal, $z\in B_t(x)$ iff $x\in B_t(z)$, and the same pairwise normalization is used in both directions, the positive-part identity $[a-b]_+-[b-a]_+=a-b$ gives the local-averaging heuristic
\[
q_t(x)\,\Div_{q_t}\mQ_t^{\heu}(x)
\approx
\frac{\alpha_t}{|B_t(x)|}\sum_{z\in B_t(x)}(\mu_t(z)-\mu_t(x)),
\]
so that
\[
q_t(x)\,g_t^{\heu}(x)
=
\mu_t(x)+q_t(x)\,\Div_{q_t}\mQ_t^{\heu}(x)
\approx
(1-\alpha_t)\mu_t(x)
+\frac{\alpha_t}{|B_t(x)|}\sum_{z\in B_t(x)}\mu_t(z).
\]
The heuristic therefore damps the part of $\mu_t(x)$ that deviates from a local neighborhood average, while leaving a residual local mean. Larger neighborhoods can bring the heuristic closer to the dense zero-residual construction.

\paragraph{Zero extra score-network evaluations.}
A practical remark on the cost of \texttt{HEU}: the source term $\mu_t(y)$ at a destination state $y$ admits the analytical expansion
\[
\mu_t(y)=q_t(y)g_t(y)+\sum_{z\neq y}\big(Q_t(y,z)q_t(z)-Q_t(z,y)q_t(y)\big),
\]
i.e.\ $\mu_t(y)=q_t(y)\,g_t^0(y)$ with $g_t^0$ defined in~\eqref{eq:pure-reweighting}. This form is well-defined on every $y$, including unvisited destinations of an empirical particle cloud, and uses only edge rates and density values that the proposal step already evaluates; in particular, no additional score-network forward pass is required to evaluate $\mu_t$ on the one-hop neighborhood of the current particles.

\paragraph{Choosing $k_t(x)$ in mask diffusion.}
The normalization factor $k_t(x)$ sets the strength of local reallocation. In mask diffusion on $\sX=[V+1]^L$, if $m(x)$ denotes the number of masked positions, the outgoing neighborhood has size $|N_t^+(x)|=V\,m(x)$ and the incoming predecessor set has size $|N_t^-(x)|=L-m(x)$. We consider two natural choices: an \emph{aggressive} outgoing-star normalization
\begin{equation}
k_t(x):=1+|N_t^+(x)|=1+V\,m(x),
\end{equation}
and a more \emph{conservative} one-hop normalization
\begin{equation}
k_t(x):=1+|N_t^+(x)|+|N_t^-(x)|=1+V\,m(x)+L-m(x),
\end{equation}
which better reflects the local-averaging picture by accounting for both outgoing and incoming one-step neighbors. The damping parameter $\alpha_t<1$ can be used in either case to further stabilize the reallocation when $g_t^0$ is noisy.

\section{Assumptions and Proofs for Section~\ref{sec:theoretical-analysis}}
\label{app:assumption-blocks}

This appendix states the exact hypotheses used in Section~\ref{sec:theoretical-analysis} and proves the two theory results in the notation of the main text. The first result is a population score-ratio stability statement for the guided Feynman--Kac flow when the forward noising rates are known and the reverse rates are represented through the learned local ratio. The second result is a particle-only statement for a fixed deterministic grid Feynman--Kac recursion. We state the assumptions, prove the two theorems, give a supporting weight-variance lemma, and conclude with remarks on the role and scope of the results.

\subsection{Assumptions}

\begin{assumption}[Stability of guided unnormalized operators]
\label{assump:guided-bounds}
Let $(q_t)_{t\in[0,T]}$ and $(q_t^{\widehat s})_{t\in[0,T]}$ be the normalized Feynman--Kac flows generated by the exact and implemented guided pairs defined in Assumption~\ref{assump:score-window}. Equivalently, let their unnormalized versions solve
\[
\partial_t\rho_t
=
\big(\widetilde{\mL}_t+\diag(\widetilde G_t)\big)\rho_t,
\qquad
\partial_t\widehat\rho_t
=
\big(\widehat{\widetilde{\mL}}_t+\diag(\widehat{\widetilde G}_t)\big)\widehat\rho_t,
\qquad
\rho_0=\widehat\rho_0=q_0,
\]
with $q_t=\rho_t/(\vone^\top\rho_t)$ and $q_t^{\widehat s}=\widehat\rho_t/(\vone^\top\widehat\rho_t)$. Define
\[
\widetilde\lambda_t(x):=\sum_{y\neq x}\widetilde Q_t(y,x),
\qquad
\lambda_t^\bwd(x):=\sum_{y\neq x}Q_t^\bwd(y,x),
\qquad
\bar\lambda_{\rm op}(t):=q_t\!\left(\widetilde\lambda_t+\gamma\lambda_t^\bwd\right).
\]
Assume $\bar\lambda_{\rm op}$ is integrable on $[0,T]$ and that there exist constants $\Lambda_{\rm op},B<\infty$ such that, for all $t\in[0,T]$,
\begin{equation}
\label{eq:assump-bound-QG-app}
\max_{x\in\sX}\sum_{y\neq x}\widehat{\widetilde Q}_t(y,x)\le \Lambda_{\rm op},
\qquad
\norminf{\widetilde G_t}\le B,
\qquad
\norminf{\widehat{\widetilde G}_t}\le B.
\end{equation}
\end{assumption}

\begin{remark}[Guided bounds]
\label{rem:role-of-guided-bounds}
The bounds in~\eqref{eq:assump-bound-QG-app} are state-uniform regularity conditions on the exact and implemented guided operators. The constant $B$ keeps the unnormalized masses $Z_t,\widehat Z_t$ in $[e^{-Bt},e^{Bt}]$ and controls the $\ell^1$ logarithmic norm of the implemented Metzler operator in the proof. The outflow bound $\Lambda_{\rm op}$ is a well-posedness bound for the implemented CTMC; it is not part of the final total-variation constant.
\end{remark}

\begin{assumption}[Known-forward local-ratio perturbation]
\label{assump:score-window}
Assume the forward noising rates $Q_{T-t}^\fwd(x,y)$ are known and that the exact reverse rates are represented by the local ratio
\[
s_t(x,y):=\frac{p_t^\bwd(y)}{p_t^\bwd(x)},
\qquad
Q_t^\bwd(y,x)=Q_{T-t}^\fwd(x,y)s_t(x,y).
\]
The implemented reverse rates use $\widehat s_t$:
\[
\widehat Q_t^\bwd(y,x)=Q_{T-t}^\fwd(x,y)\widehat s_t(x,y).
\]
For $x\neq y$, define the exact and implemented guided rates by
\begin{align}
\widetilde Q_t(y,x)
&=
\gamma Q_{T-t}^\fwd(x,y)s_t(x,y)^\gamma e^{r_t(y)-r_t(x)},
&
\widehat{\widetilde Q}_t(y,x)
&=
\gamma Q_{T-t}^\fwd(x,y)\widehat s_t(x,y)^\gamma e^{r_t(y)-r_t(x)}.
\end{align}
Their unnormalized Feynman--Kac potentials are
\begin{align}
\widetilde G_t(x)
&=\dot r_t(x)+\sum_{y\neq x}\Big(\widetilde Q_t(y,x)-\gamma Q_t^\bwd(y,x)\Big),\notag\\
\widehat{\widetilde G}_t(x)
&=\dot r_t(x)+\sum_{y\neq x}\Big(\widehat{\widetilde Q}_t(y,x)-\gamma \widehat Q_t^\bwd(y,x)\Big).
\end{align}
Assume there exist constants $0<\underline\kappa\le\overline\kappa<\infty$ such that
\begin{equation}
\label{eq:score-relative-bound-app}
\underline\kappa
\le
\frac{\widehat s_t(x,y)}{s_t(x,y)}
\le
\overline\kappa
\qquad
\text{whenever }\widetilde Q_t(y,x)+Q_t^\bwd(y,x)>0.
\end{equation}
\end{assumption}

\begin{remark}[Ratio window]
\label{rem:role-of-score-window}
The window~\eqref{eq:score-relative-bound-app} is uniform across relevant edges but it is not an accuracy condition. It only makes the constants
\[
C_{a,\underline\kappa,\overline\kappa}
:=
\sup_{u\in[\underline\kappa,\overline\kappa]}
\frac{|u^a-1|}{\sqrt{\ell(u)}}
\]
finite for the two exponents $a=\gamma$ and $a=1$ used in the proof. Because the reverse generator is also rebuilt from $\widehat s_t$, the case $\gamma=1$ is not error-free: the implemented reverse generator $\widehat Q_t^\bwd=Q_{T-t}^\fwd\widehat s_t$ still changes with the score estimate, so the exponent-$1$ term remains.
\end{remark}

\begin{assumption}[Score-entropy coverage of the tilted path]
\label{assump:training-loss-coverage}
Let
\begin{equation}
\label{eq:score-entropy-edge-measure-app}
\nu_t^{\rm SE}(x,y):=p_t^\bwd(x)Q_t^\bwd(y,x),
\qquad x\neq y,
\end{equation}
be the source-weighted score-entropy edge measure associated with the reverse diffusion rates. Define the time-pointwise score-entropy training loss and its time-integrated scalar,
\begin{align}
\label{eq:training-loss-app}
\mathcal L_{\rm DD}(t)
&:=
\sum_{x\in\sX}\sum_{y\neq x}
\nu_t^{\rm SE}(x,y)\,
\ell\!\left(\frac{\widehat s_t(x,y)}{s_t(x,y)}\right),\notag\\
\mathfrak L_{\rm DD}
&:=\int_0^T\mathcal L_{\rm DD}(t)\,\dif t,
\qquad
\ell(u):=-\log u-1+u,
\end{align}
and the pointwise tilt-to-base ratio
\begin{equation}
h_t(x):=\frac{q_t(x)}{p_t^\bwd(x)}.
\end{equation}
Define further the time-pointwise edgewise coverage factor and its $L^2$-in-time integrated scalar,
\begin{align}
C_{\rm DD}(t)
&:=
\left(\sum_{x\in\sX}\sum_{y\neq x}\nu_t^{\rm SE}(x,y)
\big(C_{\gamma,\underline\kappa,\overline\kappa}h_t(y)+C_{1,\underline\kappa,\overline\kappa}h_t(x)\big)^2\right)^{1/2},
\\[2pt]
\mathfrak C_{\rm DD}
&:=
\left(\int_0^T C_{\rm DD}(t)^2\,\dif t\right)^{1/2}.
\label{eq:Cdd-integrated-app}
\end{align}
Assume $C_{\rm DD}(t)$ is finite for a.e.\ $t$ and $\mathfrak C_{\rm DD}<\infty$.
\end{assumption}

\begin{remark}[Meaning of the score-entropy coverage factor]
\label{rem:score-entropy-coverage}
The measure in~\eqref{eq:score-entropy-edge-measure-app} is the population edge weighting of source-weighted score-entropy training~\citep{lou2024discrete}; it is the object represented by $\mathcal L_{\rm DD}$ in the main text. The Feynman--Kac proof also uses the internal operator-error edge measure
\begin{equation}
\label{eq:operator-edge-measure-app}
\nu_t^{\rm op}(x,y):=q_t(x)\Big(\widetilde Q_t(y,x)+\gamma Q_t^\bwd(y,x)\Big),
\qquad x\neq y.
\end{equation}
This is not a separate training loss. It appears because score error perturbs both the guided off-diagonal jump rates and the diagonal reverse-rate term in the potential. The comparison with score-entropy training is exact. Indeed, using $\widetilde Q_t(y,x)=\gamma Q_{T-t}^\fwd(x,y)s_t(x,y)^\gamma e^{r_t(y)-r_t(x)}$, $Q_t^\bwd(y,x)=Q_{T-t}^\fwd(x,y)s_t(x,y)$, and $\nu_t^{\rm SE}(x,y)=p_t^\bwd(x)Q_t^\bwd(y,x)$,
\begin{align}
\frac{q_t(x)\widetilde Q_t(y,x)}{\nu_t^{\rm SE}(x,y)}
&=\gamma\frac{q_t(x)}{p_t^\bwd(x)}s_t(x,y)^{\gamma-1}e^{r_t(y)-r_t(x)}=\gamma h_t(y),\notag\\
\frac{\gamma q_t(x)Q_t^\bwd(y,x)}{\nu_t^{\rm SE}(x,y)}
&=\gamma h_t(x). \label{eq:op-se-decomp-app}
\end{align}
Hence
\begin{equation}
\label{eq:nu-op-se-identity-app}
\nu_t^{\rm op}(x,y)
=\gamma\big(h_t(x)+h_t(y)\big)\nu_t^{\rm SE}(x,y).
\end{equation}
Thus the endpoint coverage weights $h_t(x)$ and $h_t(y)$ are the price for using a score model trained under the original diffusion path while sampling the tilted path. The integrated scalar $\mathfrak C_{\rm DD}$ in~\eqref{eq:Cdd-integrated-app} preserves this edgewise profile through an $L^2(\dif t)$ aggregation rather than replacing it by a worst-case supremum. A coarser sufficient condition is obtained by setting
\[
H_t:=\sup_{\nu_t^{\rm SE}(x,y)>0}\max\{h_t(x),h_t(y)\},
\qquad
\bar\lambda_{\rm SE}(t):=\sum_x\sum_{y\neq x}\nu_t^{\rm SE}(x,y),
\]
which yields $C_{\rm DD}(t)\le(C_{\gamma,\underline\kappa,\overline\kappa}+C_{1,\underline\kappa,\overline\kappa})H_t\sqrt{\bar\lambda_{\rm SE}(t)}$ and hence
\[
\mathfrak C_{\rm DD}\le(C_{\gamma,\underline\kappa,\overline\kappa}+C_{1,\underline\kappa,\overline\kappa})\sqrt{\int_0^T H_t^2\,\bar\lambda_{\rm SE}(t)\,\dif t}.
\]
Large endpoint coverage weights mean that the tilt has moved probability mass to states or edge endpoints poorly represented under the original score-training law, where local-ratio errors can be amplified even when the average training loss under $p_t^\bwd$ is small.
\end{remark}

\begin{assumption}[Fixed grid-based controlled Feynman--Kac model]
\label{assump:fixed-grid-fk}
Fix a time grid $0=t_0<\cdots<t_M=T$. For each $k=0,\dots,M-1$, let $\mQ_{t_k}^\eff$ be a deterministic off-diagonal rate family and let $g_{t_k}^\eff:\sX\to\sR$ be a deterministic potential. Let $\mL_{t_k}^\eff$ be the associated full generator,
\[
L_{t_k}^\eff(y,x)=Q_{t_k}^\eff(y,x)\quad(y\neq x),
\qquad
L_{t_k}^\eff(x,x)=-\sum_{y\neq x}Q_{t_k}^\eff(y,x),
\]
and define
\[
\mP_k:=\exp\!\big(\Delta t_k\mL_{t_k}^\eff\big),
\qquad
W_k(x):=\exp\!\big(\Delta t_k\,g_{t_k}^\eff(x)\big),
\qquad
\Delta t_k:=t_{k+1}-t_k.
\]
With the column convention, $P_k(y,x)$ is the transition probability from source $x$ to destination $y$, and $(\mathsf P_k f)(x):=(\mP_k^\top f)(x)=\sum_yP_k(y,x)f(y)$. Assume $0<\inf_x W_k(x)\le\sup_x W_k(x)<\infty$ for every $k$.
\end{assumption}

\subsection{Formal score-ratio theorem and proof}
\label{app:proof-score-error}

\begin{theorem}[Formal version of Theorem~\ref{thm:ratio-error}: known forward rates]
\label{thm:ratio-error-formal}
Suppose Assumptions~\ref{assump:guided-bounds} and~\ref{assump:score-window} hold. For $a>0$, define
\begin{equation}
\label{eq:Cscore-def}
C_{a,\underline\kappa,\overline\kappa}
:=
\sup_{u\in[\underline\kappa,\overline\kappa]}
\frac{|u^a-1|}{\sqrt{\ell(u)}},
\qquad
C_{\rm sc}:=\max\{C_{\gamma,\underline\kappa,\overline\kappa},C_{1,\underline\kappa,\overline\kappa}\}.
\end{equation}
At $u=1$ the quotient is interpreted by continuous extension, equal to $\sqrt{2}a$. These constants are finite, and:

\smallskip
\noindent\textbf{(a) Intermediate operator-error form (proof-internal).} Let
\begin{equation}
\mathcal{E}_{\rm op}(t)
:=
\sum_{x\in\sX}\sum_{y\neq x}\nu_t^{\rm op}(x,y)\,
\ell\!\big(\widehat s_t(x,y)/s_t(x,y)\big),
\end{equation}
where $\nu_t^{\rm op}$ is the proof-internal operator-error edge measure in~\eqref{eq:operator-edge-measure-app}, with total mass $\bar\lambda_{\rm op}(t)$. Then
\begin{equation}
\label{eq:tv-Eq-bound-app}
\TV(q_T^{\widehat s},q_T)
\le
C_{\rm sc}\,e^{2BT}
\int_0^T \sqrt{\bar\lambda_{\rm op}(t)\mathcal{E}_{\rm op}(t)}\,\dif t.
\end{equation}
This bound is purely an internal step in the proof and is not to be interpreted as a population training objective.

\smallskip
\noindent\textbf{(b) Score-entropy training-loss form.} If, in addition, Assumption~\ref{assump:training-loss-coverage} holds, then
\begin{equation}
\label{eq:tv-training-loss-bound-app}
\TV(q_T^{\widehat s},q_T)
\le
\gamma e^{2BT}
\int_0^T
C_{\rm DD}(t)\sqrt{\mathcal L_{\rm DD}(t)}\,\dif t.
\end{equation}

\smallskip
\noindent\textbf{(c) Time-integrated form (main-text statement).} With the integrated quantities $\mathfrak L_{\rm DD}$ in~\eqref{eq:training-loss-app} and $\mathfrak C_{\rm DD}$ in~\eqref{eq:Cdd-integrated-app}, Cauchy--Schwarz in $t$ applied to~\eqref{eq:tv-training-loss-bound-app} gives
\begin{equation}
\label{eq:tv-integrated-loss-bound-app}
\TV(q_T^{\widehat s},q_T)
\le
\gamma\, e^{2BT}\,\mathfrak C_{\rm DD}\,\sqrt{\mathfrak L_{\rm DD}}.
\end{equation}
This is exactly the main-text bound~\eqref{eq:main-training-loss-bound}.
\end{theorem}

\noindent\textbf{Proof.}
We first prove that $C_{\rm sc}<\infty$. For any fixed $a>0$, the map $u\mapsto |u^a-1|/\sqrt{\ell(u)}$ is continuous on $(0,\infty)\setminus\{1\}$. At $u=1$, Taylor expansion gives $u^a-1=a(u-1)+O((u-1)^2)$ and $\ell(u)=(u-1)^2/2+O((u-1)^3)$, hence the limit equals $\sqrt{2}a$. The function is therefore continuous on the compact interval $[\underline\kappa,\overline\kappa]$, so both constants in~\eqref{eq:Cscore-def} are finite.

Throughout, set
\[
\mA_t:=\widetilde{\mL}_t+\diag(\widetilde G_t),
\qquad
\widehat{\mA}_t:=\widehat{\widetilde{\mL}}_t+\diag(\widehat{\widetilde G}_t),
\]
so the unnormalized flows obey $\partial_t\rho_t=\mA_t\rho_t$ and $\partial_t\widehat\rho_t=\widehat{\mA}_t\widehat\rho_t$ with common initial condition. Write $Z_t:=\vone^\top\rho_t$, $\widehat Z_t:=\vone^\top\widehat\rho_t$, and $\delta_t := \widehat\rho_t-\rho_t$.

\paragraph{Mass control.}
Since each generator $\widetilde{\mL}_t$ satisfies $\vone^\top\widetilde{\mL}_t=\vzero^\top$ by construction,
\[
\partial_t Z_t = \vone^\top\partial_t\rho_t = \vone^\top\diag(\widetilde G_t)\rho_t = \rho_t(\widetilde G_t),
\qquad
|\partial_t Z_t|\le \norminf{\widetilde G_t}Z_t \le BZ_t.
\]
Grönwall's inequality gives $e^{-Bt}\le Z_t\le e^{Bt}$, and the same argument with $\widehat{\widetilde G}_t$ gives $e^{-Bt}\le\widehat Z_t\le e^{Bt}$.

\paragraph{Operator perturbation from the local ratio.}
Let $u_t(x,y):=\widehat s_t(x,y)/s_t(x,y)$. From Assumption~\ref{assump:score-window},
\begin{equation}
\Delta\widetilde Q_t(y,x):=\widehat{\widetilde Q}_t(y,x)-\widetilde Q_t(y,x)
=\widetilde Q_t(y,x)\big(u_t(x,y)^\gamma-1\big),
\end{equation}
and
\begin{equation}
\Delta Q_t^\bwd(y,x):=\widehat Q_t^\bwd(y,x)-Q_t^\bwd(y,x)
=Q_t^\bwd(y,x)\big(u_t(x,y)-1\big).
\end{equation}
The potential difference is
\[
\widehat{\widetilde G}_t(x)-\widetilde G_t(x)
=
\sum_{y\neq x}\Big(\Delta\widetilde Q_t(y,x)-\gamma\Delta Q_t^\bwd(y,x)\Big).
\]
The diagonal entry of column $x$ of $\widehat{\mA}_t-\mA_t$ is therefore
\[
-\sum_{y\neq x}\Delta\widetilde Q_t(y,x)
+
\sum_{y\neq x}\Big(\Delta\widetilde Q_t(y,x)-\gamma\Delta Q_t^\bwd(y,x)\Big)
=
-\gamma\sum_{y\neq x}\Delta Q_t^\bwd(y,x),
\]
while the off-diagonal entry from source $x$ to destination $y$ is $\Delta\widetilde Q_t(y,x)$. Consequently, for any vector $v$ and any state $z$,
\begin{equation}
\big((\widehat{\mA}_t-\mA_t)v\big)(z)
=
\sum_{x\neq z}\Delta\widetilde Q_t(z,x)v(x)
-
\gamma v(z)\sum_{y\neq z}\Delta Q_t^\bwd(y,z).
\end{equation}
Using $\rho_t=Z_tq_t$ and the triangle inequality,
\begin{align}
\normone{(\widehat{\mA}_t-\mA_t)\rho_t}
&\le
Z_t\sum_xq_t(x)\sum_{y\neq x}\widetilde Q_t(y,x)|u_t(x,y)^\gamma-1|\notag\\
&\quad+
Z_t\gamma\sum_xq_t(x)\sum_{y\neq x}Q_t^\bwd(y,x)|u_t(x,y)-1|.
\label{eq:source-pre-CS}
\end{align}

\paragraph{Internal operator-error loss.}
By~\eqref{eq:Cscore-def}, $|u^\gamma-1|\le C_{\rm sc}\sqrt{\ell(u)}$ and $|u-1|\le C_{\rm sc}\sqrt{\ell(u)}$ on the ratio window. With $\nu_t^{\rm op}$ as in~\eqref{eq:operator-edge-measure-app}, the right-hand side of~\eqref{eq:source-pre-CS} is bounded by
\[
Z_t C_{\rm sc}\sum_{x,y:x\neq y}\nu_t^{\rm op}(x,y)\sqrt{\ell(u_t(x,y))}.
\]
Cauchy--Schwarz under the nonnegative measure $\nu_t^{\rm op}$ gives
\begin{align}
\sum_{x,y:x\neq y}\nu_t^{\rm op}(x,y)\sqrt{\ell(u_t(x,y))}
&\le
\sqrt{\sum_{x,y:x\neq y}\nu_t^{\rm op}(x,y)}
\sqrt{\sum_{x,y:x\neq y}\nu_t^{\rm op}(x,y)\ell(u_t(x,y))}\notag\\
&=
\sqrt{\bar\lambda_{\rm op}(t)\mathcal E_{\rm op}(t)}.
\end{align}
Combining the preceding displays,
\begin{equation}
\label{eq:source-term-averaged-loss}
\normone{(\widehat{\mA}_t-\mA_t)\rho_t}
\le
Z_t C_{\rm sc}\sqrt{\bar\lambda_{\rm op}(t)\mathcal E_{\rm op}(t)}.
\end{equation}
This is the only place where the edgewise local-ratio loss enters; no supremum over states is used in the internal bound.

\paragraph{Duhamel stability.}
The error $\delta_t = \widehat\rho_t-\rho_t$ satisfies $\delta_0=0$ and
\[
\partial_t\delta_t
=
\widehat{\mA}_t\delta_t+(\widehat{\mA}_t-\mA_t)\rho_t.
\]
Let $\widehat\Phi_{T,s}$ denote the evolution operator of $\partial_\tau v=\widehat{\mA}_\tau v$ for $\tau\in[s,T]$. Because $\widehat{\mA}_t=\widehat{\widetilde{\mL}}_t+\diag(\widehat{\widetilde G}_t)$ is Metzler and the columns of $\widehat{\widetilde{\mL}}_t$ sum to zero, its $\ell^1$ logarithmic norm is
\[
\mu_1(\widehat{\mA}_t)
=\max_x\left((\widehat{\mA}_t)_{xx}+\sum_{y\neq x}|(\widehat{\mA}_t)_{yx}|\right)
=\max_x\widehat{\widetilde G}_t(x)
\le B.
\]
The logarithmic-norm propagator estimate gives $\normcol{\widehat\Phi_{T,s}}\le\exp(B(T-s))$. By Duhamel's principle and~\eqref{eq:source-term-averaged-loss},
\begin{align}
\normone{\delta_T}
&\le
\int_0^T \exp\!\big(B(T-s)\big)\,Z_s\,C_{\rm sc}\sqrt{\bar\lambda_{\rm op}(s)\mathcal E_{\rm op}(s)}\,\dif s\notag\\
&\le
C_{\rm sc}e^{BT}\int_0^T\sqrt{\bar\lambda_{\rm op}(s)\mathcal E_{\rm op}(s)}\,\dif s.
\label{eq:unnormalized-diff-bound}
\end{align}

\paragraph{Normalization.}
For nonnegative non-zero vectors $a,b$,
\[
\left\|\frac{a}{\vone^\top a}-\frac{b}{\vone^\top b}\right\|_1
\le
\frac{2\normone{a-b}}{\min\{\vone^\top a,\vone^\top b\}}.
\]
Applying this with $a=\widehat\rho_T$, $b=\rho_T$ and dividing by $2$ gives
$\TV(q_T^{\widehat s},q_T)\le \normone{\delta_T}/\min\{\widehat Z_T,Z_T\}\le e^{BT}\normone{\delta_T}$. Combined with~\eqref{eq:unnormalized-diff-bound}, this proves part (a), namely~\eqref{eq:tv-Eq-bound-app}.

\paragraph{Score-entropy training-loss comparison.}
For the sharper training-loss form, return to~\eqref{eq:source-pre-CS}. The identities in~\eqref{eq:op-se-decomp-app} give
\begin{align}
Z_t^{-1}\normone{(\widehat{\mA}_t-\mA_t)\rho_t}
&\le
\gamma\sum_{x,y:x\neq y}\nu_t^{\rm SE}(x,y)
\big(C_{\gamma,\underline\kappa,\overline\kappa}h_t(y)+C_{1,\underline\kappa,\overline\kappa}h_t(x)\big)
\sqrt{\ell(u_t(x,y))}\notag\\
&\le
\gamma\,C_{\rm DD}(t)\sqrt{\mathcal L_{\rm DD}(t)}.
\label{eq:source-se-training-bound}
\end{align}
Substituting~\eqref{eq:source-se-training-bound} into the same Duhamel and normalization argument proves~\eqref{eq:tv-training-loss-bound-app}.

\paragraph{Time-integrated form.}
Applying Cauchy--Schwarz in $t$ to~\eqref{eq:tv-training-loss-bound-app},
\[
\int_0^T C_{\rm DD}(t)\sqrt{\mathcal L_{\rm DD}(t)}\,\dif t
\le
\sqrt{\int_0^T C_{\rm DD}(t)^2\,\dif t}\,\sqrt{\int_0^T\mathcal L_{\rm DD}(t)\,\dif t}
=
\mathfrak C_{\rm DD}\,\sqrt{\mathfrak L_{\rm DD}},
\]
which gives~\eqref{eq:tv-integrated-loss-bound-app}. \hfill$\square$

\subsection{Formal particle theorem and proof}
\label{app:proof-many-particle}

\begin{theorem}[Formal version of Theorem~\ref{thm:many-particle-convergence}]
\label{thm:many-particle-formal}
Under Assumption~\ref{assump:fixed-grid-fk}, define the exact fixed-grid Feynman--Kac recursion by
\begin{equation}
q_0^\Delta:=q_{t_0},
\qquad
q_{k+1}^\Delta(f)
:=
\frac{q_k^\Delta(W_k\mathsf P_k f)}{q_k^\Delta(W_k)},
\qquad k=0,\ldots,M-1.
\end{equation}
Initialize $N$ particles $X_0^1,\ldots,X_0^N$ i.i.d.\ from $q_0^\Delta$ and set $q_0^{N,\Delta}:=N^{-1}\sum_{i=1}^N\delta_{X_0^i}$. At each step $k$, conditionally on the current particles, draw ancestors $A_k^1,\ldots,A_k^N$ multinomially with probabilities proportional to $W_k(X_k^j)$ over candidate ancestors $j=1,\ldots,N$, and then draw
\[
\Pr\!\left(X_{k+1}^i=y\mid X_k^{A_k^i}=x\right)=P_k(y,x)
\quad\text{independently over }i.
\]
Let $q_k^{N,\Delta}:=N^{-1}\sum_{i=1}^N\delta_{X_k^i}$. Define
\[
\bar\beta_{k,M}:=\prod_{j=k}^{M-1}\beta_j,
\qquad
\beta_j:=\frac{\sup_x W_j(x)}{\inf_x W_j(x)},
\]
with the convention that an empty product is $1$. Then, for every bounded $f:\sX\to\sR$,
\begin{equation}
\label{eq:smc-rate-app}
\E\Big[\big|q_M^{N,\Delta}(f)-q_M^\Delta(f)\big|\Big]
\le
\frac{\osc(f)}{2\sqrt N}
\sum_{k=0}^{M}\bar\beta_{k,M}.
\end{equation}
\end{theorem}

\noindent\textbf{Proof.}
The proof is self-contained for the deterministic every-step bootstrap Feynman--Kac system above. It does not use score-estimation, adaptive-control, or time-discretization arguments. We write $\mathcal{F}_k:=\sigma(X_0^{1:N},\ldots,X_k^{1:N})$ for the natural filtration of the particle system at time $k$.

\paragraph{Conditional sampling.}
For any probability measure $\mu$ on $\sX$, define the one-step Feynman--Kac map
\[
\Phi_k(\mu)(f):=
\frac{\mu(W_k\mathsf P_k f)}{\mu(W_k)},
\]
which lies between $\inf_x\mathsf P_k f(x)$ and $\sup_x\mathsf P_k f(x)$ and so defines a probability measure on $\sX$. The exact recursion is $q_{k+1}^\Delta=\Phi_k(q_k^\Delta)$. By the multinomial-resampling-then-mutation update, conditional on $\mathcal{F}_k$:
\begin{enumerate}[leftmargin=1.2em,topsep=2pt,itemsep=2pt,parsep=0pt]
\item the ancestor indices $A_k^1,\ldots,A_k^N$ are i.i.d.\ with
\(
\Pr(A_k^i=j\mid\mathcal{F}_k)=\frac{W_k(X_k^j)}{\sum_{l=1}^N W_k(X_k^l)};
\)
\item given the ancestors, the new particles are drawn independently with $\Pr(X_{k+1}^i=y\mid X_k^{A_k^i}=x)=P_k(y,x)$.
\end{enumerate}
Hence, for any bounded $f:\sX\to\sR$ and any $i$,
\[
\E[f(X_{k+1}^i)\mid\mathcal{F}_k]
=\sum_{j=1}^N\frac{W_k(X_k^j)}{\sum_l W_k(X_k^l)}\,(\mathsf P_k f)(X_k^j)
=\frac{q_k^{N,\Delta}(W_k\mathsf P_k f)}{q_k^{N,\Delta}(W_k)}
=\Phi_k(q_k^{N,\Delta})(f),
\]
and the $X_{k+1}^i$ are i.i.d.\ given $\mathcal{F}_k$. Equivalently, $q_{k+1}^{N,\Delta}$ is the empirical measure of $N$ i.i.d.\ samples from $\mu_{k+1}^N := \Phi_k(q_k^{N,\Delta})$. \emph{This conditional i.i.d.\ property is the only place the every-step bootstrap structure is used in the entire proof.}

\paragraph{Future Feynman--Kac weights.}
For $0\le k\le \ell\le M$, define the future Feynman--Kac semigroup on test functions by
\[
\mathsf H_{k,\ell}:=(W_k\mathsf P_k)(W_{k+1}\mathsf P_{k+1})\cdots(W_{\ell-1}\mathsf P_{\ell-1}),
\qquad
\mathsf H_{M,M}:=I.
\]
Probabilistically, $\mathsf H_{k,\ell}f(x)=\E_x\big[\big(\prod_{j=k}^{\ell-1}W_j(X_j)\big)\,f(X_\ell)\big]$ where $(X_j)_{j=k}^\ell$ is the inhomogeneous Markov chain with $X_k=x$ and $\Pr(X_{j+1}=y\mid X_j=x)=P_j(y,x)$. Set
\[
m_{k,\ell}:=\prod_{j=k}^{\ell-1}\inf_x W_j(x),
\qquad
M_{k,\ell}:=\prod_{j=k}^{\ell-1}\sup_x W_j(x),
\qquad
B_{k,\ell}:=\frac{M_{k,\ell}}{m_{k,\ell}}=\prod_{j=k}^{\ell-1}\beta_j.
\]
Because each $\mathsf P_j$ is a Markov averaging operator and each $W_j$ is positive,
\begin{equation}
\label{eq:future-weight-bounds}
m_{k,\ell}\le \mathsf H_{k,\ell}\vone(x)\le M_{k,\ell},
\qquad
|\mathsf H_{k,\ell}f(x)|\le M_{k,\ell}\norminf{f}
\end{equation}
for every $x$.

\paragraph{Empirical error after a positive transform.}
Let $R$ be a positive operator with $m\le R\vone\le \overline m$. For a probability measure $\zeta$ and the empirical measure $\zeta^N$ of $N$ i.i.d.\ samples from $\zeta$, define $K_Rf:=Rf/(R\vone)$. We claim
\begin{equation}
\label{eq:positive-ratio-empirical-bound}
\E\!\left[
\left|
\frac{\zeta^N(Rf)}{\zeta^N(R\vone)}
-
\frac{\zeta(Rf)}{\zeta(R\vone)}
\right|
\,\middle|\,\zeta
\right]
\le
\frac{\overline m}{2m\sqrt N}\osc(K_Rf)
\le
\frac{\overline m}{2m\sqrt N}\osc(f).
\end{equation}
\emph{Proof of}~\eqref{eq:positive-ratio-empirical-bound}. Set $a:=\zeta(Rf)/\zeta(R\vone)$. Then $\zeta(R(f-a))=\zeta(Rf)-a\zeta(R\vone)=0$, so
\[
\frac{\zeta^N(Rf)}{\zeta^N(R\vone)}-\frac{\zeta(Rf)}{\zeta(R\vone)}
=\frac{\zeta^N(R(f-a))}{\zeta^N(R\vone)}
=\frac{(\zeta^N-\zeta)(R(f-a))}{\zeta^N(R\vone)}.
\]
Now $K_Rf$ is a pointwise positive average of $f$, so $a$ is an $R\vone\,\zeta$-weighted average of $K_Rf$, hence $|K_Rf(x)-a|\le \sup K_Rf - \inf K_Rf=\osc(K_Rf)$ for all $x$. Writing $g:=R(f-a)=R\vone\cdot(K_Rf-a)$ and using $0\le R\vone\le \overline m$, we get $\osc(g)\le \overline m\osc(K_Rf)$. Popoviciu's variance inequality gives $\Var_\zeta(g)\le \osc(g)^2/4\le \overline m^2\osc(K_Rf)^2/4$. The empirical-mean variance bound $\E[|(\zeta^N-\zeta)(g)|\mid\zeta]\le \sqrt{\Var_\zeta(g)/N}$ and $\zeta^N(R\vone)\ge m$ deterministically yield the first inequality in~\eqref{eq:positive-ratio-empirical-bound}. The second inequality is $\osc(K_Rf)\le\osc(f)$, which holds because $K_R$ is a positive averaging operator.

\paragraph{Nonlinear telescoping.}
For $0\le k\le M$ define the multi-step Feynman--Kac map
\[
\Phi_{k,M}(\mu)(f):=
\frac{\mu(\mathsf H_{k,M}f)}{\mu(\mathsf H_{k,M}\vone)}.
\]
Set $\mu_0^N:=q_0^\Delta$ and $\mu_k^N:=\Phi_{k-1}(q_{k-1}^{N,\Delta})$ for $k\ge 1$. By the conditional-sampling paragraph, conditional on $\mathcal F_{k-1}$, $q_k^{N,\Delta}$ is the empirical measure of $N$ i.i.d.\ samples from $\mu_k^N$.

The flow identity $\Phi_{k,M}\circ\Phi_{k-1}=\Phi_{k-1,M}$ (which is just the chain rule for the Feynman--Kac semigroup; both sides equal $\mu\mapsto \mu(\mathsf H_{k-1,M}f)/\mu(\mathsf H_{k-1,M}\vone)$ since $\mathsf H_{k-1,M}=(W_{k-1}\mathsf P_{k-1})\mathsf H_{k,M}$) gives, for $k\ge 1$,
\[
\Phi_{k,M}(\mu_k^N)
=
\Phi_{k,M}\big(\Phi_{k-1}(q_{k-1}^{N,\Delta})\big)
=
\Phi_{k-1,M}(q_{k-1}^{N,\Delta}).
\]
Together with $\Phi_{0,M}(\mu_0^N)=\Phi_{0,M}(q_0^\Delta)=q_M^\Delta$ and $\Phi_{M,M}(q_M^{N,\Delta})=q_M^{N,\Delta}$, this yields the telescoping
\begin{equation}
\label{eq:nonlinear-telescoping-app}
q_M^{N,\Delta}(f)-q_M^\Delta(f)
=
\sum_{k=0}^{M}
\left[
\Phi_{k,M}(q_k^{N,\Delta})(f)-\Phi_{k,M}(\mu_k^N)(f)
\right].
\end{equation}
The $k=0$ term is the initial-sampling error; the terms $k\ge1$ are the errors introduced by resampling and mutation at later grid times.

\paragraph{Summing the transported errors.}
For each $k\in\{0,\ldots,M\}$, apply~\eqref{eq:positive-ratio-empirical-bound} conditionally on $\mathcal F_{k-1}$ (or $\mathcal F_{-1}:=\sigma(\emptyset)$ for $k=0$) with $\zeta=\mu_k^N$, $\zeta^N=q_k^{N,\Delta}$, and $R=\mathsf H_{k,M}$. By~\eqref{eq:future-weight-bounds}, $\overline m/m=B_{k,M}=\bar\beta_{k,M}$, hence
\[
\E\!\left[
\Big|\Phi_{k,M}(q_k^{N,\Delta})(f)-\Phi_{k,M}(\mu_k^N)(f)\Big|
\,\Big|\,\mathcal F_{k-1}
\right]
\le
\frac{\bar\beta_{k,M}\osc(f)}{2\sqrt N}.
\]
Taking the unconditional expectation (the right-hand side is deterministic) and summing over $k$ via~\eqref{eq:nonlinear-telescoping-app} and the triangle inequality proves~\eqref{eq:smc-rate-app}. \hfill$\square$

\subsection{Incremental weight variance}
\label{app:proof-incremental-weight-variance}

The next lemma supports the discussion at the end of Section~\ref{sec:many-particle}: it converts the variance-control objective of Section~\ref{sec:methodology} into a quantitative bound on the variance of the centered one-step incremental weight $\bar W_k(x):=\exp(\Delta t_k\bar g_k(x))$, where $\bar g_k(x):=g_{t_k}^\eff(x)-q_k^\Delta(g_{t_k}^\eff)$ is the residual centered potential.

\begin{lemma}[One-step weight-variance bound]
\label{lem:incremental-weight-variance}
For each grid step $k$,
\begin{equation}
\label{eq:incremental-weight-variance-main}
\Var_{q_k^\Delta}\!\big(\bar W_k\big)
\le
\Delta t_k^2\,
\exp\!\big(2\Delta t_k\norminf{\bar g_k}\big)\,
\Var_{q_k^\Delta}\!\big(g_{t_k}^\eff\big).
\end{equation}
Multiplying $W_k$ by the deterministic constant $\exp(-\Delta t_k\,q_k^\Delta(g_{t_k}^\eff))$ leaves the normalized recursion~\eqref{eq:main-fixed-grid-recursion} unchanged, so~\eqref{eq:incremental-weight-variance-main} controls the per-step variance contribution to the SMC dynamics.
\end{lemma}

\begin{proof}
Let $\Xi\sim q_k^\Delta$ and set $X:=g_{t_k}^\eff(\Xi)-q_k^\Delta(g_{t_k}^\eff)$. Then $\E[X]=0$, $\Var(X)=\Var_{q_k^\Delta}(g_{t_k}^\eff)$, and $|X|\le\norminf{\bar g_k}$. With $\bar W_k=e^{\Delta t_k X}$, $\Var(\bar W_k)\le\E[(\bar W_k-1)^2]$. The mean-value theorem gives $|e^{\Delta t_k X}-1|\le\Delta t_k\,e^{\Delta t_k\norminf{\bar g_k}}|X|$. Squaring and taking expectations yields
$\Var_{q_k^\Delta}(\bar W_k)\le \Delta t_k^2\,e^{2\Delta t_k\norminf{\bar g_k}}\,\E[X^2]=\Delta t_k^2\,e^{2\Delta t_k\norminf{\bar g_k}}\,\Var_{q_k^\Delta}(g_{t_k}^\eff)$,
which proves~\eqref{eq:incremental-weight-variance-main}.
\end{proof}

\subsection{Discussion of the theorems}
\label{app:discussion-thms}

\begin{remark}[Scope of the score-ratio theorem]
\label{rem:learned-base-error}
Theorem~\ref{thm:ratio-error} does not treat $Q_t^\bwd$ as fixed independently of $\widehat s_t$. It covers the usual discrete-diffusion situation in which $Q_{T-t}^\fwd$ is known and the implemented reverse generator is
\[
\widehat Q_t^\bwd(y,x)=Q_{T-t}^\fwd(x,y)\widehat s_t(x,y).
\]
Thus the score error inside $Q_t^\bwd$ is part of the bound. What remains outside the theorem are errors not expressible as this known-forward local-ratio perturbation, such as a misspecified forward noising kernel, numerical generator error, or an independently learned off-diagonal rate not of the form $Q_{T-t}^\fwd\widehat s_t$.
\end{remark}

\begin{remark}[Regularity, coverage, and accuracy]
\label{rem:assumptions-discussion}
The uniform quantities in Theorem~\ref{thm:ratio-error} are \emph{regularity} and \emph{coverage} conditions, not score-accuracy conditions: none of $B$, the well-posedness bound $\Lambda_{\rm op}$, $\underline\kappa$, or $\overline\kappa$ shrinks as the score estimate improves. Score quality enters only through the standard score-entropy loss $\mathcal L_{\rm DD}(t)$ in~\eqref{eq:dd-training-loss-main}, paid by the edgewise endpoint coverage profile $C_{\rm DD}(t)$ aggregated in time as $\mathfrak C_{\rm DD}$. Consequently, if $\widehat s_t\to s_t$ in $\mathfrak L_{\rm DD}$ and $\mathfrak C_{\rm DD}<\infty$, then $\TV(q_T^{\widehat s},q_T)\to0$ at rate $\sqrt{\mathfrak L_{\rm DD}}$. Part~(a) of Theorem~\ref{thm:ratio-error-formal} additionally records the proof-internal averaged loss $\mathcal E_{\rm op}(t)$ under $\nu_t^{\rm op}$, with no state-wise supremum on $\widehat s_t/s_t$; identity~\eqref{eq:nu-op-se-identity-app} converts it to the score-entropy form at the price of $C_{\rm DD}(t)$.
\end{remark}

\begin{remark}[Particle theorem scope]
\label{rem:not-bound}
The result compares $q_M^{N,\Delta}$ \emph{only} with the exact terminal law $q_M^\Delta$ of the fixed grid model. It does not bound score-estimation error, training error, continuous-time discretization error, numerical error in simulating $\mP_k$, adaptive ESS resampling, or feedback from estimating \texttt{HEU} or \texttt{D-VCG} using the same finite particle cloud. The theorem should therefore be cited only as a particle-only consistency result for the fixed Feynman--Kac model. Every-step multinomial resampling is essential to the proof because it guarantees the conditional i.i.d.\ structure used in the conditional-sampling paragraph of the proof of Theorem~\ref{thm:many-particle-formal}.
\end{remark}

\section{Experimental Details}
\label{sec:experimental-details}

This appendix records the implementation details needed to reproduce the experiments of Section~\ref{sec:experiments}: the shared SMC implementation (Appendix~\ref{app:shared}), the finite-state CTMC benchmark (Appendix~\ref{app:toy-details}), and the 2D Ising benchmark (Appendix~\ref{app:ising-details}), including its model and reference sampler, sampler settings and \texttt{D-VCG} basis library, metrics, and configurations. The large-scale experiments of Appendix~\ref{app:large-scale} are described there. 

\subsection{Shared implementation}
\label{app:shared}

Both benchmarks use ESS-adaptive systematic resampling, triggered when $\mathrm{ESS}/N<\tau$, with default $\tau=0.5$. The nonnegative \texttt{D-VCG} coefficients are selected by active-set candidate enumeration with a small diagonal ridge. We select the candidate with the smallest variance-plus-anchor-penalty score before forming $\mQ_t^\eff=\sum_j\theta_j^\star\mQ_t^{(j)}$. Because candidate generation and scoring use distinct regularization terms, this is an approximate solve of the stated objective.

\subsection{Finite-state CTMC}
\label{app:toy-details}

\paragraph{Setup and canonical configurations.}
We use two random tensor-product CTMC families, \emph{uniform-state diffusion} (each site is independently refreshed to a uniform random vocabulary index at rate one) and \emph{masked-absorbing diffusion} (each site is independently absorbed to a special mask token at rate one), both initialized from a Dirichlet$(1)$ prior on the $V^L$ state space. The four canonical configurations all share $V=5$, $L=3$, $N=4000$ particles, an 80-step power-2 time grid, ESS threshold $0.5$, systematic resampling, and $10$ random seeds per cell; midpoint rates and potentials are used in a half-weight/propagation/half-weight splitting, and the reported comparisons use this fixed discretization. The configurations differ only in the diffusion family and per-regime tilt strength:
\emph{uniform-state/reward} with $\sigma_r=3.0$ (Gaussian random reward vector);
\emph{uniform-state/annealing} with $\gamma=3.0$;
\emph{masked-absorbing/reward} with $\sigma_r=1.0$ (Gaussian random reward vector);
and \emph{masked-absorbing/annealing} with $\gamma=1.3$, since at $\gamma>1.3$ the standard Feynman--Kac SMC recursion becomes numerically unstable on masked-absorbing diffusion.

\begin{figure}[t]
  \centering
  \begin{subfigure}[t]{0.48\textwidth}
    \includegraphics[width=\linewidth]{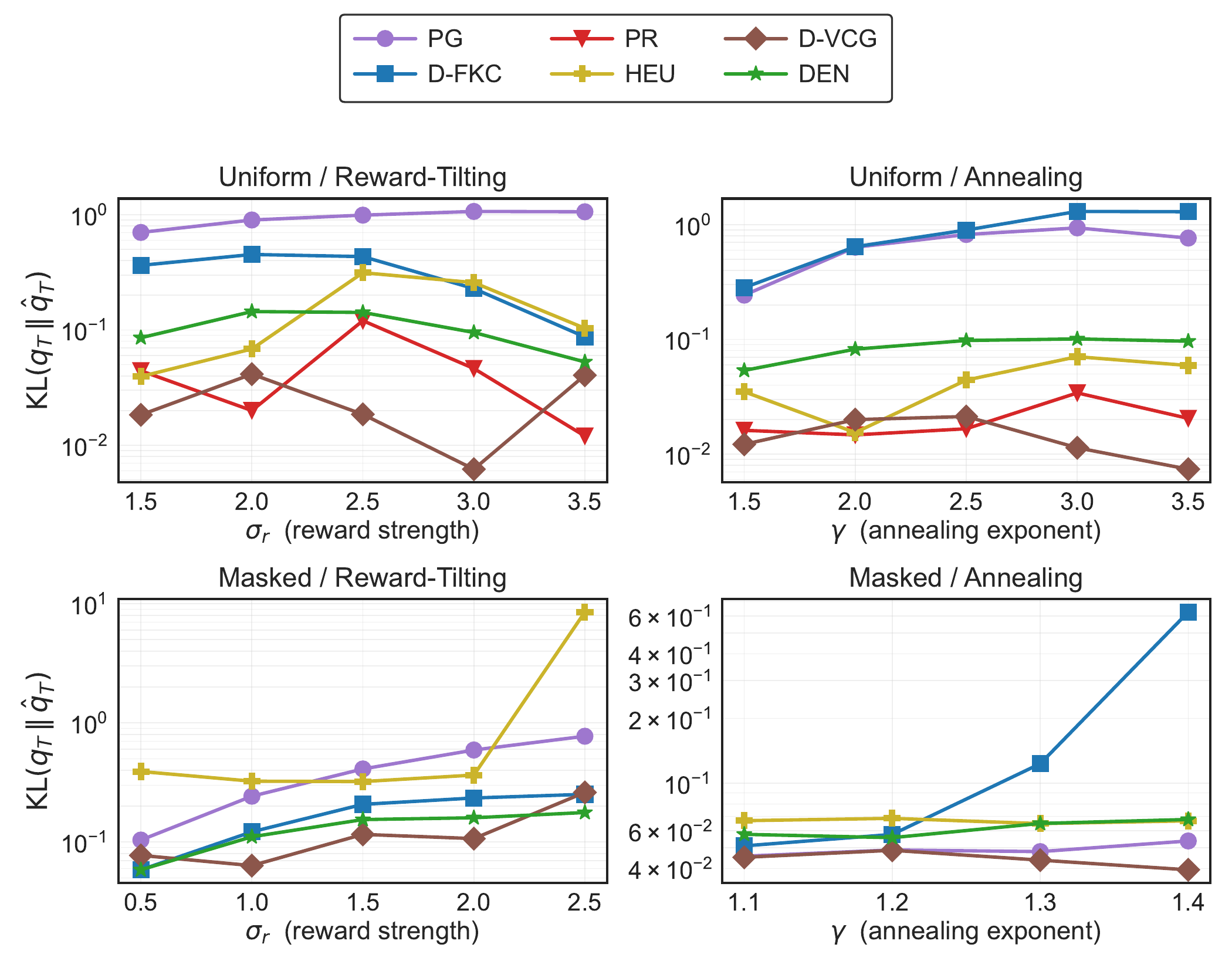}
    \caption{Sweep of $\KL(q_T\,\|\,\widehat q_T^N)$ versus reward scale
      $\sigma_r$ (reward regimes) or annealing exponent $\gamma$
      (annealing regimes) across the four canonical $V=5$, $L=3$
      configurations of Section~\ref{sec:toy}. Six samplers (\texttt{PG}, \texttt{D-FKC},
      \texttt{PR}, \texttt{HEU}, \texttt{D-VCG}, \texttt{DEN}). All methods are evaluated on the same fixed grid within each canonical regime.}
    \label{fig:toy-paper}
  \end{subfigure}
  \hfill
  \begin{subfigure}[t]{0.48\textwidth}
    \includegraphics[width=\linewidth]{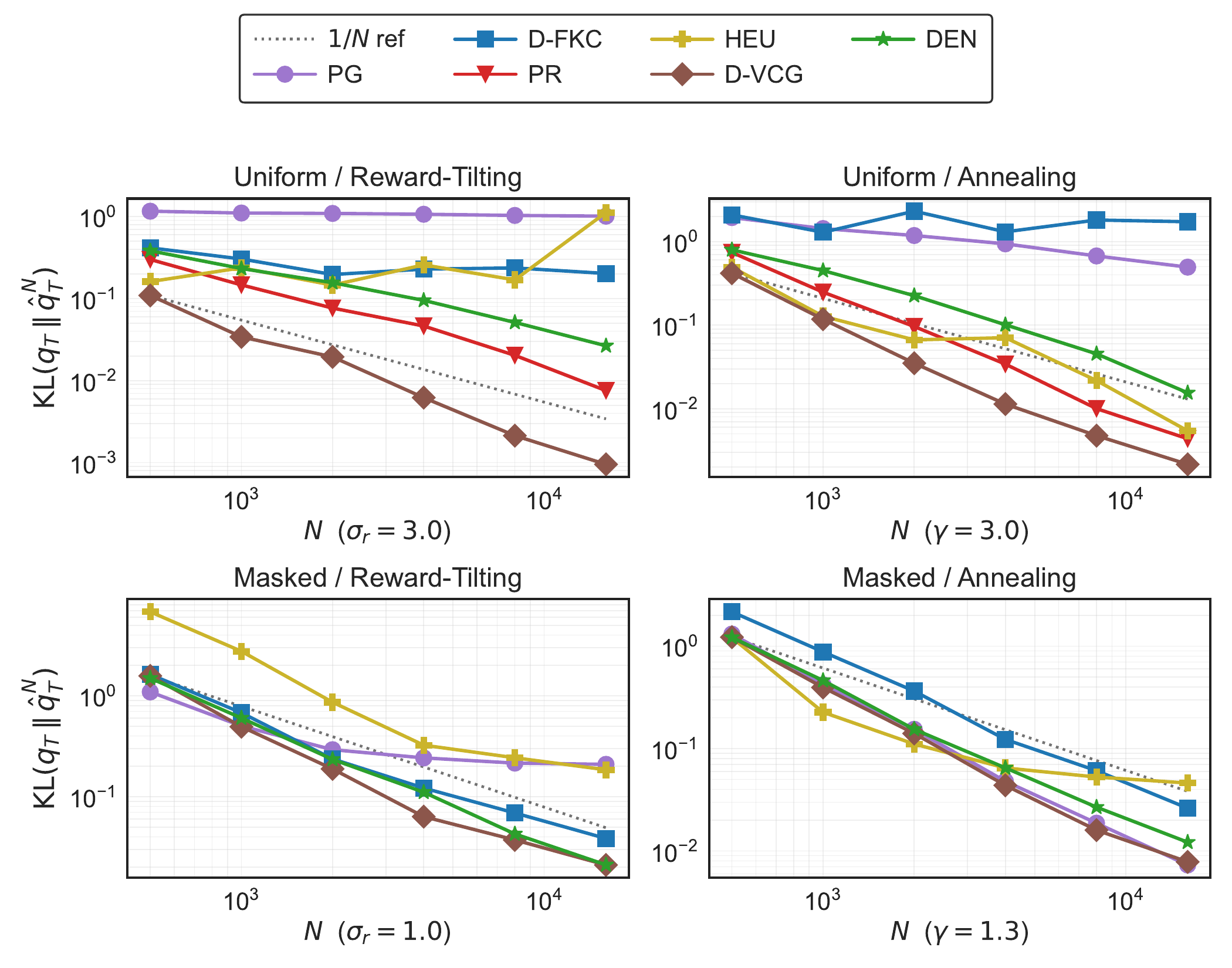}
    \caption{Particle-count scaling at the four canonical cells of
      Section~\ref{sec:toy} with $N\in\{500,\dots,16000\}$, log-log
      axes; the dashed line is a $1/N$ reference anchored at the
      smallest-$N$ \texttt{D-VCG} value. Each panel uses the per-regime tilt of the matching
      Fig.~\ref{fig:toy-row} cell.}
    \label{fig:toy-scaling}
  \end{subfigure}
  \caption{Finite-state CTMC ablations. Across both the parameter
    sweeps and the particle-count axis, \texttt{D-VCG} matches or improves on the dense
    oracle \texttt{DEN} without \texttt{DEN}'s tractability assumptions.}
  \label{fig:toy-ablations}
\end{figure}

\paragraph{\texttt{D-VCG} damping.}
The objective in Proposition~\ref{prop:disc-vcg} minimizes only the \emph{importance-weight} variance contribution; finite-$N$ Feynman--Kac SMC also incurs a \emph{propagation} variance from categorical sampling under $\mQ_t^\vcg$. When the initial backward marginal $p_{t_0}^\bwd$ already covers the support of $q_T$ well (typical for small dense state spaces), pushing particles by the unregularized least-squares solution $\vtheta^\star$ is wasteful: the propagation noise dominates. Multiplicatively damping $\vtheta_{\rm used}=s\,\vtheta^\star$ interpolates between pure reweighting ($s{=}0$) and the unregularized optimum ($s{=}1$); the toy experiments use $s=0.25$ in both uniform-state regimes, $s=0.75$ for masked reward tilting, and $s=1$ for masked annealing, playing the same role for \texttt{D-VCG} as the damping factor $\alpha_t$ does for \texttt{HEU}.

\paragraph{Metric, aggregation, and IS references.} For the finite-state CTMC, the closed-form $q_T$ allows the direct KL metric
\[
\KL(q_T\,\|\,\widehat q_T^N)
=
\sum_\sigma q_T(\sigma)\log\frac{q_T(\sigma)}{\widehat q_T^N(\sigma)},
\]
where $\widehat q_T^N$ is the weighted-particle histogram. Both probability vectors are clipped below at $\varepsilon=10^{-15}$ and renormalized before evaluating KL to avoid $\log 0$. Toy curves show geometric means over ten seeds; the main-figure bars are obtained by exponentiating the mean log KL plus or minus one standard deviation of log KL. The IS (prior) and IS (guided) baselines are self-normalized importance sampling with $N=4000$ samples drawn from $p_T^\bwd$ or the \texttt{D-FKC} guided proposal, reweighted by the exact density ratio $q_T/p_T^\bwd$ or $q_T/\widetilde q_T$. Both require closed-form $q_T$ and are tractable only on the toy.

\paragraph{Ablations.} Beyond the canonical cells reported in Fig.~\ref{fig:toy-row}, Figure~\ref{fig:toy-ablations} shows two ablations. The parameter sweep over $\sigma_r$ (reward) or $\gamma$ (annealing) shows regime-dependent performance. \texttt{PR} is competitive at some uniform-state parameter settings, while \texttt{D-VCG} improves over it in both main uniform-state configurations. \texttt{DEN} provides a dense zero-residual reference; its finite-particle terminal error need not be a lower bound for sparse methods. The particle-count sweep $N\in\{500,\dots,16{,}000\}$ includes a $1/N$ reference trend; the standard Feynman--Kac SMC baselines are much less stable on masked-absorbing/annealing.

\subsection{2D Ising}
\label{app:ising-details}

\subsubsection{Model, training, and reference}
\label{app:base-model}

\paragraph{Base model.} The score network is a U-Net trained with the denoising score-entropy loss under the \texttt{linear} diffusion schedule on 200{,}000 Swendsen--Wang samples at the training inverse-temperature $\beta_{\rm train}\in\{0.3,0.4\}$.

\paragraph{Target and reference.} The reward is the linear magnetization $r(\sigma)=\sum_i\sigma_i\equiv M(\sigma)$, so the tilted target $q_{\beta,\beta_r}(\sigma)\propto e^{-\beta H(\sigma)+\beta_r M(\sigma)}$ is the Ising model in a uniform external field $h_{\rm eff}=\beta_r/\beta$. For each $(\beta,\beta_r)$ cell we use $2000$ Swendsen--Wang reference samples. For nonzero fields, a fresh reference uses the augmented graph with a ghost spin coupled to every site at strength $h_{\rm eff}$ (standard ghost-spin construction~\citep{swendsen1987nonuniversal}). These references estimate $\langle M\rangle$, $\langle E\rangle$, and $C(r)$ under the target distribution up to MCMC error.

\subsubsection{Sampler}
\label{app:ising-sampler}

\paragraph{Time discretization.} We use the \texttt{linear} schedule of~\citep{lou2024discrete} with noise horizon $\sigma_{\max}=10$, discretized into $S$ equal-$\dif t$ steps (the grid size $M$ of Algorithm~\ref{alg:fk-smc}; we write $S$ to avoid a clash with the magnetization $M(\sigma)$). We use $S=5000$ for annealing and $S=2000$ for reward-tilt and joint experiments. For the binary 2D Ising spins, each frozen per-site $2\times 2$ generator is exponentiated exactly per step in closed form, which removes Euler error for the frozen site update.

\paragraph{Reward ramp.} For reward-tilt and joint experiments, we instantiate the realization $r_t(x)=\beta_t\,r(x)$ from Section~\ref{subsec:smc-scaling-discrete} with the linear ramp 
\[
\beta_t=\beta_r t/T=\beta_r(1-\sigma/\sigma_{\max}),\qquad \sigma:=T-t,\quad T=\sigma_{\max}
\] where $t$ is increasing reverse time and $\sigma$ is the decreasing noise level. The strength hyperparameter $\beta_r$ rescales the data-end reward (equivalently, $r$ is replaced by $\beta_r r$ before applying the unit-strength ramp of the main text). Thus $\beta_0=0$ at the noisy end (matching the Feynman--Kac-exact uniform-prior initialisation) and $\beta_T=\beta_r$ at the data end, giving an effective data-end potential $r_T=\beta_r r$. Pure annealing uses $\beta_t\equiv 0$.

\paragraph{\texttt{D-VCG} basis library.}
\label{app:bases}
We instantiate Proposition~\ref{prop:disc-vcg} with a small collection of nonnegative basis rate families of the form~\eqref{eq:Q-basis}, $\mQ_t^{(j)}(y,x)=\mQ_t^\bwd(y,x)\,\varphi_t^{(j)}(y,x)$, specified through their multipliers $\varphi_t^{(j)}$. The time-$t$ tilt parameters are the annealing exponent $\gamma=\beta_{\rm target}/\beta_{\rm train}$ and the reward strength $\beta_t\ge0$. Let $\mQ_t^\bwd$ denote the pretrained backward-rate family and let
\[
\ell_t(x,y):=\log\frac{p_t^\bwd(y)}{p_t^\bwd(x)} .
\]
For the Ising reward $r(\sigma)=\sum_i\sigma_i$, write $\Delta r(x,y)=r(y)-r(x)$ and $\Delta H(x,y)=H(y)-H(x)$, with $H(\sigma)=-J_{\rm Ising}\sum_{\langle ij\rangle}\sigma_i\sigma_j$ and $J_{\rm Ising}=1$ on the periodic lattice. The basis library is
\begin{itemize}[leftmargin=1.2em,topsep=2pt,itemsep=2pt,parsep=0pt]
  \item \textbf{Untilted backward} (canonical basis from the main text):
    \[
    \varphi_t^{(1)}(y,x):= 1.
    \]
  \item \textbf{Target-aligned anchor} (the Ising implementation scales the canonical main-text basis by $\gamma$):
    \[
    \varphi_t^{(2)}(y,x):= \gamma e^{(\gamma-1)\ell_t(x,y)+\beta_t\Delta r(x,y)}.
    \]
    The pure-annealing specialization $\varphi_t^{(\rm anneal)}(y,x):=\gamma e^{(\gamma-1)\ell_t(x,y)}$ (set $\beta_t=0$) and the pure reward-tilt specialization $\varphi_t^{(\rm tilt)}(y,x):=e^{\beta_t\Delta r(x,y)}$ (set $\gamma=1$) are used in the corresponding ablation regimes.
  \item \textbf{Intermediate annealing basis} (milder annealing tilt with exponent $\gamma_{\rm mid}:=(1+\gamma)/2$):
    \[
    \varphi_t^{(\gamma_{\rm mid})}(y,x):= \gamma_{\rm mid}e^{(\gamma_{\rm mid}-1)\ell_t(x,y)}.
    \]
  \item \textbf{Energy-flux basis} (Ising-specific basis that, departing from~\eqref{eq:Q-basis}, replaces the learned backward-rate backbone with the bare forward-rate prefactor $Q_{T-t}^{\fwd}(x,y)$, so as to isolate the energy-gradient direction from the learned score model):
    \[
    \mQ_t^{(E)}(y,x):= Q_{T-t}^{\fwd}(x,y)\,[-\beta_{\rm target}\Delta H(x,y)]_+ .
    \]
\end{itemize}
For pure annealing, the minimal library uses $\{\varphi^{(1)},\varphi^{(\rm anneal)}\}$; the augmented library adds $\mQ^{(E)}$ and, when $\gamma>1$, the midpoint basis. For joint annealing and reward tilting, the minimal library contains the backward, annealing, and joint bases; the augmented library adds the energy basis, with no midpoint extra. At $\gamma=1$, the redundant annealing basis is removed, giving two and three bases in pure tilt. Thus ``2-basis'' and ``4-basis'' are figure labels for controller variants, not uniform counts across regimes.

\paragraph{Anchor regularizer.}
\label{app:adaptive-lambda}
For the joint anneal\,$+$\,tilt regime (Config~2a in Appendix~\ref{app:ising-configs}; the pure-tilt and pure-annealing configurations are run with $\lambda_{\rm eff}\equiv 0$), the variance objective of Proposition~\ref{prop:disc-vcg} is augmented with a quadratic anchor-penalty term:
\begin{equation}
\label{eq:anchor-penalised-qp}
\min_{\vtheta\ge0}\;
\Var_{q_t}\!\bigg(g_t^0+\sum_i\theta_i\,(\Div_{q_t}\mQ_t^{(i)})\bigg)
+\lambda_{\rm eff}(t)\,\|\vtheta-\vtheta_{\rm anchor}\|_2^2,
\end{equation}
where $\vtheta_{\rm anchor}$ is the one-hot vector on the target-aligned proposal $\mQ^{(2)}$, and the schedule is
\begin{equation}
\label{eq:lambda-eff}
\lambda_{\rm eff}(t)=\mathrm{coef}\cdot\beta_t^2,
\qquad \mathrm{coef}=10^{4},
\end{equation}
combined with the ramp $\beta_t=\beta_r t/T$ in increasing reverse time, with $T=\sigma_{\max}$ (see the reward ramp above). At the noisy end ($t=0$, $\beta_t=0$) the anchor penalty is inactive; at the data end ($t=T$, $\beta_t=\beta_r$) it pulls $\vtheta$ toward the target-aligned committed proposal.

\paragraph{HEU on Ising.}
\texttt{HEU} is omitted from the Ising experiments because the empirical particle measure $\widehat q_t^N(y)=0$ on the overwhelming majority of single-flip neighbors of the $N$ active particles in $\sX=\{\pm1\}^{16\times 16}$, so the empirical-cloud evaluation of $\mu_t(y)$ used in the toy benchmark vanishes spuriously on most candidate destinations. The analytical-expansion form $\mu_t(y)=q_t(y)g_t(y)+\sum_{z\neq y}\big(Q_t(y,z)q_t(z)-Q_t(z,y)q_t(y)\big)$ recommended in Appendix~\ref{app:heu-discussion} resolves this issue without extra score-network evaluations, and adopting it is a natural next step for \texttt{HEU} on DLM-scale state spaces.

\subsubsection{Metrics and aggregation}
\label{app:metrics}

For Ising experiments, we evaluate four scalar metrics per cell, comparing $N$ generated particles to $2000$ SW reference samples:
\begin{itemize}[leftmargin=1.2em,topsep=2pt,itemsep=2pt]
\item $\Wass_2(|m|)$ (annealing) or $\Wass_2(M)$ (reward-tilt): the empirical 1-D Wasserstein-2 distance between the per-particle absolute per-site magnetization $|m(\sigma)|=|\sum_i\sigma_i|/L^2$ or total magnetization $M(\sigma)=\sum_i\sigma_i$ and the reference.
\item $\Wass_2(E_{\rm total})$: the same 1-D Wasserstein-2 distance applied to the total energy $H(\sigma)$. This metric is sensitive to the joint distribution since $H$ is non-linear in $\sigma$.
\item $\mathrm{MSE}(C(r))$: mean-squared error of the uncentered row correlation. For $L=16$ and zero-based lattice indices,
\[
    C(r;\sigma)=\frac{1}{(L-2)(L-2-r)}
    \sum_{i=1}^{L-2}\sum_{j=1}^{L-2-r}\sigma_{i,j}\sigma_{i,j+r},
    \qquad r=1,\ldots,L-3=13.
\]
This estimator excludes a one-site boundary margin and does not subtract magnetization. We average $C(r;\sigma)$ over generated and reference samples, then average the squared difference between these estimates over $r$.
\item $\langle|m|\rangle$ or $\langle M\rangle$: per-method mean of the moment observable, plotted against the SW reference for visual sanity.
\end{itemize}

Ising curves show means over three seeds, with bands equal to the across-seed standard deviation divided by $\sqrt{3}$. Improvement factors divide the seed-mean errors within each setting before taking geometric means or peaks across settings.

\subsubsection{Configurations and additional results}
\label{app:ising-configs}

The five Ising configurations are as follows. 
\begin{itemize}[leftmargin=1.2em,topsep=2pt,itemsep=2pt,parsep=0pt]
    \item \textbf{Config 1a (pure annealing, Fig.~\ref{fig:ising-anneal}):} trains at $\beta_{\rm train}=0.4$, sweeps $\beta_{\rm target}\in[0.20,0.55]$ over $9$ values with $\beta_r=0$, $\gamma\in[0.5,1.375]$, $N=500$, $S=5000$, and $\lambda_{\rm eff}=0$.
    \item \textbf{Config 1b (annealing on a paramagnetic base, Fig.~\ref{fig:ising-anneal-b03}):} trains at $\beta_{\rm train}=0.3$, sweeps $\beta_{\rm target}\in[0.20,0.60]$ over $10$ values with $\gamma\in[0.667,2.0]$, $N=500$, $S=5000$, $\lambda_{\rm eff}=0$, and a tighter ESS threshold $\tau=0.25$ chosen by ablation as best for this paramagnetic regime.
    \item \textbf{Config 2a (joint anneal\,$+$\,tilt, Fig.~\ref{fig:ising-tilt}):} uses $\beta_{\rm train}=0.4$, $\beta_{\rm target}=0.45$, $\gamma=1.125$ (so the reward tilt sits on top of a mild $\gamma>1$ annealing), sweeps $\beta_r\in[0.02,0.40]$ over $7$ values with $N=500$, $S=2000$, and the adaptive regularizer $\lambda_{\rm eff}(t)=10^4\beta_t(t)^2$ from \eqref{eq:lambda-eff}.
    \item \textbf{Config 2b (pure tilt, Fig.~\ref{fig:ising-tilt-bt045}):} uses $\beta_{\rm train}=\beta_{\rm target}=0.4$, $\gamma=1$ (so $\varphi_t^{(2)}$ collapses to the pure-tilt multiplier $e^{\beta_t\Delta r(x,y)}$), the same reward sweep $\beta_r\in[0.02,0.40]$, $N=500$, $S=2000$, and $\lambda_{\rm eff}=0$: pure tilt is run as the unregularized \texttt{D-VCG} baseline so that Config~2b isolates the variance objective alone, with the anchor penalty reserved for the joint anneal+tilt regime of Config~2a.
    \item \textbf{Config 3 (particle-count sweep, Fig.~\ref{fig:ising-N-ablation}):} uses $\beta_{\rm train}=\beta_{\rm target}=0.4$, $\gamma=1$, $\beta_r=0.20$, $\lambda_{\rm eff}=0$, $S=2000$, and sweeps $N\in\{100,200,500,1000,2000,5000\}$.
\end{itemize}

All five configurations use the matrix-exponential integrator, ESS-adaptive systematic resampling at the threshold listed above, $3$ random seeds, and the four-method comparison \texttt{PG}, \texttt{D-FKC}, \texttt{D-VCG} (2-basis), and \texttt{D-VCG} (4-basis/augmented).

\begin{figure}[t]
  \centering
  \begin{subfigure}[t]{0.48\textwidth}
    \includegraphics[width=\linewidth]{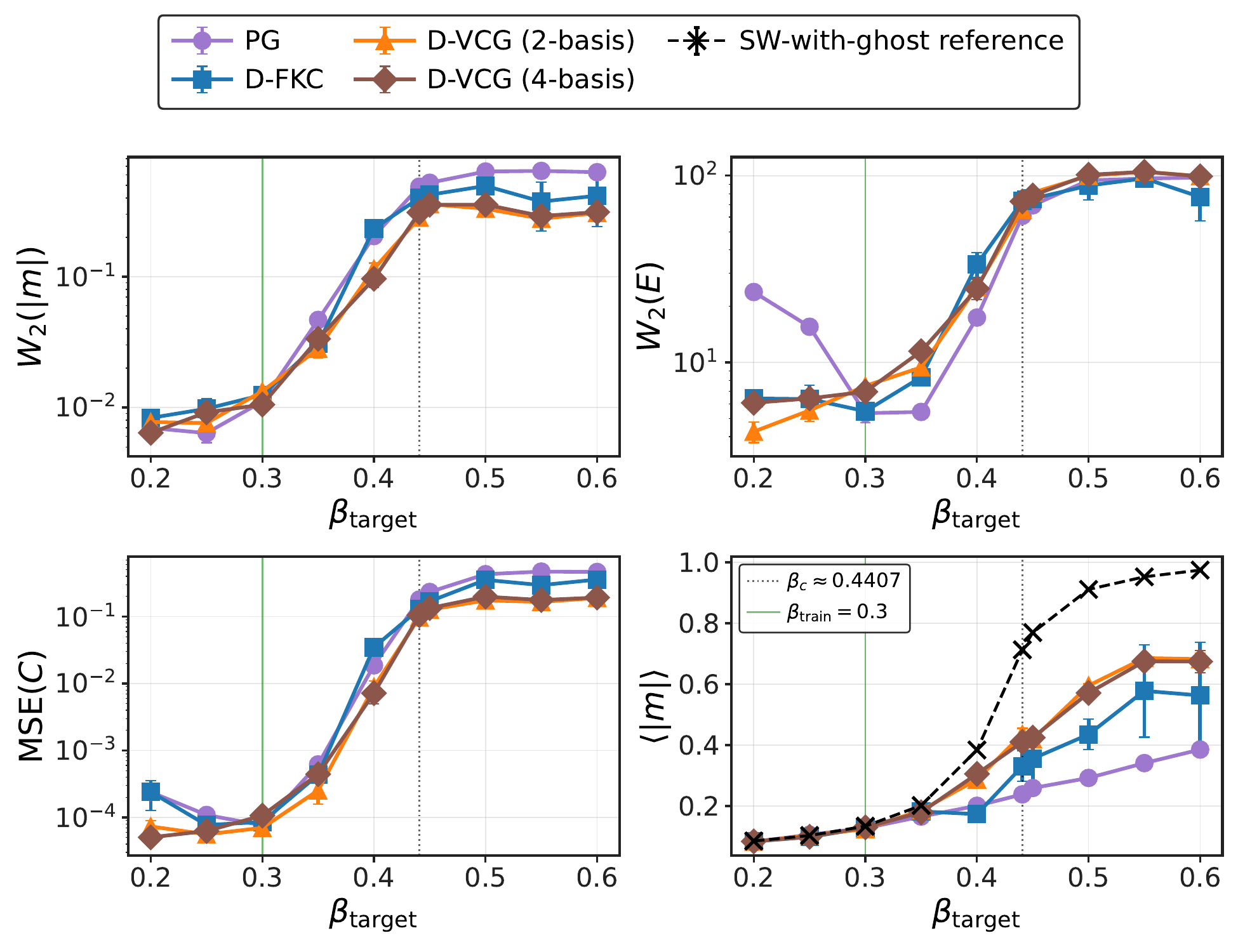}
    \caption{Annealing on a paramagnetic base
      ($\beta_{\rm train}=0.3$, $N=500$, $\tau=0.25$, $10$ target
      temperatures in $[0.20,0.60]$). \texttt{D-VCG} (4-basis) reduces
      $\mathrm{MSE}(C(r))$ over \texttt{D-FKC} by $1.69\times$ in geometric mean,
      with peak $\mathbf{4.84\times}$ at $\beta_{\rm target}=0.40$.}
    \label{fig:ising-anneal-b03}
  \end{subfigure}
  \hfill
  \begin{subfigure}[t]{0.48\textwidth}
    \includegraphics[width=\linewidth]{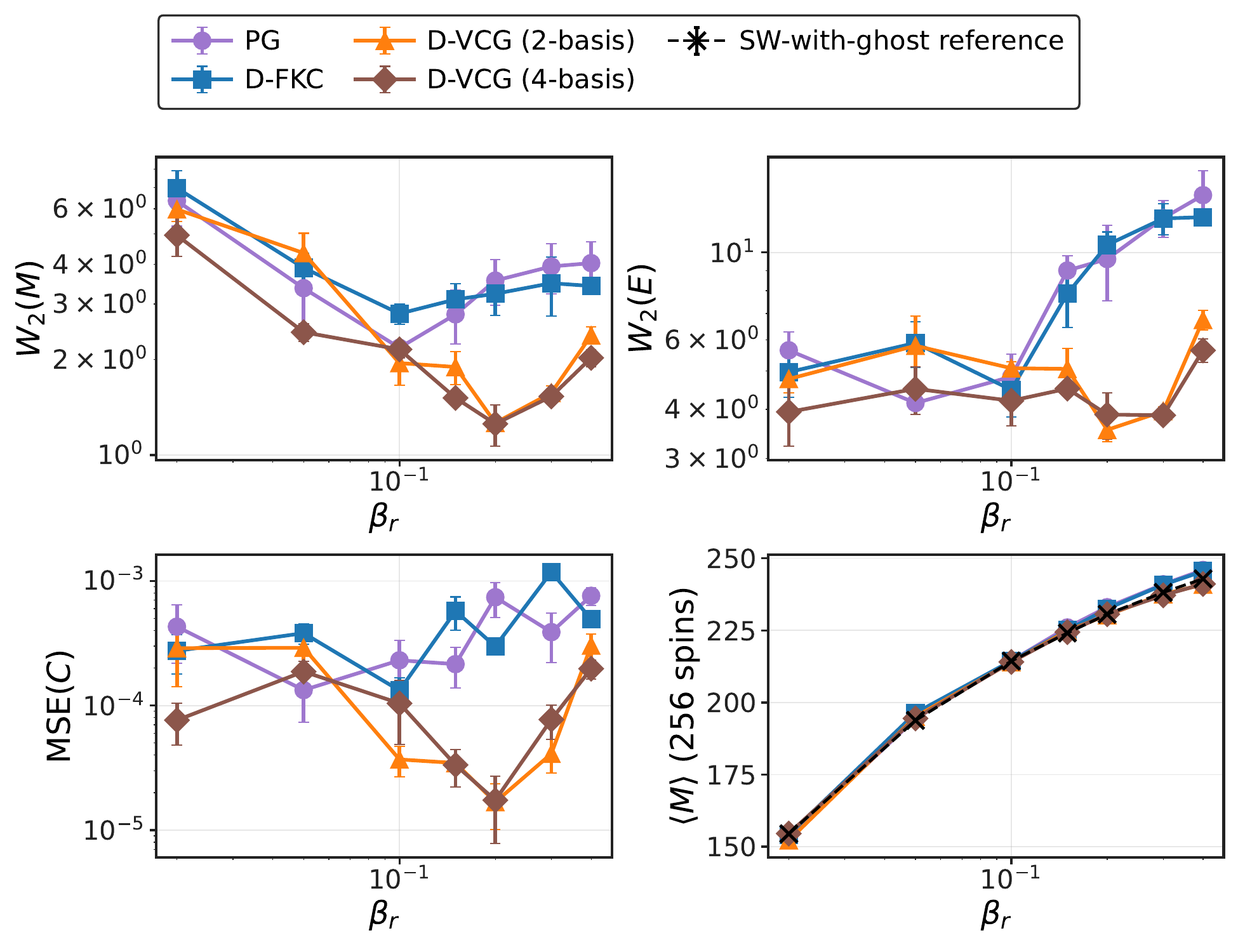}
    \caption{Pure tilt at
      $\beta_{\rm train}=\beta_{\rm target}=0.4$, $\gamma=1$,
      $\lambda_{\rm eff}=0$, $N=500$. With both the mild
      joint-anneal component and the anchor regularizer switched off,
      \texttt{D-VCG} still reduces $\mathrm{MSE}(C(r))$ over \texttt{D-FKC} by
      $5.21\times$ in geometric mean on $\beta_r\in[0.02,0.40]$, with peak
      $\mathbf{17.2\times}$ at $\beta_r=0.15$.}
    \label{fig:ising-tilt-bt045}
  \end{subfigure}
  \caption{Further 2D Ising configurations. The same axes and
    methods as Fig.~\ref{fig:ising} confirm that the \texttt{D-VCG}
    advantage (i) holds for a second base-model training
    temperature (\textbf{a}) and (ii) survives in the pure-tilt
    limit ($\gamma=1$, $\lambda_{\rm eff}=0$, panel \textbf{b}),
    where the joint-mode anchor reduces to the free-controller
    \texttt{D-VCG} of Proposition~\ref{prop:disc-vcg}.}
  \label{fig:ising-additional}
\end{figure}

\begin{figure}[t]
  \centering
  \includegraphics[width=0.6\textwidth]{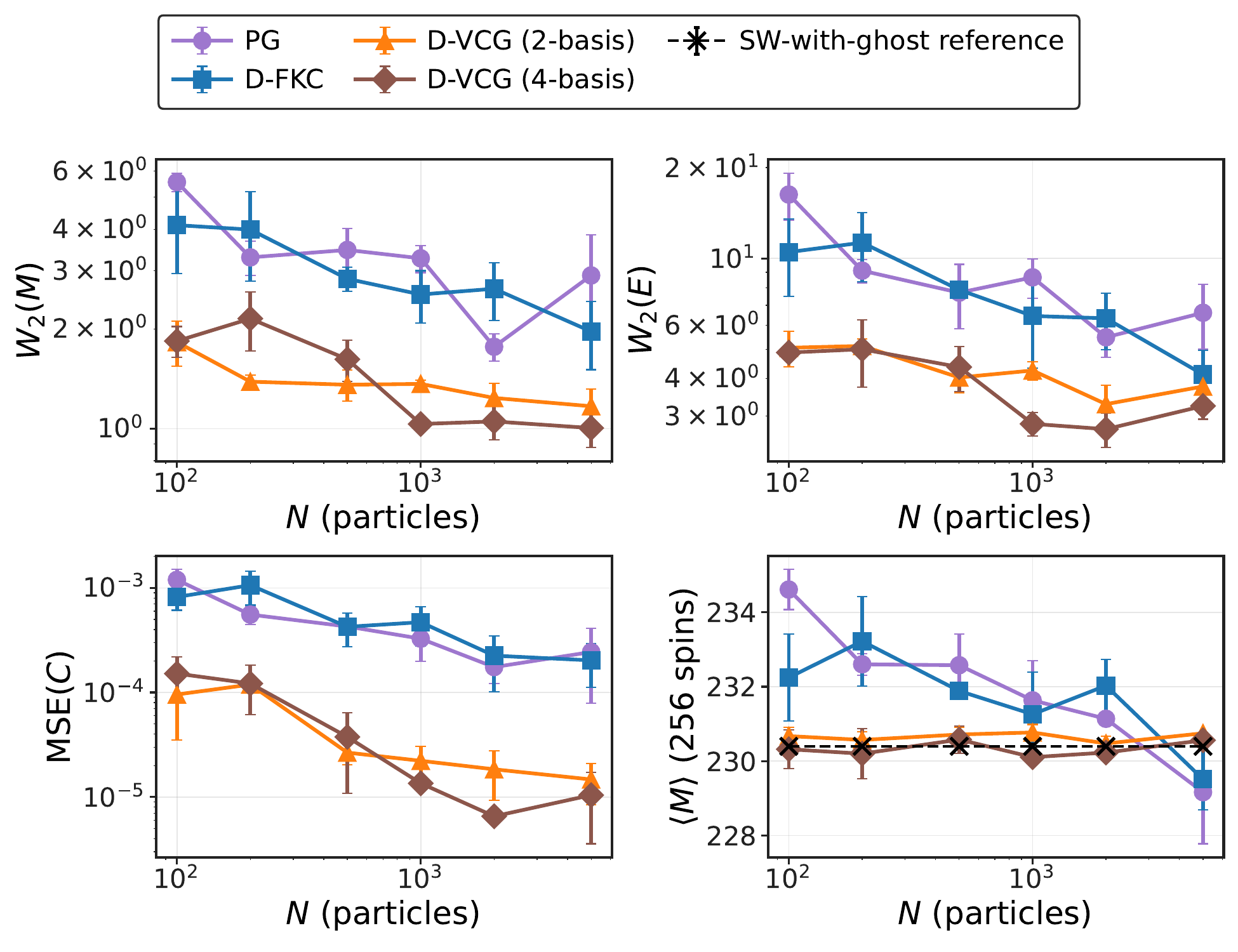}
  \caption{Particle-count sweep at $\beta_r=0.20$,
    $\beta_{\rm train}=\beta_{\rm target}=0.4$, $\gamma=1$,
    $\lambda_{\rm eff}=0$ (Config 3). Across
    $N\in\{100,\dots,5000\}$, \texttt{D-VCG} (4-basis) reduces
    $\mathrm{MSE}(C(r))$ over \texttt{D-FKC} by $15.21\times$ in geometric mean,
    with peak $\mathbf{34.7\times}$ at $N=1000$, showing that the
    advantage does not close as $N$ grows.}
  \label{fig:ising-N-ablation}
\end{figure}

\paragraph{Appendix experiments.}
Fig.~\ref{fig:ising-anneal-b03} re-runs the annealing sweep on a paramagnetic base ($\beta_{\rm train}=0.3$, well below criticality), checking that the gain over \texttt{D-FKC} does not depend on the base model being trained near the critical temperature; the gain is smaller because all samplers agree more closely with the SW reference, and it peaks at $\beta_{\rm target}=0.40$. Fig.~\ref{fig:ising-tilt-bt045} sets $\gamma=1$ and $\lambda_{\rm eff}=0$, removing both the joint-mode anchor and the regularizer of~\eqref{eq:anchor-penalised-qp}, so the remaining gain isolates the basis-mixing mechanism of the QP. Fig.~\ref{fig:ising-N-ablation} sweeps $N$ at $\beta_r=0.20$ in the Config~2b setting: the gap does not close over the tested range, and \texttt{D-FKC} at $N=5000$ ($10\times$ the particle budget) still loses to \texttt{D-VCG} at $N=500$ by $5.4\times$ on $\mathrm{MSE}(C(r))$.

\paragraph{Compute.}
On a single A100, the Ising per-step wall-clock for \texttt{D-VCG} (4-basis) is within roughly $5\%$ of \texttt{D-FKC} at $N=500$: the matrix-exponential integrator and score-network forward together account for the dominant cost, while the $\le 5\times 5$ active-set QP solve is negligible.

\section{Additional Large-Scale Experiments}
\label{app:large-scale}

This appendix probes \texttt{D-VCG} (Proposition~\ref{prop:disc-vcg}) on two large-scale settings that go beyond the analytic CTMC and Ising benchmarks of Section~\ref{sec:experiments}. Appendix~\ref{app:music-infill} reports a discrete inverse problem on monophonic music sequences (a reward-tilted special case of Section~\ref{subsec:smc-scaling-discrete} with $\gamma=1$), and Appendix~\ref{app:cfg-t2i} reports text-to-image generation under classifier-free guidance.

\subsection{Discrete inverse problem on monophonic music (reward-tilted sampling)}
\label{app:music-infill}

We follow the music inverse-problem setup of~\citep{chu2025split}. As in~\citep{campbell2022continuous}, we preprocess the Lakh pianoroll dataset~\citep{dong2018musegan} into monophonic note sequences of length $L = 256$ over a per-site vocabulary of size $V = 129$, and use the SEDD~\citep{lou2024discrete} checkpoint of~\citep{chu2025split} as the pretrained reverse model on the $V^L$ state space. The observation model is
\begin{equation*}
y = G(x) + n,
\end{equation*}
where $G$ is a random masking operator that reveals a fraction $\rho \in \{40\%, 60\%\}$ of the positions of $x$ and $n$ is the corruption noise. The induced log-likelihood plays the role of the reward of Section~\ref{subsec:smc-scaling-discrete},
\begin{equation*}
r(x) \;:=\; \log p(y \mid x) \;=\; -\frac{1}{\sigma_y}\,\|G(x) - y\|_0 \;+\; \mathrm{const},
\end{equation*}
with noise scale $\sigma_y = 0.1$, so the target $q(x)\propto p_T^\bwd(x)\,e^{r(x)}$ is the reward-tilted special case of Proposition~\ref{prop:debiased-discrete-dynamics} at $\gamma = 1$. Following~\citep{chu2025split}, we evaluate generated samples by two scalar metrics: (i) the Hellinger distance between the histograms of generated and ground-truth notes, measuring fidelity to the prior, and (ii) the \emph{measurement error}, defined as the fraction of revealed positions on which $G(x)$ disagrees with $y$.

We compare \texttt{D-VCG} against two baselines. \texttt{SGDD}~\citep{chu2025split} is a state-of-the-art split-Gibbs sampler; \texttt{D-FKC}~\citep{hasan2026discrete} is the standard Feynman--Kac SMC baseline from Section~\ref{sec:experiments}. To ensure a fair comparison, we fix the total number of function evaluations (NFEs) at $1024$ for all three methods: \texttt{SGDD} uses $32$ outer iterations and $32$ inner denoising steps; \texttt{D-FKC} and \texttt{D-VCG} both use $M = 128$ denoising steps and $N = 8$ particles. \texttt{D-VCG} is instantiated with the two bases of Section~\ref{sec:basis} (untilted backward and target-aligned anchor).

\begin{table}[H]
\centering
\caption{Quantitative evaluation on monophonic music infilling. Lower is better.}
\label{tab:music-infilling}
\begin{tabular}{lcccc}
\toprule
& \multicolumn{2}{c}{$\rho = 40\%$} & \multicolumn{2}{c}{$\rho = 60\%$} \\
\cmidrule(lr){2-3} \cmidrule(lr){4-5}
& Hellinger $\downarrow$ & Meas. Error $\downarrow$ & Hellinger $\downarrow$ & Meas. Error $\downarrow$ \\
\midrule
\texttt{SGDD}  & 0.076 & 4.93 & 0.151 & 6.11 \\
\texttt{D-FKC} & \textbf{0.033} & 0.29 & 0.081 & 0.88 \\
\texttt{D-VCG} & 0.035 & \textbf{0.00} & \textbf{0.079} & \textbf{0.05} \\
\bottomrule
\end{tabular}
\end{table}

Table~\ref{tab:music-infilling} reports the quantitative results. At both masking ratios $\rho \in \{40\%, 60\%\}$, the two Feynman--Kac samplers \texttt{D-VCG} and \texttt{D-FKC} substantially outperform \texttt{SGDD} in Hellinger distance. Under the measurement-error metric, \texttt{D-VCG} dominates both baselines.

\subsection{Text-to-image generation (classifier-free guidance)}
\label{app:cfg-t2i}

Following the experimental setup of~\citep{ou2026inference}, we evaluate \texttt{D-VCG} on text-to-image generation with Meissonic~\citep{bai2025meissonic} as the pretrained discrete-diffusion base model. We report two text--image alignment scores: the Human Preference Score v2 (\textsc{HPSv2})~\citep{wu2023human} and the Multi-Dimensional Human Preference Score (\textsc{MPS})~\citep{zhang2024learning}. The prompt set consists of $25$ prompts drawn from each of four categories (anime illustration, concept art, acrylic painting, and watercolor painting), for a total of $100$ prompts. Classifier-free guidance (CFG) targets a geometric combination of the unconditional and conditional reverse marginals,
\begin{equation}
\label{eq:cfg-target}
q^{\rm cfg}_t(x) \;\propto\; p_t^\bwd(x)^{1-s}\, p_t^\bwd(x \mid c)^{s},
\end{equation}
where $c$ denotes the input prompt and $s$ is the guidance scale; we fix $s = 9$ following~\citep{bai2025meissonic}. Equation~\eqref{eq:cfg-target} is a particular instance of the discrete tilted path of Section~\ref{subsec:smc-scaling-discrete}: the identifications
\[
\gamma \;=\; 1-s, \qquad r_t(x) \;=\; s\,\log p_t^\bwd(x \mid c)
\]
recover $q^{\rm cfg}_t \propto (p_t^\bwd)^\gamma e^{r_t}$. Although here $\gamma < 0$ -- so the canonical representative of Proposition~\ref{prop:debiased-discrete-dynamics} is not a valid rate family -- the basis-controlled \texttt{D-VCG} construction below still produces a nonnegative effective rate $\mQ_t^\eff$. Since~\citep{ou2026inference} shows that \texttt{D-FKC} already dominates earlier inference-time baselines on CFG, we restrict the comparison to \texttt{D-FKC} vs.\ \texttt{D-VCG}.

\paragraph{CFG-specific bases for \texttt{D-VCG}.}
We instantiate \texttt{D-VCG} with two per-token rate bases naturally suggested by the CFG structure~\eqref{eq:cfg-target}. At each masked position, the transformer yields two categorical distributions over the per-token vocabulary $\mathcal{V}$: the unconditional $p_t^\bwd(\cdot \mid \varnothing)$ (empty prompt) and the conditional $p_t^\bwd(\cdot \mid c)$. We take $\mQ_t^\bwd$ to be the unconditional reverse rate family of Section~\ref{subsec:smc-scaling-discrete} and define the per-token log local ratios
\begin{equation}
\label{eq:cfg-scores}
\ell_t(x, y \mid \varnothing) := \log\frac{p_t^\bwd(y \mid \varnothing)}{p_t^\bwd(x \mid \varnothing)}, \qquad \ell_t(x, y \mid c) := \log\frac{p_t^\bwd(y \mid c)}{p_t^\bwd(x \mid c)},
\end{equation}
in the same convention as $\ell_t$ in Appendix~\ref{app:bases}. Following the basis form $\mQ_t^{(j)}(y,x) = \mQ_t^\bwd(y,x)\,\varphi_t^{(j)}(y,x)$ of~\eqref{eq:Q-basis}, the two CFG multipliers are
\begin{align}
\label{eq:cfg-basis-1}
\varphi_t^{(1)}(y, x) &:= 1, \\
\label{eq:cfg-basis-2}
\varphi_t^{(2)}(y, x) &:= \exp\!\big(s\,[\ell_t(x, y \mid c) - \ell_t(x, y \mid \varnothing)]\big) = \left(\frac{p_t^\bwd(y \mid c)\, p_t^\bwd(x \mid \varnothing)}{p_t^\bwd(x \mid c)\, p_t^\bwd(y \mid \varnothing)}\right)^{s}.
\end{align}
Here $\varphi_t^{(1)}$ is the \emph{untilted backward} basis of Section~\ref{sec:basis} (which simply returns the unconditional reverse rate $\mQ_t^\bwd$), while $\varphi_t^{(2)}$ is the \emph{target-aligned anchor} of Section~\ref{sec:basis} for the parameter identification $(\gamma, r_t) = (1-s,\, s\log p_t^\bwd(\cdot \mid c))$: the corresponding rate
\[
Q_t^{(2)}(y,x) \;=\; Q_{T-t}^\fwd(x,y)\,\left(\frac{p_t^\bwd(y \mid \varnothing)}{p_t^\bwd(x \mid \varnothing)}\right)^{\!1-s}\!\left(\frac{p_t^\bwd(y \mid c)}{p_t^\bwd(x \mid c)}\right)^{\!s}
\]
is the exact CFG-anchored Feynman--Kac proposal, recovered from the local-ratio decomposition $Q_t^\bwd(y,x) = Q_{T-t}^\fwd(x,y)\,p_t^\bwd(y \mid \varnothing)/p_t^\bwd(x \mid \varnothing)$ of Section~\ref{subsec:smc-scaling-discrete}. Both methods use $M = 64$ denoising steps and $N = 8$ particles, and we report the mean and standard error of each metric across the $100$ prompts.

\begin{table}[H]
\centering
\caption{Quantitative evaluation on CFG text-to-image generation: mean $\pm$ standard error over $100$ prompts. Higher is better.}
\label{tab:cfg-quant}
\begin{tabular}{lcc}
\toprule
& \textsc{MPS} $\uparrow$ & \textsc{HPSv2} $\uparrow$ \\
\midrule
\texttt{D-FKC} & $15.282 \pm 0.032$ & $0.2980 \pm 0.0003$ \\
\texttt{D-VCG} & $\mathbf{15.422} \pm 0.032$ & $\mathbf{0.3033} \pm 0.0003$ \\
\bottomrule
\end{tabular}
\end{table}

\paragraph{Quantitative results.}
Table~\ref{tab:cfg-quant} shows that \texttt{D-VCG} outperforms \texttt{D-FKC} on both metrics, with higher mean \textsc{MPS} ($15.422$ vs.\ $15.282$) and \textsc{HPSv2} ($0.3033$ vs.\ $0.2980$).

\paragraph{Qualitative results.}
To complement Table~\ref{tab:cfg-quant}, Table~\ref{tab:cfg-comparison} shows six representative prompts spanning anime illustration, concept art, acrylic painting, and watercolor painting. \texttt{D-VCG} generally adheres more closely to the prompt: for the dragon-rider prompt it renders both the rider and the trailing scarf, while \texttt{D-FKC} omits the rider entirely, and across the remaining prompts it more accurately reflects fine-grained details such as ``wooden bridges,'' ``cliffs,'' and the co-occurrence of ``orchids, ferns, and delicate green leaves.'' These observations are consistent with the gains on \textsc{MPS} and \textsc{HPSv2} in Table~\ref{tab:cfg-quant}.

\newcolumntype{C}[1]{>{\centering\arraybackslash}m{#1}}

\begin{table*}[!htbp]
\centering
\caption{Qualitative comparison of \texttt{D-FKC} and \texttt{D-VCG} on six representative CFG prompts.}
\label{tab:cfg-comparison}
\renewcommand{\arraystretch}{1.3}
\setlength{\tabcolsep}{3pt}
\resizebox{\textwidth}{!}{%
\begin{tabular}{C{1.2cm}@{\hskip 8pt}C{0.3\textwidth}C{0.3\textwidth}C{0.3\textwidth}}
\toprule
\textbf{Prompt} 
& \makecell{\textit{An anime dragon rider} \\ \textit{above snowy mountains,} \\ \textit{scarf trailing behind} \\ \textit{in the wind.}}
& \makecell{\textit{A dramatic anime close-up} \\ \textit{of a masked hero under} \\ \textit{falling snow, city skyline} \\ \textit{behind.}}
& \makecell{\textit{Concept art of a giant} \\ \textit{mechanical whale swimming} \\ \textit{through clouds above} \\ \textit{a coastal town.}} \\
\midrule
\textbf{\texttt{D-FKC}}
& \includegraphics[width=0.3\textwidth]{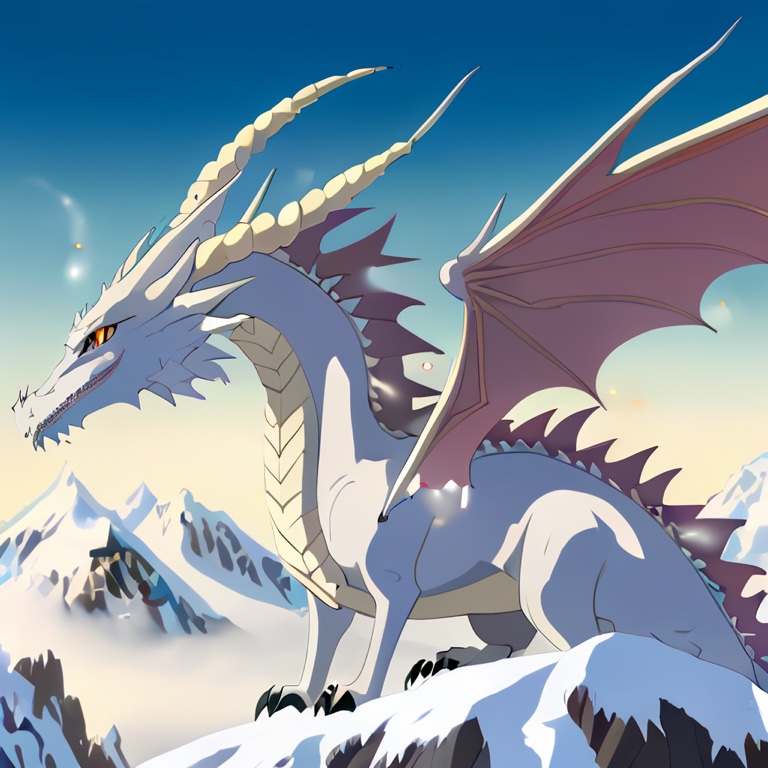}
& \includegraphics[width=0.3\textwidth]{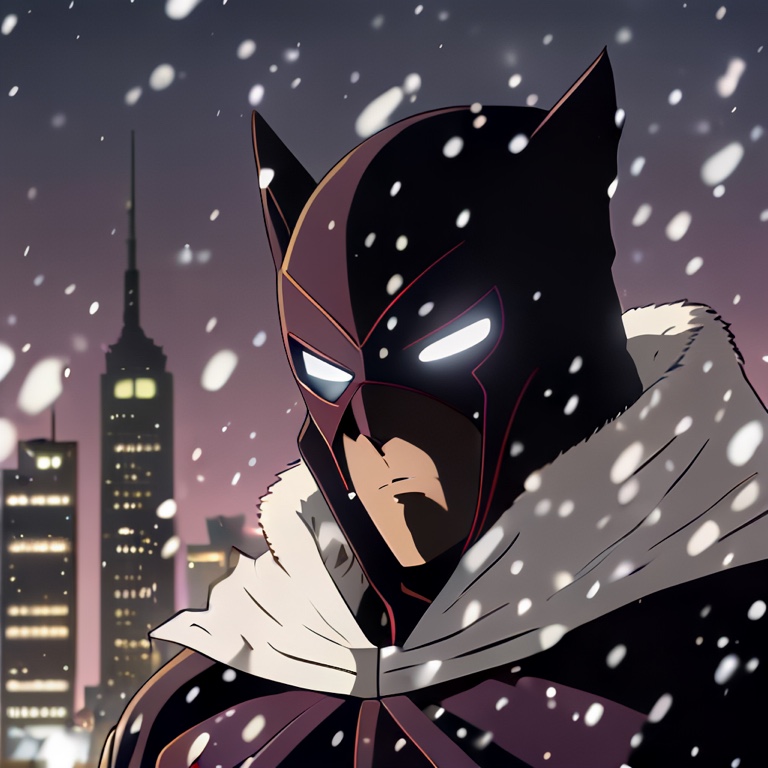}
& \includegraphics[width=0.3\textwidth]{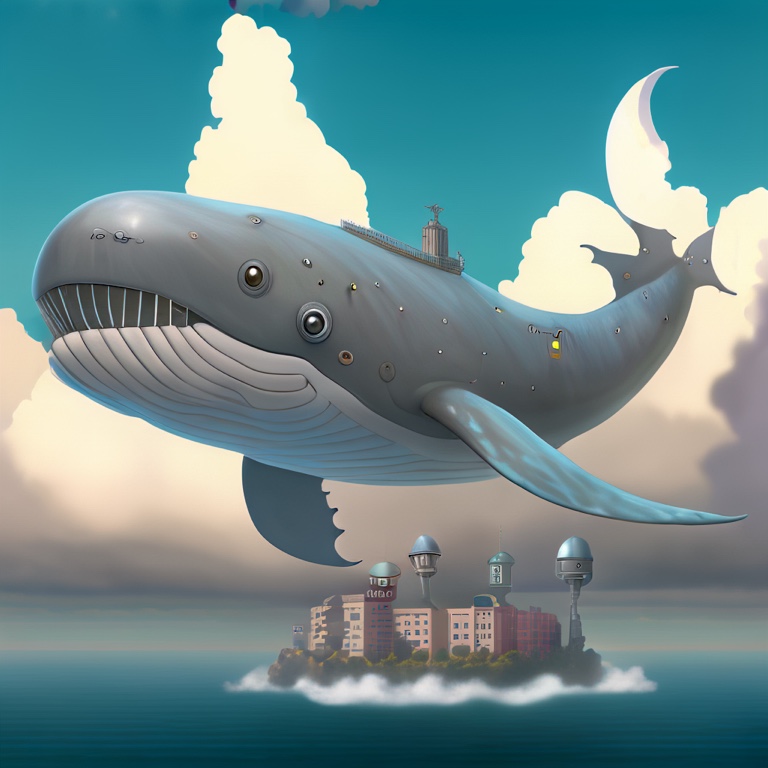} \\
\textbf{\texttt{D-VCG}}
& \includegraphics[width=0.3\textwidth]{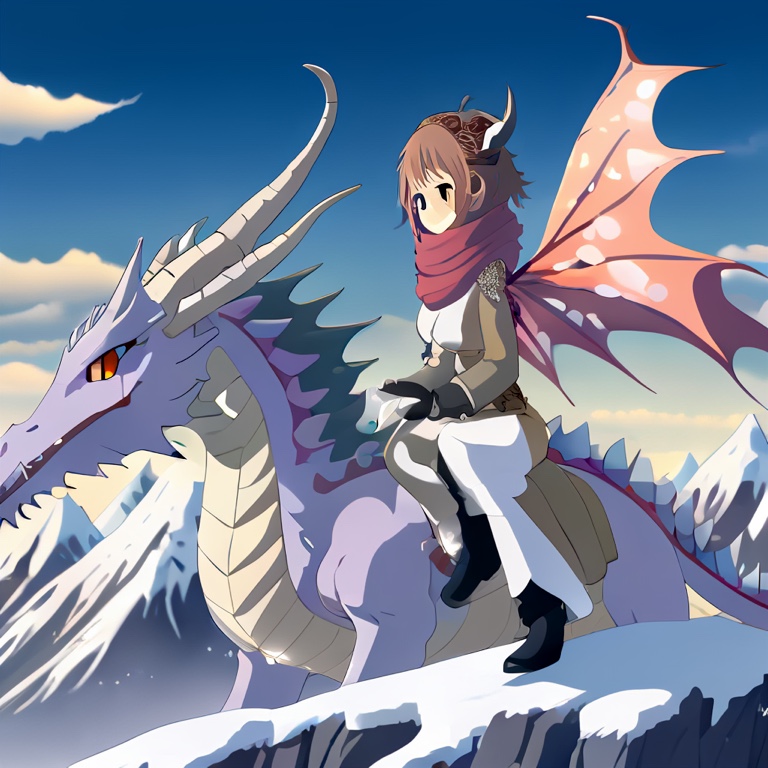}
& \includegraphics[width=0.3\textwidth]{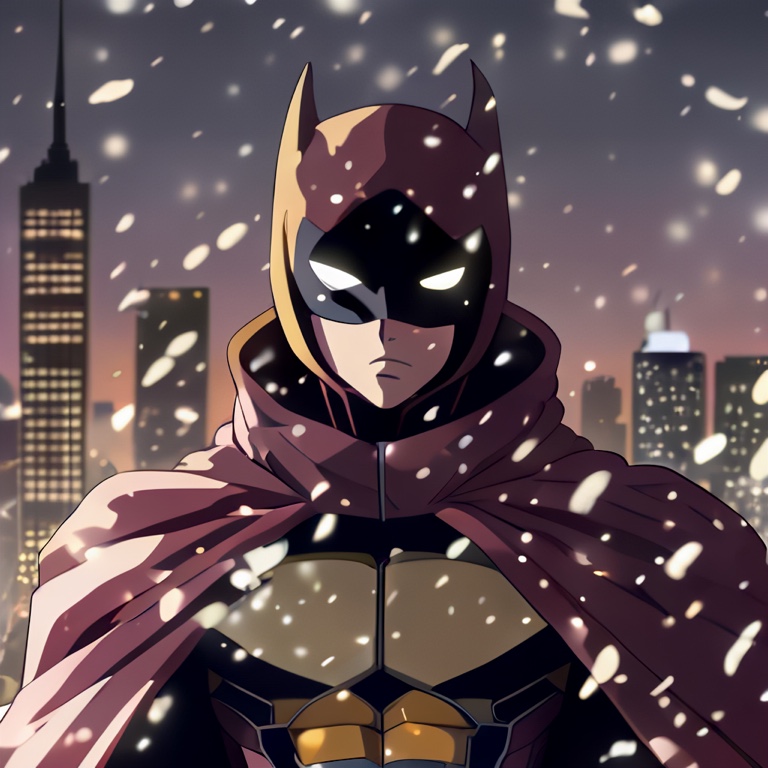}
& \includegraphics[width=0.3\textwidth]{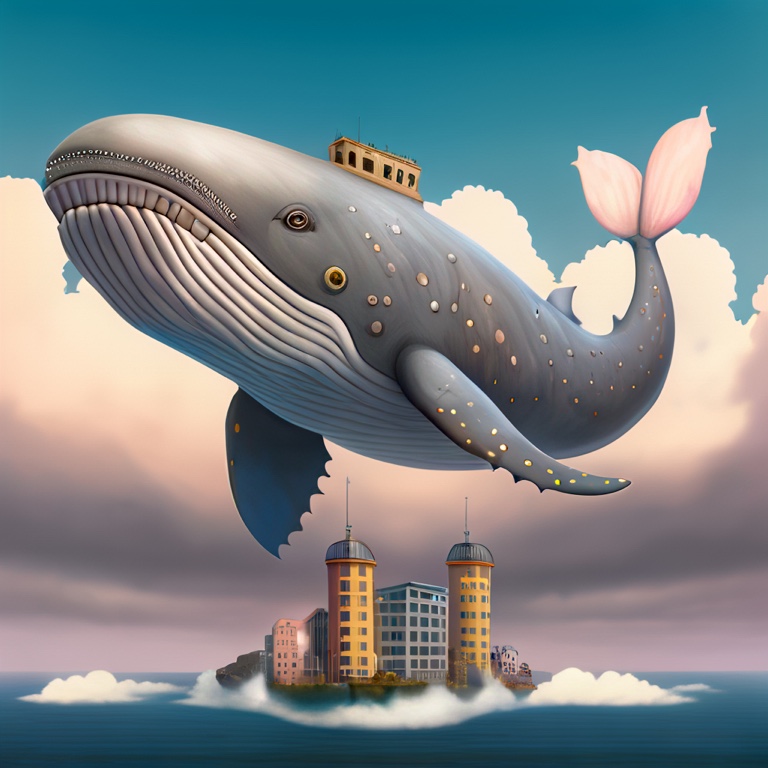} \\
\midrule
\textbf{Prompt} 
& \makecell{\textit{Concept art of a hidden} \\ \textit{pirate cove with waterfalls,} \\ \textit{wooden bridges, and} \\ \textit{lantern-lit paths.}}
& \makecell{\textit{A soft acrylic painting} \\ \textit{of a lighthouse on cliffs} \\ \textit{during a calm} \\ \textit{blue evening.}}
& \makecell{\textit{A watercolor botanical} \\ \textit{illustration of orchids,} \\ \textit{ferns, and delicate} \\ \textit{green leaves.}} \\
\midrule
\textbf{\texttt{D-FKC}}
& \includegraphics[width=0.3\textwidth]{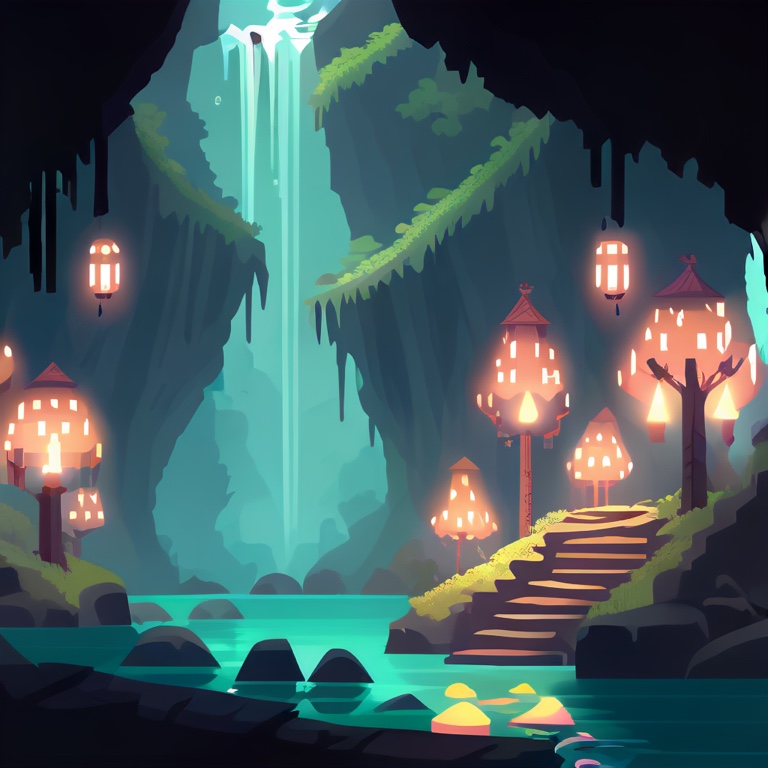}
& \includegraphics[width=0.3\textwidth]{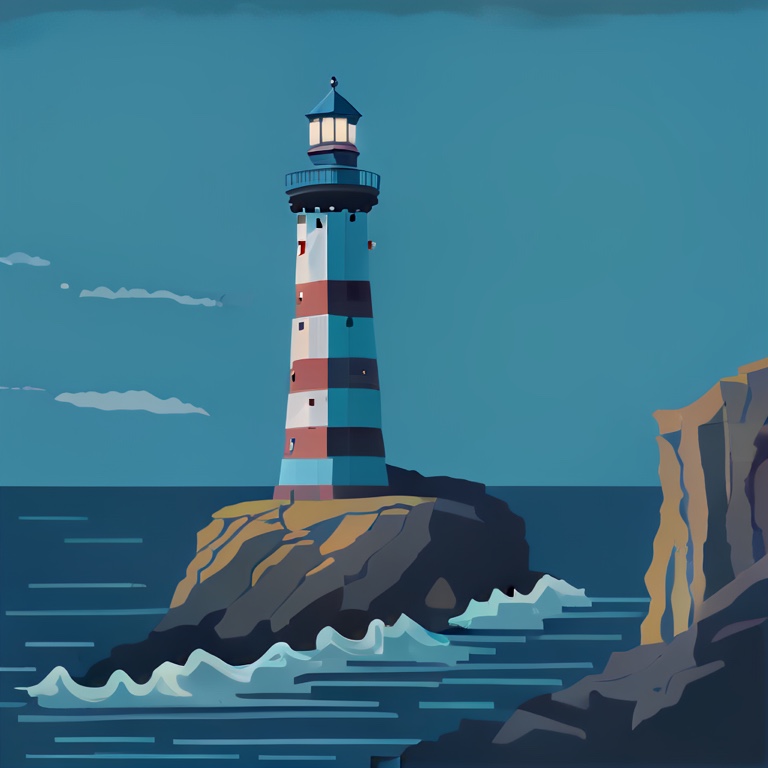}
& \includegraphics[width=0.3\textwidth]{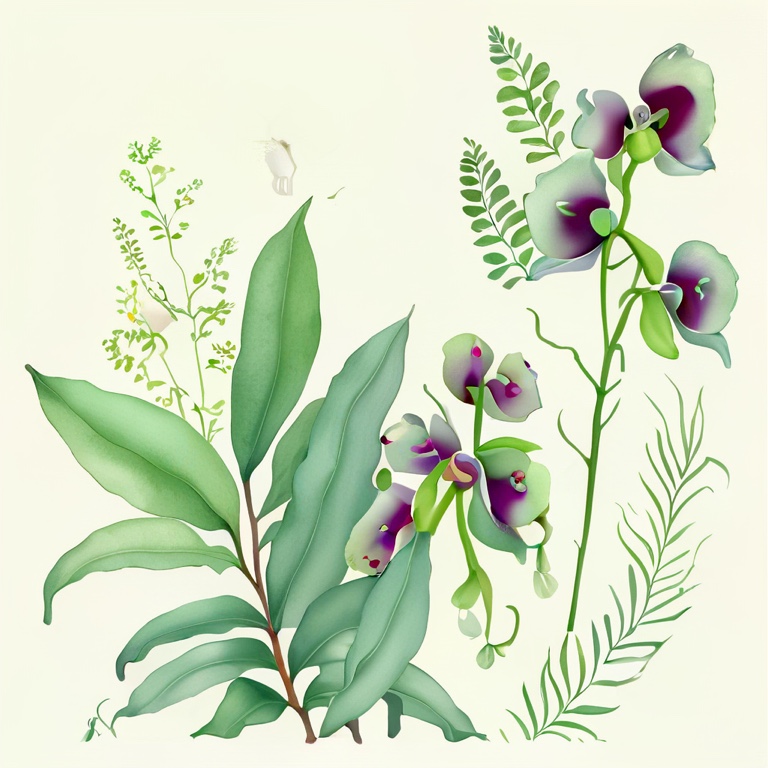} \\
\textbf{\texttt{D-VCG}}
& \includegraphics[width=0.3\textwidth]{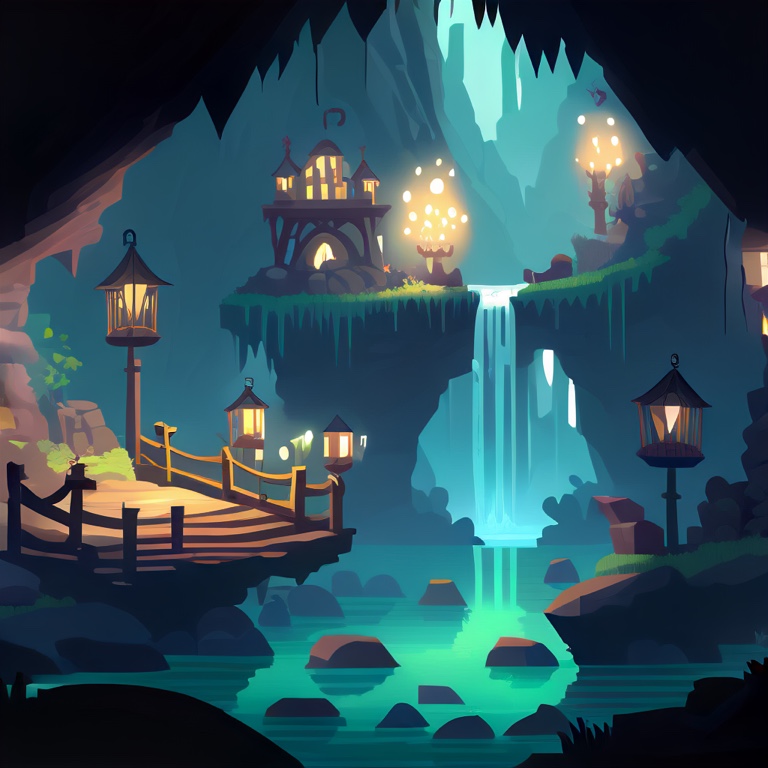}
& \includegraphics[width=0.3\textwidth]{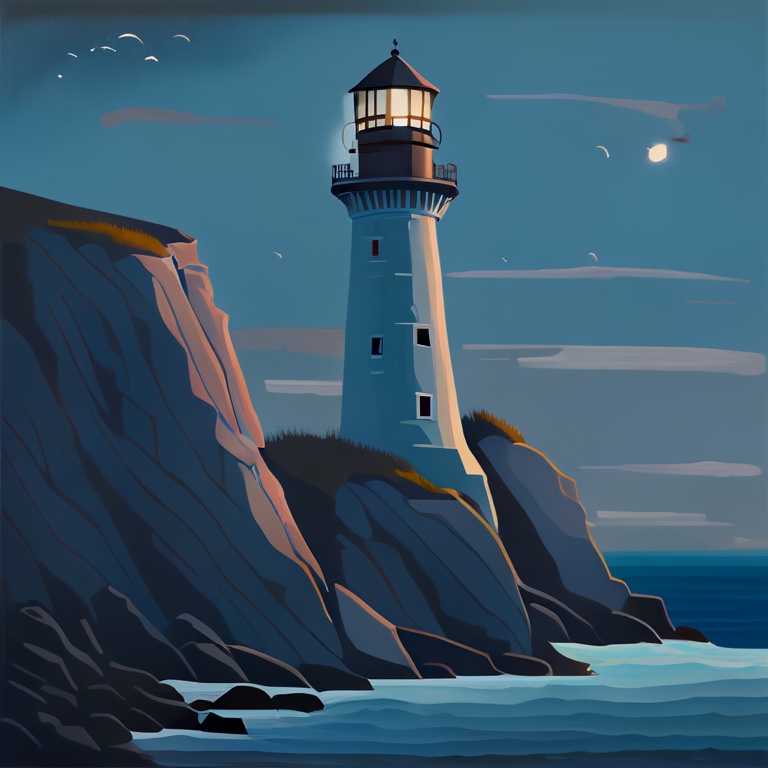}
& \includegraphics[width=0.3\textwidth]{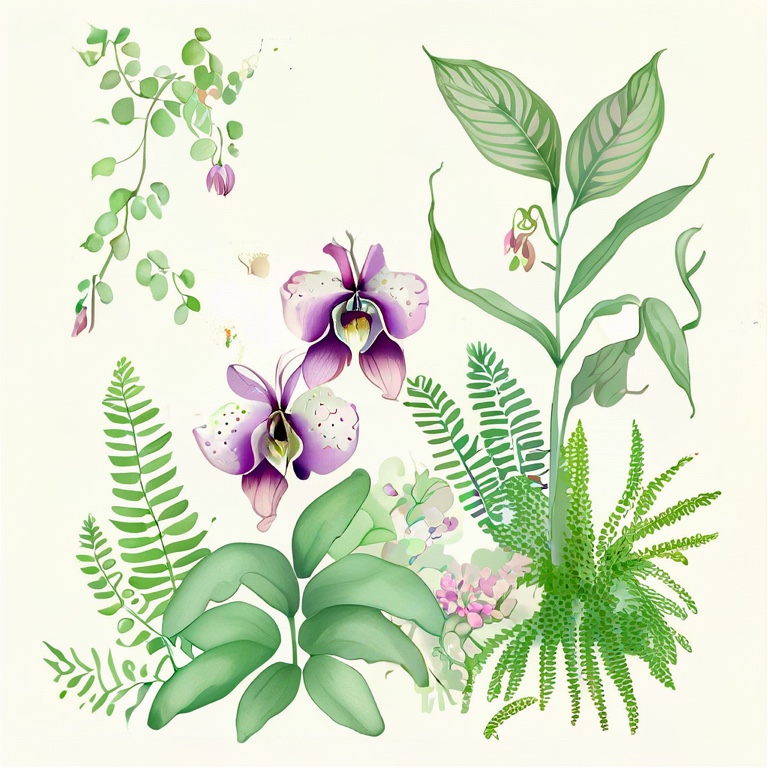} \\
\bottomrule
\end{tabular}}

\end{table*}

\end{document}